\documentclass[11pt, onecolumn]{IEEEtran}
\usepackage{amsmath,amssymb,amsfonts,amsthm}
\usepackage{algorithmic}
\usepackage{graphicx}
\usepackage{textcomp}
\usepackage{xcolor}
\usepackage{scalerel}
\usepackage[plain,noend]{algorithm2e}
\usepackage{times}
\usepackage{bm}
\usepackage[numbers,compress]{natbib}
\usepackage{paralist}
\usepackage{mathrsfs}
\usepackage{multirow}
\usepackage{subfig}
\usepackage{enumitem}
\usepackage{array,booktabs}
\usepackage{float}
\usepackage[normalem]{ulem}
\usepackage{anyfontsize}
\usepackage{url}
\usepackage{tablefootnote}

\newtheorem{theorem}{Theorem}
\newtheorem{remark}{Remark}
\newtheorem{corollary}{Corollary}
\newtheorem{definition}{Definition}
\newtheorem{lemma}{Lemma}

\newcommand{\bsb}{}%{\boldsymbol}%
\newcommand{\bsbX}{X}%{{\boldsymbol{X}}}
\newcommand{\bsbx}{x}%{{\boldsymbol{x}}}
\newcommand{\bsby}{y}%{{\boldsymbol{y}}}
\newcommand{\bsbY}{Y}%{{\boldsymbol{Y}}}
\newcommand{\bsbb}{\beta}%{{\boldsymbol{\beta}}}
\newcommand{\bsbg}{\gamma}%{{\boldsymbol{\gamma}}}
\newcommand{\bsbs}{s}%{{\boldsymbol{s}}}
\newcommand{\bsbS}{S}%{{\boldsymbol{S}}}
\newcommand{\bsbT}{T}%{{\boldsymbol{T}}}
\newcommand{\bsbI}{I}%{{\boldsymbol{I}}}
\newcommand{\bsbSig}{\Sigma}%{{\boldsymbol{\Sigma}}}
\newcommand{\bsbxi}{\xi}%{{\boldsymbol{\xi}}}
\newcommand{\bsba}{\alpha}%{{\boldsymbol{\alpha}}}
\newcommand{\bsbA}{A}%{{\boldsymbol{A}}}
\newcommand{\bsbB}{B}%{{\boldsymbol{B}}}
\newcommand{\bsbeps}{\epsilon}%{{\boldsymbol{\epsilon}}}
\newcommand{\bsbE}{\mathcal E}%{{\boldsymbol{E}}}
\newcommand{\bsbDelta}{\Delta}%{{\boldsymbol{\Delta}}}
\newcommand{\EE}{\,\mathbb{E}}
\newcommand{\EP}{\,\mathbb{P}}
\newcommand{\breg}{{\mathbf{\Delta}}}
\newcommand{\Breg}{{\mathbf{D}}}
\newcommand{\linrateparam}{\varepsilon}

\makeatletter
\newcommand{\oset}[3][0ex]{%
        \mathrel{\mathop{#3}\limits^{
                        \vbox to#1{\kern-2\ex@
                                \hbox{$\scriptstyle#2$}\vss}}}}
\makeatother
\newcommand{\back}{\oset\smallsetminus}

\DeclareMathOperator*{\argmin}{argmin}

\DeclareMathOperator{\vect}{\mbox{vec}\,}
\DeclareMathOperator{\tr}{tr}

\begin{document}

\title{Slow Kill for Big Data Learning}

% \author{\IEEEauthorblockN{Yiyuan She}
% \IEEEauthorblockA{\textit{Depart of Statistics} \\
% \textit{Florida State University}\\
% Tallahassee, Florida 32306, U.S.A. \\
% yshe@stat.fsu.edu}
% \and
% \IEEEauthorblockN{Jiahui Shen}
% \IEEEauthorblockA{\textit{Department of Statistics} \\
% \textit{Florida State University}\\
% Tallahassee, Florida 32306, U.S.A. \\
% jshen@fsu.edu}
% \and
% \IEEEauthorblockN{Adrian Barbu}
% \IEEEauthorblockA{\textit{Department of Statistics} \\
% \textit{Florida State University}\\
% Tallahassee, Florida 32306, U.S.A. \\
% abarbu@stat.fsu.edu}
% }

\author{Yiyuan She, Jiahui Shen, and Adrian Barbu\\
Department of Statistics, Florida State University
%\thanks{This work is supported in part by the National Science Foundation. (\emph{Corresponding author: Yiyuan She}.)}
%\thanks{Y. She is with the Department of Statistics, Florida State University, Tallahassee, Florida 32306, U.S.A. e-mail:  yshe@stat.fsu.edu.}% <-this % stops a space
%\thanks{J. Shen was with the Department of Statistics, Florida State University, Tallahassee, Florida 32306, U.S.A. e-mail:  jshen@fsu.edu.}% <-this % stops a space
%\thanks{A. Barbu is with the Department of Statistics, Florida State University, Tallahassee, Florida 32306, U.S.A. e-mail:  abarbu@stat.fsu.edu.}
}

\maketitle
%\IEEEtitleabstractindextext{
\begin{abstract}
Big-data applications often involve a vast number of observations and features,    creating new challenges for variable selection and parameter estimation. This paper presents  a novel technique called ``slow kill,''  which utilizes nonconvex constrained optimization, adaptive  $\ell_2$-shrinkage, and increasing learning rates. The fact that the problem size can decrease during the slow kill iterations makes it particularly effective for large-scale variable screening.  The interaction between statistics and optimization provides valuable insights into  controlling quantiles, stepsize, and shrinkage parameters in order to relax the regularity conditions required to achieve the desired level of statistical accuracy. Experimental results on real and synthetic data show that slow kill outperforms state-of-the-art algorithms in various situations while being computationally efficient for large-scale data.
\end{abstract}

%}

\section{Introduction}
\label{intro}
This paper studies how to build a parsimonious and predictive model in big data applications, where both the number of predictors and the number of observations can be extremely large. Let $\bsby\in\mathbb{R}^{n}$ be a response vector with $n$ samples and $\bsbX = [\bsbx_1, \ldots, \bsbx_p] \in\mathbb{R}^{n\times p}$ be a design matrix consisting of $p$ features or predictors.
Consider a general learning problem with loss $l_0(\bsbX\bsbb;\bsby)$ to measure the discrepancy between $\bsbX\bsbb$ and $\bsby$.  As $p$ can be much larger than $n$, a sparsity-promoting regularizer is often used to capture  model parsimony
\begin{equation} \label{eq:init_prob}
        \min_{\bsbb\in\mathbb{R}^{p}}l_{0}(\bsbX\bsbb;\bsby)+P(\bsbb;\lambda),
\end{equation}
where $\lambda$ is a regularization parameter.
There are numerous options for $l_0$ and $P$,   neither of which are necessarily convex. In many cases, $l_0$ may be a  negative log-likelihood function, but we will consider a more general setup that may not be based on likelihood.

Over the past decade, there have been significant advancements in statistical theory for the minimizers of the penalized problem \eqref{eq:init_prob}. However, modern scientists often encounter challenges with big data, making it impractical to obtain globally optimal estimators  even when convexity is present.  This paper aims to incorporate computational considerations into statistical modeling, resulting in a new big-data learning framework with theoretical guarantees.
 When tackling these challenges in large-scale variable selection, the desired algorithms should possess the following traits:

(a) Ease in tuning.
It is common in practice to seek a solution with  a \emph{prescribed} cardinality (or a specific number of variable, denoted by  $q$).   However, using an algorithm designed for the penalized problem   (1) may require excessive  computation, and the regularization parameter $\lambda$ may not be as intuitive when attempting to achieve this objective. Many practitioners perform a grid search for $\lambda$. However, when dealing with big data, the grid must be fine enough to encompass potentially useful candidate models, resulting in a substantial computational burden.

(b) Scalability.
In addition to being efficient, % on small or moderate-size datasets,
an ideal algorithm should   be easy to implement.   Since ad-hoc procedures can be unreliable,  it is preferable to employ an algorithm based on \textit{optimization} rather than relying on heuristics.   It would also be advantageous if the algorithm could adapt its parameters according to the available computational resources, which necessitates an understanding of the algorithm's iteration complexity and per-iteration cost.

(c) Statistical guarantee.
It is widely recognized that the   lasso  is effective for variable selection when the design matrix exhibits low coherence and the signal is sufficiently strong   \citep{bickel2009,bellec2018slope}. Some simpler and faster methods, such as those for variable screening \citep{Fan2008},  are based on the assumption of independent (or only mildly correlated) features.
While these weak-correlation assumptions allow for aggressive feature elimination, they are often restrictive   for real-world high-dimensional data.
Evaluating a globally optimal solution to  \eqref{eq:init_prob} with an $\ell_0$-type penalty \citep{zhang2010} does have a  statistically sound guarantee regardless of coherence, but is only computationally feasible for small datasets.  Therefore, a more pressing challenge is  to design an iterative process that can relax the stringent regularity conditions required for attaining optimal statistical accuracy.

This work proposes a new approach called \textit{slow kill}  to tackle the aforementioned challenges. The main features of the algorithm are as follows.

\begin{itemize}
        \item Interestingly, slow kill works in the opposite direction of forward pathwise methods and boosting algorithms, which all  build up  a model from the null  \citep{buhlmann2003boosting,needell2009cosamp,ZhangIT11,zhang2013multi,zhao2018pathwise}.
        \item Slow kill incorporates  adaptive $\ell_2$-shrinkage and growing learning rates to handle coherent designs and reduce computational burden. Its roots in optimization make it computationally scalable and easy to tune parameters.

        \item
        Theoretically, slow kill  enjoys rigorous, provable guarantees of accuracy and linear convergence  in a statistical sense. In particular, our theory supports backward  quantile control and fast learning.
\end{itemize}

The rest of  the paper is organized as follows.  Section \ref{iq} investigates a hybrid regularized estimation in the regression setting to motivate some basic elements of slow kill and  compares it to related works.  Section \ref{piq} introduces the general slow kill procedure for a differentiable loss function and analyzes how the statistical error changes as the cycles progress.  Section \ref{experiments} performs extensive simulations and real data experiments to compare slow kill to some state-of-the-art methods in terms of both efficiency and accuracy. We summarize our findings in Section \ref{summary}.  More technical   details   are provided in the appendix.

\emph{Notations and symbols.} The following notations and symbols will be used. Let $[n] = \{ 1, \ldots, n\}$ and $\lfloor x \rfloor$ be the largest integer smaller than or equal to $x$. Define $a\vee b = \max(a, b)$ and $a\wedge b = \min(a, b)$.
We use  $a \lesssim b$ to denote $a \leq cb$ for some positive constant $c$, and the constants denoted by $c$ or $C$ may not be the same at each occurrence.
Given any $\bsbb \in \mathbb{R}^p$, we use $\mathcal J(\bsbb) \subset [p]$ to denote its support, i.e., $\mathcal J(\bsbb)=\{j: \beta_j \ne 0\}$, and $J(\bsbb ) = |\mathcal J(\bsbb)| = \|\bsbb\|_0 = \sum_{j=1}^{p} 1_{\beta_j \neq 0}$.
Given $I \subset [p]$, we use $\bsbX_{I}$ to denote the sub-matrix of $\bsbX$ formed with the columns in $I$, and $\bsbb_{I}$ the subvector associated with $I$. In particular, $\bsbx_j$ denotes the $j$th column of $\bsbX$ for any $j \in [p]$.
When $\bsbA$ is a symmetric matrix, we use $\bsbA_{I}$ to denote the sub-matrix of $\bsbA$ formed with the columns and rows indexed by $I$, and $\lambda_{\max}(\bsbA)$, $\lambda_{\min}(\bsbA)$ to denote its largest and smallest eigenvalues, respectively.

Given $\bsbX \in \mathbb R^{n\times p}$,  the restricted isometry numbers $\rho_+(s)$, $\rho_-(s)$ \citep{Candes2005} are the smallest and largest numbers, respectively, that satisfy
\begin{align} \label{ripconsts}
        \rho_-(s) \|\bsbb\|_2^2 \leq \|\bsbX \bsbb\|_2^2 \leq \rho_+(s)\|\bsbb\|_2^2, \ \forall \bsbb \in \mathbb R^p: \|\bsbb\|_0 \leq s,
\end{align}
and their dependence on $\bsbX$ is omitted. Obviously, $ 0 \le \rho_-(s) \le \rho_+(s) \le \rho_+(p) = \|\bsbX\|_2^2 $, where $\|\bsbX\|_2$ denotes the spectral norm of $\bsbX$.

For ease of presentation, we introduce a quantile-thresholding operator $\Theta^{\#}$ which performs simultaneous thresholding and $\ell_2$-shrinkage \citep{she2013group}. Given any $\bsbs =[s_1, \ldots, s_p]^T \in \mathbb R^p$, $\Theta^{\#}(\bsbs; q, \eta) = [t_1, \ldots, t_p]^T$ satisfying $t_{(j)} = s_{(j)} / (1 + \eta) \mbox{ if } 1 \leq j \le q \mbox{, and } 0 \mbox{ otherwise,} $ where $s_{(1)}, \ldots, s_{(p)}$ are the order statistics of $s_1, \ldots, s_p$ satisfying $|s_{(1)}| \geq \cdots \geq |s_{(p)}|$, and $t_{(1)}, \ldots, t_{(p)}$ are defined similarly.
To avoid ambiguity, we make a $\Theta^{\#}$-uniqueness assumption in performing $\Theta^{\#}(\bsbs; q, \eta)$ throughout the paper:  either $|s_{(q)}| > |s_{(q+1)}|$ or $s_{(q)} = s_{(q+1)} = 0$ occurs. The multivariate quantile thresholding function    $\vec{\Theta}^\# (\bsbS; {q}, \eta)$  for any $\bsbS=[\bsbs_1, \ldots, \bsbs_p]^{T}\in \mathbb{R}^{p\times m}$ is defined as a $p\times m$ matrix ${\bsbT}=[\bsb{t}_1, \ldots, \bsb{t}_p]^{T}$ with ${\bsb{t}}_j=\bsbs_j/(1+\eta)$ if $\|\bsbs_j\|_2$ is among the ${q}$ largest  elements in   $\{\|\bsbs_j\|_2: 1\leq s \leq p\}$,  and $\bsb{0}$ otherwise.

%%%%%%%%%%%%%%%%%%%%%%%%%%%%%%%%%%%%%%%%%%%%%%%%%%%%%%%%%%%%%%%%%%%%%%%%%%%%%%%
%%%%%%%%%%%%%%%%%%%%%%%%%%%%%%%%%%%  Section 2
%%%%%%%%%%%%%%%%%%%%%%%%%%%%%%%%%%%%%%%%%%%%%%%%%%%%%%%%%%%%%%%%%%%%%%%%%%%%%%%

\section{Why Backward Selection?} \label{iq}
This section is to motivate a ``top-down''  algorithm design   in the fundamental regression setting. The quadratic loss is an important case of strongly convex losses and examining this case will provide a foundation for more general studies  under restricted strong convexity. % \citep{wainwright2019high}.

Assume $\bsby = \bsbX \bsbb^* + \bsbeps$, where $\bsbb^* \in \mathbb{R}^p$, $\|\bsbb^*\|_0 \le s$ with $s \le  p \wedge n$. To begin with, we consider an $\ell_0$-constrained, $\ell_2$-penalized optimization problem to estimate the coefficient vector in high dimensions,
\begin{equation} \label{eq:criterion}
        \min_{\bsbb} \frac{1}{2} \|\bsby - \bsbX \bsbb\|_2^2 + \frac{\eta_0}{2} \|\bsbb\|_2^2 \equiv f(\bsbb)  \mbox{\; s.t. } \|\bsbb\|_0 \leq q.
\end{equation}
When $\bsbX, \bsby$ are not centered, an intercept term $\bsb{1} \alpha$ should be added in the loss, and $\alpha$ is subject to no regularization. The hybrid regularization  in  (1)  differs from the commonly used linear combination of $\ell_1$ and $\ell_2$ penalties  in the elastic net \citep{zou2005regularization}.
Compared to the regular $\ell_1$ penalty and other nonconvex penalties, % \citep{fan2001variable,zhang2010},
$\|\cdot\|_0$ is arguably  an ideal choice  for enforcing sparsity and does not incur any unwanted bias. The constraint parameter $q \, (\le p)$   directly controls the number of variables in the resulting model, making it more convenient to use than a penalty parameter $\lambda$. The simultaneous $\ell_{2}$-penalty is  to compensate for collinearity and large noise, and is later used to overcome some obstacles in backward elimination. The associated regularization parameter $\eta_0$ can be easily tuned  and is not highly sensitive in experiments. Our theoretical analysis will reveal the benefits of a carefully designed  shrinkage sequence  for both numerical stability and statistical accuracy.

Problem \eqref{eq:criterion} is nonconvex and includes a discrete constraint. While it can be challenging to computationally solve problems of this nature, it is possible to find a local minimum using a scalable iterative optimization algorithm. Moreover, in the era of big data, it may not be necessary to fully solve \eqref{eq:criterion} in order to achieve good statistical performance for ``regular'' problems and analyzing algorithm-driven non-global estimators is crucial to discovering new and cost-effective methods for improving the statistical performance of nonconvex optimization.
Concretely, to introduce a prototype algorithm, we  first construct a surrogate function $g(\bsbb, \bsbb^-)$  for \eqref{eq:criterion},
\begin{equation*}
        \begin{split}
        g(\bsbb,\bsbb^{-})
        =\frac{1}{2}\lVert\bsby-\bsbX\bsbb^{-}\rVert_{2}^{2}+\langle\bsbX^{T}(\bsbX\bsbb^{-}-\bsby),\bsbb-\bsbb^{-}\rangle
        +\frac{\rho}{2}\lVert\bsbb-\bsbb^{-}\rVert_{2}^{2}+\frac{\eta_0}{2}\lVert\bsbb\rVert_{2}^{2},
        \end{split}
\end{equation*}
with $\rho>0$ to be chosen later, and then define a sequence of iterates by
\begin{equation} \label{g-opt}
        \bsbb^{(t+1)} = \arg \min_{\bsbb: \|\bsbb\|_0 \leq q} g(\bsbb, \bsbb^{(t)}).
\end{equation}
Recall the quantile-thresholding operator $\Theta^{\#}$ defined at the end of Section \ref{intro}. With some simple algebra (details omitted), we obtain an iterative quantile-thresholding algorithm
\begin{equation} \label{eq:update_rule}
        \bsbb^{(t+1)} = \Theta^{\#}\Big\{\bsbb^{(t)} - \frac{1}{\rho} \bsbX^T (\bsbX\bsbb^{(t)} - \bsby ); q, \frac{\eta_0}{ \rho} \Big\}.
\end{equation}
% denoted by \textbf{IQ} afterwards.
The first step amounts to the sure independence screening  \citep{Fan2008} when  $\bsbb^{(0)} = \bsb{0}$. However,   \eqref{eq:update_rule} iterates to lessen greediness with a low    per-iteration cost.

The update rule in \eqref{eq:update_rule} possesses some desirable computational properties.   For instance, if $\rho$ is large enough (more specifically, $\rho\ge \rho_+(2q)$ with $\rho_+(\cdot)$ defined in \eqref{ripconsts}), then the algorithm shows a worst-case sublinear convergence rate, regardless of the problem's dimensions, coherence, and signal strength. The obtained solutions (though not necessarily optimal) can be characterized as \textit{fixed points} of the algorithm mapping defined in \eqref{g-opt}. For more results and technical details, please refer to Theorem \ref{th:regconvergence}.

This class of procedures has been used in signal and information processing \citep{she2013group,SheetalPIQ}, and in the special case of $\eta_0 = 0$,  the plain update rule of \eqref{eq:update_rule} falls under the category of iterative hard-thresholding (IHT) algorithms  \citep{blumensath2008iterative,blumensath2009iterative} which only exhibit    mediocre performance (cf.   Remark \ref{remark3} and Section \ref{experiments}). % We will see that  the   simultaneous $\ell_2$ shrinkage can provide essential help in performing   backward selection and  relaxing  regularity conditions.
In fact,  there is much potential for improvement by {adaptively} adjusting   the three key parameters  $\rho, \eta_0, q$ in \eqref{eq:update_rule}, which has not been systematically explored in the literature.

\subsection{Statistical error analysis: power and limitations }
\label{subsec:IQerranal}
 While optimization error is important for analyzing an algorithm, our main focus is on {\textit{statistical error}}. This subsection  investigates  the prototype algorithm \eqref{eq:update_rule} to motivate  new techniques in later sections.
In order to  obtain sharp nonasymptotic results for this algorithm, it is important to note that  the thresholds vary from iteration to iteration and the final estimator may not be globally optimal.

Recall $\bsby = \bsbX\bsbb^* + \bsbeps$ with $\|\bsbb^*\|_0 \leq s$. Let $$\vartheta:= q/s$$ with $\vartheta > 1$ throughout the paper. A fixed point $\hat\bsbb$ associated with  \eqref{eq:update_rule} that satisfies the following equation is called a $\Theta^\#$-estimator,
\begin{equation} \label{eq:fix_point_property}
\hat\bsbb= \Theta^{\#}\big\{\hat\bsbb -  \frac{1}{\rho}\bsbX^T (\bsbX\hat\bsbb - \bsby ); q, \bar\eta_0\big\}, \ \mbox{with} \ \bar\eta_0 = \eta_0/\rho.
\end{equation}
Theorem \ref{th:statistical_accuracy}   studies the statistical accuracy of these estimators. %, rather than their optimization errors relative to a certain minimizer $\bsbb^o$.

\begin{theorem} \label{th:statistical_accuracy}
Assume that $\bsbeps$ is a sub-Gaussian random vector with mean zero and scale bounded by $\sigma$ (cf. Definition \ref{def:subgauss} in the appendix).
Let $\hat \bsbb$ be any estimator satisfying \eqref{eq:fix_point_property} for some $\eta_0\ge0$ with $\|\hat\bsbb\|_0=q$, and
  $\rho>0$ be chosen such that
\begin{equation} \label{eq:assumption_r0}
        \begin{split}
        \frac{\rho - \{(2-\linrateparam)\sqrt{\vartheta}-1\}\eta_0}{\sqrt{\vartheta}}\| \bsbb\|_2^2 \leq (2-\delta)\|\bsbX\bsbb\|_2^2 \quad
\forall \bsbb: \|\bsbb\|_0 \leq (1+ \vartheta)s
        \end{split}
\end{equation}
for some $\linrateparam, \delta > 0$. Then with probability at least $1- Cp^ {-c},$
{%\footnotesize
\begin{equation}
        \|\bsbX(\hat \bsbb - \bsbb^*)\|_2^2 \vee \frac{\eta_0\linrateparam}{\delta}\|\hat\bsbb - \bsbb^*\|_2^2 \lesssim   \frac{1}{\delta^2}\sigma^2\vartheta s\log{\frac{ep}{\vartheta s}}  +\frac{\eta_0}{\delta\linrateparam}\|\bsbb^*\|_2^2 ,  \label{pred_error_bound}
\end{equation}
}where $C, c > 0$ are constants.
\end{theorem}

From the error bound, \eqref{eq:update_rule}  can achieve the minimax optimal error rate of    $\mathcal O(\sigma^2 s \log(ep/s))$  \citep{raskutti2011minimax}, under the assumption of \eqref{eq:assumption_r0}   and when $\vartheta, \delta, \linrateparam$ are treated as constants. The result does not need  $\eta_0$ to be exactly zero. % In the case when $\eta_0 = 0$ or properly small, \eqref{pred_error_bound} becomes
% \begin{align} \label{pred_error_zero}
% \|\bsbX(\hat \bsbb - \bsbb^*)\|_2^2 & \lesssim \frac{1}{\delta^2}\sigma^2 \vartheta s\log{\frac{ep}{ \vartheta s}}.
% \end{align}
% This is advantageous over the penalized estimators, like the LASSO and MCP, which can only attain $\sigma^2 s\log(ep)$ \citep{bickel2009,zhang2010}.
In fact,   a positive $\eta_0$   can actually be beneficial in satisfying the condition of   \eqref{eq:assumption_r0}     (e.g.,  $\rho = (1.9\sqrt{\vartheta} - 1)\eta_0+1.9 \sqrt{\vartheta}  \rho_{-}(q+s)$ and $\linrateparam =\delta=0.1$,  applicable to $q >n$).  % \eqref{pred_error_zero} is still true as long as $\eta_0\|\bsbb^*\|_2^2 \lesssim (\sigma^2 \vartheta s/\delta) \log\{ep/(\vartheta s)\}$.
Another interesting observation is that  %  of the $\Theta^\#$-estimators,
$\rho$    should be chosen to be properly small   to achieve good statistical accuracy, which  is in contrasts to the bound  $\rho\ge \rho_+(2q)$ mentioned earlier for numerical convergence. The remarks below   make some further extensions and comparisons.   %in Theorem \ref{th:regconvergence}.
\\
\begin{remark}[{Estimation error bounds and faithful variable selection}] \label{remark1}
%\normalfont
{        The $\ell_2$-recovery result of Theorem \ref{th:statistical_accuracy} is fundamental, and can be used to derive estimation error bounds in other norms under proper regularity conditions.
        \begin{theorem} \label{support_recovery}
                In the setup of Theorem \ref{th:statistical_accuracy}, suppose  the regularity condition \eqref{eq:assumption_r0} is replaced by
                \begin{equation} \label{eq:assumption_r0_new}
                        \begin{split}
                                \big\{ \frac{\rho - (2\sqrt{\vartheta}-1)\eta_0}{\sqrt{\vartheta}} + \delta\rho_+((1+\vartheta)s) \big\}\|\bsbb\|_2^2 \leq 2\|\bsbX\bsbb\|_2^2, \    \forall \bsbb: \|\bsbb\|_0 \leq (1+ \vartheta)s
                        \end{split}
                \end{equation}
                for some $\delta > 0$. Then
                %\footnotesize
                \begin{equation}
                        \|\hat \bsbb - \bsbb^* \|_2^2 \lesssim  \frac{1}{\delta^2\rho_+((1+\vartheta)s)}\sigma^2 \vartheta s\log{\frac{ep}{\vartheta s}}+\frac{\eta_0^2}{\delta^2\rho_+((1+\vartheta)s)} \|\bsbb^*\|_2^2    \label{estimation_error_bound}
                \end{equation}
                holds with probability at least $1- Cp^ {-c}$, for some $C, c > 0$.
                Moreover, under
                \begin{equation} \label{eq:assumption_infty}
                        \nu\|\bsbb\|_{\infty} \le \| (\bsbX^T\bsbX + \eta_0 \bsbI) \bsbb \|_{\infty} /n, \ \  \bsbb: \|\bsbb\|_0 \leq (1+ \vartheta)s
                \end{equation}
                for some $\nu >0$, any fixed-point $\hat\bsbb$ satisfies
                \begin{equation} \label{element_wise_error}
                 \|\hat\bsbb - \bsbb^*\|_{\infty} \le \frac{(\rho+\eta_0)}{n\nu \sqrt{\vartheta - 1}}\frac{\|\hat\bsbb - \bsbb^*\|_2}{\sqrt{s}} + \frac{\|\bsbX^T\bsbeps\|_{\infty}}{n\nu} + \frac{\eta_0}{n\nu}\|\bsbb^*\|_{\infty,}        \end{equation}
and
\begin{equation}\label{element_wise_error2}
 \|(\hat\bsbb - \bsbb^*)_{\mathcal J^*}\|_{\infty}  +(1- \frac {\rho+\eta_0 }{n\nu})  \|(\hat\bsbb - \bsbb^*)_{\hat {\mathcal J}\setminus\mathcal J^*}\|_{\infty}  \le  \frac{\|\bsbX^T\bsbeps\|_{\infty}}{n\nu} + \frac{\eta_0}{n\nu}\|\bsbb^*\|_{\infty},
\end{equation}
where $\mathcal J^* =\mathcal J(\bsbb^*)$, $\hat {\mathcal J} = \mathcal J(\hat \bsbb)$.

        \end{theorem}

        Compared with   \eqref{eq:assumption_r0}, the condition of  \eqref{eq:assumption_r0_new}  replaces   $\delta\|\bsbX\bsbb\|_2^2$ by   $\delta\rho_+((1+\vartheta)s)\|\bsbb\|_2^2$. When $q$ and $s$ are small, $\rho_+((1+\vartheta)s)$ is of the order $\mathcal{O}(n)$.  Therefore,  \eqref{estimation_error_bound} becomes
        $%        \begin{align} \label{estimation_error_bound_zero}
                \| \hat \bsbb - \bsbb^* \|_2^2  \lesssim  \{\sigma^2 s\log(ep/s)\}/{n},
                $ %        \end{align}
        assuming $\delta, \vartheta$ are constants and $\eta_0$ is properly small.

Moreover,  the element-wise error bound \eqref{element_wise_error} implies   \emph{faithful} variable selection   under regularity condition \eqref{eq:assumption_infty} (which,   like  previous regularity conditions, favors  low coherence, i.e., the off-diagonal entries of $\bsbX^T \bsbX /n$ should be relatively small in magnitude). Specifically, assuming   $\vartheta, \nu, \delta$ are constants, $\|\bsbx_j\|_2 \lesssim \sqrt{n}$, $\rho+\eta_0 \lesssim n$ and the  beta-min condition $\min_{j \in \mathcal J^*}|\beta^*_j| > c \sigma \{\log(ep)/{n}\}^{1/2}$ with a sufficient large constant $c$,  \eqref{element_wise_error} indicates that  the $s$ largest elements in $|\hat\bsbb_j|$  correspond to    $ \mathcal{J}^* = \{ j: \beta^*_j \neq 0\}$ with high probability.}
\end{remark}

\begin{remark}[{Fixed points vs. globally optimal solutions}] \label{remark2} % \normalfont
The    statistical accuracy results \eqref{pred_error_bound}, \eqref{estimation_error_bound}, and \eqref{element_wise_error}  are proved for  {all}   nonglobal fixed-point estimators defined by  \eqref{eq:fix_point_property}.
Our proof can be slightly modified to show that if a globally optimal solution  can be computed, the statistical error rate remains unchanged but the left-hand side of \eqref{eq:assumption_r0} becomes 0, indicating that the regularity condition always holds for any $\delta \le 2$. However, relying on multiple starting points to obtain a globally optimal solution and thus improve statistical performance  can be inefficient for large datasets.\end{remark}
\begin{remark}[{Comparison with some  theoretical works}]\label{remark3}
% \normalfont
The aforementioned  class of IHT  algorithms  may refer to the use of hard-thresholding $\Theta_H(\bsb{s}; \lambda) = [s_i 1_{|s_i|\ge \lambda}]$ with a fixed threshold $\lambda$,  or a varying threshold as the $q/p$-th quantile of $|s_i|$ ($1\le i \le p$) by fixing $q$  \citep{blumensath2008iterative,blumensath2009iterative}.
%Clearly, with no $\ell_2$-shrinkage, iterative quantile-thresholding is an instance.
In comparison, the  $\ell_2$ component in \eqref{eq:update_rule}  {should not} be ignored, and it may result in a different sparsity pattern  in the presence of high coherence and large $p$.
Fairly speaking, the performance of IHT is not on par with   some standard statistical methods and packages (such as the lasso). This is why  we performed  theoretical   analysis in the hopes of  discovering  and developing  new techniques.

In a theoretical study, \cite{liu2020between} obtained a convergence result in terms of function value under $$\vartheta>  \rho_+^2(2q) /\rho_-^2(2q) ,$$  which improves the  condition in \cite{jain2014iterative} $$\vartheta> 32 \rho_+^2(2q) /\rho_-^2(2q). $$  Our condition in Theorem \ref{th:statistical_accuracy}  is even less restrictive.
For example,  a sufficient condition  for \eqref{eq:assumption_r0} is
$$\vartheta >  \{\rho_+ (2q) + \eta_0\}^2/[4\{\rho_- (q+s) + \eta_0\}^2],$$ or $$\vartheta >  \{\rho_+ (2q) + \eta_0\}^2/[4\{\rho_- (2q) + \eta_0\}^2)] $$  since   $ \rho_- (q+s) \ge  \rho_- (2q ) $, which becomes   $\vartheta >   \rho_+^2 (2q)  /\{4 \rho_-^2 (2q)  \} $ in the worst case of   $\eta_0=0$. In conclusion,  $32 \rho_+^2(2q) /\rho_-^2(2q)\ge\rho_+^2(2q) /\rho_-^2(2q)\ge\rho_+^2 (2q)  /\{4 \rho_-^2 (2q)  \} \ge \rho_+^2 (2q) /[4\{\rho_- (q+s) \}^2]\ge \{\rho_+ (2q) + \eta_0\}^2/[4\{\rho_- (q+s) + \eta_0\}^2]  $, and our  obtained error rate of $\sigma^2 s \log (ep/s)$ is minimax optimal.

Interested readers may also refer to  \cite{wang2014optimal,loh2015regularized,She2016,zhao2018pathwise,liu2020between,SheBregman}, for example, for the analyses of various penalties and  mixed thresholding rules, with an error rate of $ \sigma^2 s \log (e p)$.  %   \cite{jain2014iterative} and \cite{liu2020between}.
Since our purpose is to  design a new {backward} selection  algorithm for  problems with a predetermined number of features, we will not discuss their technical assumptions.  The  experiments in Section \ref{experiments}  make a comprehensive comparison of different methods  in various scenarios.
\end{remark}

\subsection{New means of improvement for large-scale data}
\label{subsec:newmeans}
Providing   provable guarantees for prediction, estimation, and variable selection   is reassuring. But the real challenge lies in finding innovative techniques that can \emph{relax} the required regularity conditions to ensure good statistical accuracy, while being more cost-effective than using multiple random starts. To gain further insights, we can use the restricted isometry numbers (as defined in \eqref{ripconsts}) to provide a sufficient condition for  \eqref{eq:assumption_r0}:
\begin{equation} \label{eq:rip_to_r0_origin}
\rho < 2\sqrt{\vartheta}\rho_-(q+s) + (2\sqrt \vartheta -1)\eta_0 \mbox{\ \ or \ } 4\vartheta >  \frac{(\rho+\eta_0)^2}{(\rho_- (q+s) +  \eta_0)^2}.
\end{equation}
\subsubsection{``Fast'' learning}

One key takeaway from the results presented in Section \ref{subsec:IQerranal} is the importance of the inverse learning rate, $\rho$. In the field of machine learning, it is commonly advised to use a ``slow" learning rate when training a nonconvex model. This  can ensure good computational performance, as evidenced by the lower bound of $\rho$ in Theorem \ref{th:regconvergence}. However, it is  important to note that according to \eqref{eq:rip_to_r0_origin}, using an excessively large value for $\rho$ may compromise the statistical guarantee of the model.

In fact,  \eqref{eq:assumption_r0}  suggests that smaller values of $\rho$ are preferred, and combining statistical and numerical analysis leads to the following range for $\rho$:
\begin{equation} \label{rho_condition}
\rho_+(2q) \le \rho \le 2 \sqrt \vartheta \rho_- (q+s)+ (2\sqrt \vartheta -1)\eta_0.
\end{equation}
%Here, again, a positive $\eta_0$ is helpful.
In convex programming, the choice of stepsize  does not affect the optimality of the solution as long as the algorithm converges. However, in our case of nonconvex constrained optimization, it is important to choose a large enough value for $1/\rho$ not only to gain fast convergence, but also to ensure statistical accuracy. To the best of our knowledge, this is a novel finding.  Since it may not be easy to determine the theoretical restricted isometry numbers in practice, a routine line search for the step size can be used. Specifically,  according to the proof in Appendix \ref{app:numericalproof}, one can use the majorization condition
$ %\begin{equation} \label{chain_first_half}
f(\bsbb^{(t+1)}) \leq g(\bsbb^{(t+1)}, \bsbb^{(t)}) \mbox{  or   }
\|\bsbX(\bsbb^{(t+1)} - \bsbb^{(t)})\|_2^2 \le \rho\|\bsbb^{(t+1)}-\bsbb^{(t)}\|_2^2
$ %\end{equation}
to prevent $\rho$ from becoming too large while still preserving the convergence properties stated  in Theorem \ref{th:regconvergence}.
The concept of using an iteration-varying sequence $\rho_t$ will be important in   the next section.

\subsubsection{``Backward'' selection}
Another important discovery is the influence of cardinality control.
 If we use a conservative inverse learning rate of $\rho = \rho_+(2q)$, then \eqref{eq:rip_to_r0_origin} imposes a limit on the restricted condition number of the design matrix:
\begin{equation} \label{theta_condition}
\vartheta >  [\rho_+ (2q) + \eta_0]^2/\{4[\rho_- (q+s) + \eta_0]^2)\}.
\end{equation}
This suggests a promising approach to relax the regularity condition by increasing the value of $\vartheta$. %  This suggests that, instead of using multiple starting points as mentioned in Remark \ref{remark2}, it may be more efficient to simply increase the value of $\vartheta$ in order to improve statistical learning abilities.

Figure \ref{figure:rip_rate}    confirms the point assuming random designs: the larger the value of $q$ is, the more likely it is for \eqref{eq:rip_to_r0_origin} to hold on large-scale data.    Random matrix theory also   supports this idea.

\begin{figure}[h!]
\centering
\includegraphics[width=0.4\textwidth, height=2in]{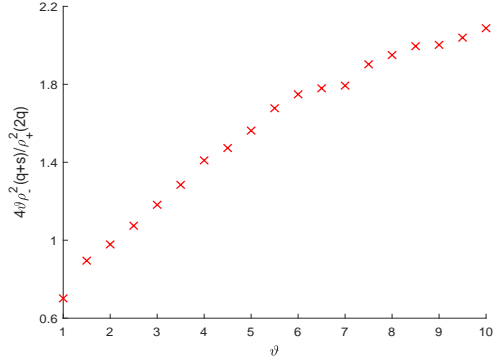}
\caption{ An illustration of how $4\vartheta \rho_-^2(q+s)/\rho_+^2(2q)$ varies as $\vartheta$ increases. Here, the rows of $\bsbX$ are independently drawn from a multivariate Gaussian distribution with zero mean and the covariance  $\bsbSig = [0.5^{|i-j|}]$,  $n = 2\mbox{,}000, p = 4\mbox{,}000, s = 4$. To determine $\rho_{\pm}$ for a given matrix $\bsbX$, we perform a random sampling.   The results are averaged over 100 independent $\bsbX$'s that are generated from the same distribution. } \label{figure:rip_rate}
\end{figure}

% Random matrix theory   lends some support; see Theorem \ref{rmt} below.

\begin{theorem} \label{rmt}
Assume that the rows of the random matrix $\bsbX \in \mathbb R^{n\times p}$ are independent and identically distributed as $N(\bsb{0}, \bsbSig)$, where $\Sigma_{ii}  \le 1$.
Let $\lambda^{(2q)}_{\max}$ be the largest eigenvalue of $\bsbSig_I$ for all $I \subset [p]$ with $|I| \le 2q$, and
$\lambda^{(q+s)}_{\min}$ be the smallest eigenvalue of $\bsbSig_I$ for all $I \subset [p]$ with $|I| \le q+s$.
Then for any $0<c<1$,
{%\scriptsize
\begin{align} \label{rate_compare}
        \frac{\rho_+(2q)}{\rho_-(q+s)}  \le
        \left\{  \frac{ (1+c)\sqrt{\lambda^{(2q)}_{\max}} + \sqrt{\{2\lambda^{(2q)}_{\max}q\log(ep/q)\}/n} + \sqrt{2q/n} }{  (1-c)\sqrt{\lambda^{(q+s)}_{\min}} - \sqrt{\{\lambda^{(q+s)}_{\min}(q+s)\log(ep/q)\}/n} - \sqrt{(q+s)/n} }   \right\}^2
\end{align}}\normalsize
with probability at least $1 - 2\exp(-nc^2/2)$, assuming $n \ge \{2(q+s)/(1-c)^2\} \{ 1/\lambda^{(q+s)}_{\min} + \log(ep/q) \}$.
\end{theorem}

The results can be extended to sub-Gaussian designs (by using, for example, Theorem 6.2 of \cite{wainwright2019high} and Weyl's theorem).
Let us consider the Toeplitz design $\bsbSig = [\tau^{|i-j|}]$ with $0 \le \tau < 1$. By the interlacing theorem,
\begin{equation*}
        \begin{split}
                (1-\tau)/(1+\tau) = \lambda_{\min}(\bsbSig) \le  \lambda^{(q+s)}_{\min}
                \le \lambda^{(2q)}_{\max} \le \lambda_{\max}(\bsbSig) = (1+\tau)/(1-\tau),
        \end{split}
\end{equation*}
and so the right-hand side of \eqref{rate_compare} is bounded by a {constant} with high probability as $n \gg q\log{(ep/q)}$. %, or $\rho_+^2(2q)/(4\rho_-^2(q+s)) \lesssim q$ holds with high probability.
Accordingly, the  regularity condition can be satisfied with a properly large $\vartheta$. % With the help from   $\eta_0$, even $q>n$ is allowed. % as shown in Figure \ref{figure:rip_rate}.

Of course,  the error bound in \eqref{pred_error_bound} also increases with larger values of $q$. To address this issue, we propose employing a \textit{decreasing} sequence of $q_t$ to progressively tighten the cardinality constraint. Based on previous discussions, it is  thus advisable to use    \textit{increasing} learning rates $1/\rho_t$ (such as $1/\rho_+(2q_t)$) in the iterative process. It may also be beneficial to adjust the shrinkage parameter to a sequence $\eta_t$, particularly when $q_t > n$.
 % Line search   can get an inverse stepsize even smaller than $\rho_+(2q_t)$.
This resulting algorithm, which combines progressive quantiles, $\ell_2$-shrinkage, and learning rates, will be referred to as ``slow kill."  It differs from the pure optimization algorithm \eqref{eq:update_rule} with a fixed $q$ and from various bottom-up boosting and greedy algorithms that are commonly used in the literature.

 %One  major characteristic is that instead of selecting  $q$ features in each  step, slow kill plays the keep-or-kill game with a sequence of $q_t$ that gradually decreases from $p$ to $q$.

The purpose of this section is to provide a  compelling rationale for certain aspects of slow kill techniques. We will  present  results in a more general setting, including fast convergence of the iterates and how slow kill improves the quality of the initial estimate as $q_t$ approaches $q$,   further relaxing the regularity conditions.

%%%%%%%%%%%%%%%%%%%%%%%%%%%%%%%%%%%%%%%%%%%%%%%%%%%%%%%%%%%%%%%%%%%%%%%%%%%%%%%
%%%%%%%%%%%%%%%%%%%%%%%%%%%%%%%%%%%  Section 3
%%%%%%%%%%%%%%%%%%%%%%%%%%%%%%%%%%%%%%%%%%%%%%%%%%%%%%%%%%%%%%%%%%%%%%%%%%%%%%%

\section{Adaptive Control of   Quantiles, Learning Rates, and $\ell_2$-shrinkage} \label{piq}
Given a general  loss, based on the  discussions in  the last section, we  pursue sparsity in $\bsbb$ via
\begin{equation}\label{eq:general criterion}
\min_{\bsbb \in\mathbb{R}^{p}} l_{0}(\bsbX\bsbb;\bsby)+\frac{\eta_0}{2}\|\bsbb\|_2^{2} \,\equiv l ( \bsbb )+\frac{\eta_0}{2}\|\bsbb\|_2^{2} \equiv f(\bsbb) \mbox{ s.t. } \lVert\bsbb\lVert_{0}\leq q,
\end{equation}
where for notational ease,  $l_0(\bsbX\bsbb;\bsby)$ is  often abbreviated as $l(\bsbb)$.
Again, the use of hybrid regularization is intended to address collinearity and large $p$. We assume that the regularization parameters $q, \eta_0$ are given in the  algorithm design and theoretical analysis. (Of course, given $q$, one can easily tune the value of $\eta_0$   using methods such as AIC; as for the selection of $q$, an information criterion is provided in the Appendix \ref{appsub:tuning}.) The {generalized Bregman function} for a differentiable $l$ is one of the main tools we use to handle a variety of losses:
\begin{equation} \label{bregman}
\breg_{l}(\bsbb_1, \bsbb_2) := l(\bsbb_1) - l(\bsbb_2) - \langle \nabla l(\bsbb_2), \bsbb_1 - \bsbb_2 \rangle,
\end{equation}
 where the differentiability can be replaced by directional differentiability to analyze a  wide range of algorithms in statistical computation  \cite{SheBregman}.
% which can be seen as the difference between the function $l$ and its linear approximation made at $\bsbb_2$.
If $l$ is also strictly convex, $\breg_{l}$ becomes the standard Bregman divergence \citep{Bregman1967,Zhang2010Bregman}.
When  $l( \cdot ) = \| \cdot \|_2^2/2$, $\breg_{l}(\bsbb_1, \bsbb_2) = \|\bsbb_1 - \bsbb_2\|_2^2/2$, which is symmetric, and we abbreviate it to $\Breg_2(\bsbb_1, \bsbb_2)$.  Define the symmetrized version of $\breg_l (\bsbb_1, \bsbb_2)$ by
$%\begin{equation} \label{bar_breg}
\bar \breg_{l} (\bsbb_1, \bsbb_2) := \{\breg_{l}(\bsbb_1, \bsbb_2) + \back\breg_{l}(\bsbb_1, \bsbb_2) \}/ 2,
$ %\end{equation}
where $\back \breg_{l}(\bsbb_1, \bsbb_2) = \breg_{l}(\bsbb_2, \bsbb_1)$.
As an extension of  \eqref{ripconsts}, we introduce two generalized restricted isometry numbers $\rho_+^{l}(s_1,s_2)$, $\rho_-^{l}(s_1, s_2)$ % as the smallest and largest numbers, respectively,
that satisfy
\begin{align}
&\breg_{l} (\bsbb_1, \bsbb_2) \le \rho_+^{l}(s_1,s_2) \Breg_2(\bsbb_1, \bsbb_2), \ \forall \bsbb_i: \| \bsbb_i\|_0\le s_{i}, i = 1,2 \label{gen-rip-upp}
\\ & \breg_{l} (\bsbb_1, \bsbb_2) \ge \rho_-^{l}(s_1,s_2) \Breg_2(\bsbb_1, \bsbb_2), \ \forall \bsbb_i: \| \bsbb_i\|_0\le s_{i}, i = 1,2. \label{gen-rip-low}
\end{align}
We differentiate $s_1, s_2$  because $\breg_l$ may not be symmetric.  These numbers  will be convenient and useful for theoretical purposes; for example,    Theorem \ref{th:general convergence} and Theorem \ref{th:new}    will use  positive     $\rho_+^l(q, q)$ and $\rho_+^{l}(q, s)$, respectively, while   Theorem \ref{th:iter} will  use nonnegative $\rho_-^{l}$.  When $l(\bsbb) = \|\bsbX\bsbb - \bsby\|_2^2/2$, $\breg_{l}(\bsbb_1, \bsbb_2) = \|\bsbX\bsbb_1 - \bsbX\bsbb_2\|_2^2/2$ and $\rho^l_+(s_1, s_2) = \rho_+(s_1 + s_2)$. More generally, if the gradient of $l_0(\cdot;\bsby)$ is $L$-Lipschitz continuous, as is the case in regression or logistic regression,
\begin{equation} \label{eq:Lipschitz}
\Vert\nabla l_0(\bsbxi_{1};\bsby)-\nabla l_0(\bsbxi_{2};\bsby)\Vert_{2}\leq L\|\bsbxi_{1}-\bsbxi_{2}\|_{2},
\end{equation}
for all $\bsbxi_1, \bsbxi_2\in \mathbb R^n$,
then it is easy to show that
\begin{equation} \label{rho_condition_general}
\rho_{+}^{ l }(s_1,s_2) \le L\rho_+(s_1+s_2) \ (\le L \|\bsbX\|_2^2).
\end{equation}
%If $l$ is differentiable and $\mu$-strongly convex  \citep{Nemirovski_Notes}, then $$ \Breg_{l}(\bsbb_1, \bsbb_2) \ge \mu \Breg_2 (\bsbb_1, \bsbb_2),$$ or $\Breg_{l} \ge \mu \Breg_2$ for simplicity.

\subsection{Numerical convergence and statistical accuracy for  the general optimization algorithm}
First, we extend the previous iterative quantile-thresholding algorithm to handle losses that may not be quadratic.  Construct the following surrogate function
\begin{equation} \label{g-opt-general}
g(\bsbb,  \bsbb^{-})  = l_0(\bsbX \bsbb; \bsby) + \frac{\eta_0}{2}\|\bsbb\|_2^2 + (\rho\Breg_2 - \breg_{l})(\bsbb, \bsbb^{-}),
\end{equation}
which is by linearizing the loss (only).
Then, similar to the derivation  in Section \ref{iq},    \eqref{g-opt-general} leads to an algorithm
\begin{equation} \label{eq:general_update_0}
\bsbb^{(t+1)} = \Theta^{\#}\Big\{\bsbb^{(t)} - \frac{1}{\rho}\bsbX^T\nabla l_0(\bsbX\bsbb^{(t)}; \bsby); q, \frac{\eta_0}{\rho}\Big\}.
\end{equation}
Some  basic numerical properties are summarized as follows.
\begin{theorem} \label{th:general convergence}
Assume that $\inf_{\bsbxi, \bsby} l_0(\bsbxi;\bsby) > - \infty$.
Consider   \eqref{eq:general_update_0} starting from an arbitrary feasible $\bsbb^{(0)}$.
Then $\rho \geq \rho_{+}^{ l }(q, q)$ guarantees that for all $t \geq 0$, $f(\bsbb^{(t+1)}) \leq g(\bsbb^{(t+1)}, \bsbb^{(t)})$ and
$ %\begin{equation}
(\rho-\rho_{+}^{ l }(q, q))\Breg_2(\bsbb^{(t+1)}, \bsbb^{(t)})\le f(\bsbb^{(t)}) - f(\bsbb^{(t+1)}),
$ %\end{equation}
and so the objective function values converge as $t \to \infty$.
Assume $\rho >\rho_{+}^{ l }(q,q) $, $\eta_0 > 0$ and $\nabla l_0$ is continuous. Then every accumulation point $\hat \bsbb $ of $ \bsbb^{(t)} $ satisfies the fixed-point equation
\begin{equation} \label{eq:fix_point_property_general}
        \hat\bsbb= \Theta^{\#} \{\hat\bsbb - \bsbX^T\nabla l_0(\bsbX\hat\bsbb;\bsby)/\rho; q, \eta_0/\rho \}.
\end{equation}
Furthermore, if $l_0(\cdot;\bsby)$ is convex, $\lim_{t \to \infty} \bsbb^{(t)} = \hat\bsbb$, and under $\|\hat\bsbb\|_0=q$, $\hat\bsbb$ is a local minimizer to problem \eqref{eq:general criterion} and the support of $\bsbb^{(t)}$ stabilizes in \emph{finitely} many iterations.
\end{theorem}

Next,   we turn to the  statistical accuracy of the  estimators that are defined by \eqref{eq:fix_point_property_general}. %obtained from the general IQ algorithm \eqref{eq:general_update_0}.
To overcome the  obstacle that  the loss is not necessarily associated with a probability density function,      we define the concept of \emph{effective noise} with respect to the statistical truth $\bsbb^*$  as
\begin{equation} \label{effective-noise}
\bsbeps = - \nabla l_0(\bsbX \bsbb^*; \bsby),
\end{equation}
where we treat $\bsbX$ as   fixed  and $\bsby$ as random in this section.
The definition of effective noise in  \eqref{effective-noise} does not depend on the regularizer.
In the special case of a generalized linear model with cumulant function $b$ and canonical link function $g = (b')^{-1}$, % \citep{agresti2012categorical},
the loss is $l(\bsbb) = l_0(\bsbX\bsbb;\bsby) = - \langle\bsby, \bsbX\bsbb\rangle + \langle 1 ,b(\bsbX\bsbb)\rangle$, and so
$\bsbeps = \bsby - g^{-1}(\bsbX\bsbb^*) = \bsby - \EE\bsby$.
For regression, the effective noise term $\bsbeps$ is equivalent to the raw noise, which is usually assumed to be Gaussian. In the case of classification using the logistic deviance, $\bsbeps$ is bounded, making it sub-Gaussian. In fact, any loss function with a bounded derivative, such as Huber's loss, Hampel's loss, or the hinge loss, will always result in a sub-Gaussian $\bsbeps$, regardless of the distribution of $\bsby$.
In this section, we assume that the effective noise is a sub-Gaussian random vector with mean zero and scale bounded by $\sigma$. However, our proof techniques can be applied more generally. The following theorem provides a risk bound for the estimators obtained by    \eqref{eq:general_update_0}, and also demonstrates the impact of the quality of the starting point on the regularity condition.

\begin{theorem} \label{th:new}
Let $\hat \bsbb: \|\hat \bsbb\|_0 = q$ be an estimate obtained from \eqref{eq:general_update_0} with a feasible starting point $\bsbb^{(0)}$, namely, $\hat \bsbb \in \min_{\|\bsbb\|_0 \le q} g(\bsbb, \hat\bsbb)$ and $f(\hat\bsbb) \le f(\bsbb^{(0)})$ with $\|\bsbb^{(0)}\|_0 \le q$.
Define
{        \begin{equation}
               % \scaleto
                {                P_o(q) = q\log(ep/q).}%{10pt}
\end{equation}}Suppose  that   $\bsbb^{(0)}$ satisfies
{  \begin{equation}
                \begin{split}
                        \EE \Breg_2 ( \bsbb^{(0)},  \bsbb^*) = \mathcal{O}(M) \frac{\sigma^2 P_o(q) + \sigma^2}{n}
                        \mbox{  for some} \ \, M: 1 \le M \le +\infty.
                \end{split}
        \end{equation}}
Let $Q = \{\rho_+(q+s)M/n\}^{1/2} + \{\rho_+^l(q, s) + \eta_0\}M/n$. Assume for some $\delta>0, 0<\linrateparam\le 1$ and large $K\ge 0$,
{%\footnotesize
\begin{equation} \label{regu_feasible_point}
         \begin{split}
                &K\sigma^2P_o(\vartheta s) + \Big\{ 2(1- \frac{1}{M})  \bar\breg_{l_0} + \frac{C}{M(Q\delta \vee 1)} \breg_{l_0} - \delta \Breg_2 \Big\} (\bsbX\bsbb, \bsbX\bsbb') \\
                &\, \ge
                \frac{1 - 1/M }{\sqrt{\vartheta}}\big[\rho - \{(2-\linrateparam)\sqrt{\vartheta} - 1\}\eta_0 \big]\Breg_2(\bsbb, \bsbb'),  \forall  \bsbb, \bsbb': \|\bsbb\|_0 \le \vartheta s, \|\bsbb'\|_0 \le s,
        \end{split}
\end{equation}}
where $C$ is some positive constant. Then
{\begin{equation} \label{feasible_point_bound}
                \begin{split}
                        \EE \big\{\Breg_2 (\bsbX\hat\bsbb, \bsbX\bsbb^*) \vee \frac{\eta_0\linrateparam}{\delta}\Breg_2(\hat\bsbb, \bsbb^*) \big\}
                        \lesssim
                        \frac{K\delta \vee 1}{\delta^2} \Big\{\sigma^2 \vartheta s \log\big(\frac{ep}{\vartheta s}\big) +\sigma^2\Big\} + \frac{\eta_0}{\delta\linrateparam}\|{\bsbb}^{*}\|_2^2.
                \end{split}
\end{equation}
}
\end{theorem}

Therefore, we can achieve the desired level of statistical accuracy as long as $K, \delta, \vartheta$ are constants and $\eta_0$ is not excessively large. When $M=+\infty$ (no requirement on $\bsbb^{(0)}$), the regularity condition \eqref{regu_feasible_point} becomes
{\begin{equation*} %\label{regu_feasible_point-degenerate}
        \begin{split}
                \frac{\rho - \{(2-\linrateparam)\sqrt{\vartheta} - 1\}\eta_0}{\sqrt{\vartheta}} \Breg_2(\bsbb, \bsbb')
                \le
                \big( 2  \bar\breg_{l_0} - \delta \Breg_2 \big) (\bsbX\bsbb, \bsbX\bsbb') + K\sigma^2P_o(\vartheta s), \forall  \bsbb, \bsbb': \|\bsbb\|_0 \le \vartheta s, \|\bsbb'\|_0 \le s,
        \end{split}
\end{equation*}}which includes  \eqref{eq:assumption_r0} as a special case. But when one uses a {decent} starting point, \eqref{regu_feasible_point} is much more relaxed. In the  extreme case where $M=1$,    the right-hand side of \eqref{regu_feasible_point}  becomes $0$, and  so with $\mu$-restricted strong convexity  $(\bar\breg_{l_0} - \mu \Breg_2)   (\bsbX\bsbb, \bsbX\bsbb')\ge 0$ for $\|\bsbb\|_0 \le \vartheta s, \|\bsbb'\|_0 \le s$,    \eqref{regu_feasible_point} is  always satisfied.
\subsection{Slow kill: algorithm design  \& sequential analysis}
Using a multi-start strategy to select a high-quality initial value for $\bsbb^{(0)}$ may be computationally infeasible for large-scale data. Fortunately,  we will see that {designing iteration-varying} thresholding and shrinkage can effectively relax the statistical regularity conditions and   improve the statistical accuracy of the   sequence of iterates.

More concretely, slow kill modifies the  optimization algorithm \eqref{eq:general_update_0} by introducing {three} auxiliary sequences $\rho_{t+1}, q_{t+1}, \eta_{t+1}$
\begin{align} \label{eq:general_update}
\bsbb^{(t+1)} = \Theta^{\#}\Big\{\bsbb^{(t)} - \rho_{t+1}^{-1} \bsbX^T\nabla l_0(\bsbX\bsbb^{(t)}; \bsby);
q_{t+1}, \bar\eta_{t+1}\Big\},   \mbox{ with } \bar\eta_{t+1} = \eta_{t+1}/\rho_{t+1}
\end{align}
where $q_t \to q$, $\eta_t \to \eta_0$. The scaled shrinkage sequence $\bar\eta_t$  will be more convenient to use  than the raw sequence  $\eta_t$  in  later analysis.
We want to understand whether adapting the inverse learning rate, cardinality, and $\ell_2$-shrinkage parameters during the iteration can lead to improved performance.
Specifically, we aim to investigate how the statistical accuracy of $\bsbb^{(t)}$ changes as $t$ increases, and under what conditions the statistical error converges geometrically fast.
%Specifically, how will the statistical accuracy of $\bsbb^{(t)}$   change as $t$ increases? Under what conditions will the statistical error converge geometrically fast?
The focus of Theorem \ref{th:iter} is on the statistical error of $\bsbb^{(t)}$ with respect to the statistical truth $\bsbb^*$, rather than on their optimization errors relative to a specific minimizer $\bsbb^o$.   We will see that in principle, slow kill benefits from  decreasing $q_t$ and $\rho_t$. % but increasing $\bar\eta_t$.[delete??]
It is also worth noting that the error bound in  \eqref{eq:seq_result}  places no requirements on $\vartheta_t, \rho_t, \eta_t$.  \begin{theorem} \label{th:iter}
Let the sequence of iterates $\bsbb^{(t)}: \|\bsbb^{(t)}\|_0 = q_t$ be generated from \eqref{eq:general_update} with a feasible $\bsbb^{(0)}$.
Given any $t \ge 1$, define
{ %\small
        \begin{align} \label{kappa_h}
                        h_{t}^{-1} &= (1 - 1/\sqrt{\vartheta_{t}})(\rho_{t}+\eta_{t}) + (1-\linrateparam)(\rho_-^{l}(q_{t}, s) + \eta_{t}), \\
                        \kappa_{t} &= (\rho_{t} - \rho^l_-(s, q_{t}) )h_{t},
        \end{align}
}
where  $\linrateparam$ is an arbitrary number in $  (0,1]$. Then the following recursive statistical error bound
\begin{align} \label{eq:seq_result}
                &\qquad \Breg_2(\bsbb^*, \bsbb^{(T+1)}) + \sum_{t=0}^T \big(\Pi_{\tau = t}^{T} h_{\tau+1}\big) (\rho_{t+1} \Breg_2 - \breg_{l})(\bsbb^{(t+1)}, \bsbb^{(t)}) \nonumber\\
                & \ \le   \sum_{t=0}^T \big( \kappa_{t+1}\cdots\kappa_{T+1}  \big) \bigg\{ \frac{A \sigma^2}{\linrateparam} \frac{\rho_+(q_{t+1} + s)}{\big(\frac{\rho_-^l(q_{t+1}, s)}{\rho_{t+1}} \vee \bar\eta_{t+1} \big) \big(1 - \frac{\rho^l_-(s, q_{t+1})}{\rho_{t+1}}\big) \rho_{t+1}^2  } \cdot  \vartheta_{t+1}s\log\big(\frac{ep}{\vartheta_{t+1} s}\big) \nonumber \\
                & \qquad\qquad\qquad\qquad\quad\quad + \frac{\bar\eta_{t+1}}{\big(1 - \frac{\rho^l_-(s, q_{t+1})}{\rho_{t+1}}\big)\linrateparam}\|\bsbb^*\|_2^2 \bigg\} + \bigg( \Pi_{t = 0}^{T} \kappa_{t+1}  \bigg)\Breg_2(\bsbb^*, \bsbb^{(0)}).
        \end{align}{holds for all $T\ge 0$, with probability at least $1 - Cp^{-cA}$, where $ C, c$ are positive constants.}
\end{theorem}

The   corollary   below showcases the usefulness of the theorem on   algorithm configuration.
\begin{corollary} \label{co:iter}
In the setup of Theorem \ref{th:iter}, given any $\linrateparam \in (0,1]$, if $\rho_t$ and $\eta_t$ are chosen to satisfy
{%\small
\begin{align}
                &\qquad\qquad\qquad\rho_{t+1} \ge \rho^l_+(q_{t+1}, q_t) \label{eq:rhotchoice}
                \\
                & \bar\eta_t \ge 0 \vee
                \frac{ (1/\sqrt{\vartheta_{t}}+\linrateparam) - 2(\rho^l_-(s, q_{t}) \wedge \rho^l_-(q_{t}, s))/\rho_t }{2 - 1/\sqrt{\vartheta_{t}}  - \linrateparam}
                \label{eq:etabarchoice}
\end{align}}so that ($\rho_{t+1} \Breg_2 - \breg_{l})(\bsbb^{(t+1)}, \bsbb^{(t)} )\ge 0$ and $ \kappa_{t} \le (1+\linrateparam)^{-1}$, then with probability at least $1 - Cp^{-cA}$   the statistical error of $\{\bsbb^{(t)}\}$ decays  \emph{geometrically} fast,
{%\small
        \begin{equation} \label{eq:seq_result2}
                \begin{split}
                \Breg_2(\bsbb^*, \bsbb^{(T+1)}) \le
                \left(\frac{1}{1+\linrateparam}\right)^{T+1} \Breg_2(\bsbb^*, \bsbb^{(0)}) +
                 \frac{1}{\linrateparam}  \sum_{t=0}^T \left(\frac{1}{1+\linrateparam}\right)^{T-t+1} E_{t+1}
                \end{split}
\end{equation}}for all $ T \ge 0$, {where}
{%\small
        \begin{equation} \label{Errtdef}
                \begin{split}
                        E_{t+1} = \big\{ 1 - \frac{\rho^l_-(s, q_{t+1})}{\rho^l_+(q_{t+1}, q_t)}\big\}^{-1} \bigg\{ \frac{A \sigma^2}{\frac{\rho_-^l(q_{t+1}, s)}{\rho^l_+(q_{t+1}, q_t)} \vee \bar\eta_{t+1} }
                         \frac{\rho_+(q_{t+1} + s)}{(\rho^l_+(q_{t+1}, q_t))^2}  \vartheta_{t+1}s\log\big(\frac{ep}{\vartheta_{t+1} s}\big)  + \bar\eta_{t+1}\|\bsbb^*\|_2^2 \bigg\}.
                \end{split}
        \end{equation}
}
\end{corollary}

The theoretical results  provide valuable insights into the design  of the three main elements of slow kill.
Let's  first apply Theorem \ref{th:iter} to  analyze the basic optimization algorithm with fixed quantiles $q_t \equiv q$ and  universal values $\rho_t \equiv \rho, \bar\eta_t \equiv \bar\eta$.   \eqref{eq:seq_result2} then shows linear convergence of the statistical error, with the first term on the right-hand side indicating the impact of the initial point.   %This is a nontrivial conclusion because $\Theta^{\#}$ is not even a nonexpansive operator as $p\gg n$, let alone a contraction.
Because $\Sigma_{t=0}^{T}\{1/(1+\linrateparam)\}^{T-t+1}\le 1/\linrateparam$, the final error %, though unable to reach zero due to the existence of the noise,
is of the order
\begin{equation} \label{seq_error}
\frac{\rho_+(q + s)}{(\rho^l_+(q, q))^2} \, \sigma^2\vartheta s\log\big(\frac{ep}{\vartheta s}\big)  + \bar\eta\|\bsbb^*\|_2^2,
\end{equation}
where the restricted condition number  $\rho^l_+(q,q)/\{\rho^l_-(s, q) \wedge \rho^l_-(q, s)\} $ and    $\linrateparam$ are assumed to be constants.
The lower bound derived in \eqref{eq:etabarchoice} can help  reduce the bias,  and suggests the benefit of using a   large quantile in this regard.  %On the other hand,  when $\eta=0$, the condition   $ 2 \sqrt \vartheta  >   {\rho^l_+(q , q)}/\{\rho^l_-(s, q ) \wedge \rho^l_-(q , s)  \} $  may be easy to violate in the presence of large coherence and/or large $p$,  as discussed previously.]

%To rectify the  situation,    \eqref{eq:etabarchoice} suggests the need of  applying a properly large  $\vartheta_t$.

On the other hand, large quantiles  can lead to an inflated variance term   $  \vartheta_{t+1}s\allowbreak\log\{{ep}/({\vartheta_{t+1} s})\}$ in \eqref{Errtdef},    which motivates the use of  {decreasing} quantiles, the most distinctive feature of slow kill. Indeed, a more careful examination of \eqref{eq:seq_result2}  shows that the factor
$%\begin{align}
{1}/({1+\linrateparam})^{T-t+1}%\label{multfactor}
$ %\end{align}
allows for much larger $q_t$   to be used in earlier iteration steps. This is because for small      $t$,   the associated error $E_{t+1}$ will be more heavily shrunk    in the final   bound. Although it can be difficult to theoretically derive the optimal cooling scheme for the sequence ${q_t}$,     various schemes seem to perform well in practice, such as $q_{t+1} = \lfloor q+(T-t)/(aTt + bT) \rfloor$ (inverse) or $\lfloor q + (p-q)/{1+a\exp(bt/T)}^{c} \rfloor$ (sigmoidal), among others.

After $q_t$ is given, the choice of $\rho_{t+1}$ can be determined theoretically using  \eqref{eq:rhotchoice}:   $\rho_{t+1} \geq \rho_{+}^l(q_{t+1}, q_t)$,  which gives an upper bound of the stepsize to prevent slow kill from diverging.
In implementation,  $\rho_+^l(q_{t+1}, q_t)$ is often unknown.  With regular design matrices (such as Toeplitz), a constant multiple of $L\{n + q_{t+1}\log(ep/q_{t+1})\}$ can be employed based on   \eqref{rmt-ineq} in the proof of Appendix  \ref{rmt-proof}, assuming that $\nabla l_0$ is $L$-Lipschitz continuous.
More generally, seen from the second term on the left-hand side of  \eqref{eq:seq_result}, we  can use a  line search % \citep{Boyd2004} %(cf. Appendix \ref{appendix_impldetail})
with    criterion
\begin{equation} \label{eq:line_search}
(\rho_{t+1} \Breg_2 - \breg_{l})(\bsbb^{(t+1)}, \bsbb^{(t)}) \geq 0.
\end{equation}
% \eqref{eq:line_search} only involves the loss but no penalty.
See   Appendix \ref{appendix_impldetail} for some implementation details of the line search.   \eqref{eq:line_search} enforces the majorization condition at $(\bsbb^{(t+1)}$, $\bsbb^{(t)})$,   and so  the resulting $\rho_{t+1}$ can be even smaller than $\rho^l_+(q_{t+1}, q_t)$. The importance  of limiting the size of $\rho_t$   was previously discussed in Section \ref{subsec:newmeans} for   $\ell_0$-constrained regression.
Similarly, having a smaller $\rho_{t+1}$ can help achieve a larger $\linrateparam$, which in turn leads to faster convergence and smaller error, as demonstrated   in  \eqref{kappa_h} and \eqref{eq:etabarchoice}.
%[Because the quantile-thresholding step \eqref{eq:general_update} is extremely cost effective, line search does not incur much additional cost per-iteration.]

The lower bound for the scaled $\ell_2$-shrinkage sequence $\bar\eta_{t}$ in Corollary \ref{co:iter} can be rewritten as
\begin{equation} \label{rho_condition_general-l}
2 \sqrt\vartheta_{t} > \frac{\rho_{t+1}+ \bar \eta_t \rho_t}{\rho^l_-(s, q_{t}) \wedge \rho^l_-(q_{t}, s) +  \bar \eta_t \rho_t}.
\end{equation}
It is similar to a    restricted condition number condition, and  extends  \eqref{theta_condition} to a general loss.
%[A positive sequence of  $\bar\eta_t$ helps to secure both statistical and computational properties of slow kill.]
Specifically, when $2q_t>n$,  \eqref{eq:etabarchoice} implies   $\bar\eta_t>
(  1/\sqrt{\vartheta_{t}} )/( 2 - 1/\sqrt{\vartheta_{t}}) =  {  1  }/({2  \sqrt{\vartheta_{t}}-1})  $, and as a result,   we recommend using  a scaled shrinkage sequence defined by
\begin{align}
\bar\eta_t = 1/(2\sqrt{q_t/\bar s} -1), \label{etabarchoiceforlargeq}
\end{align}
where  $\bar s = q \wedge n L^2/\log(ep)$ (a surrogate for $s$, according to Appendix \ref{appsec:proofseq}) and $L$ is the Lipschitz parameter of $\nabla l_0$. \eqref{etabarchoiceforlargeq}   plays an important role in early slow kill iterations and is independent of the learning rate.
% for large     $q_t$ (possibly greater than $n$).

Our analyses
support the use of the $\ell_2$-assisted {backward} quantile control to gradually tighten the constraint.
The  update formula \eqref{eq:general_update} used in slow kill   has a strong foundation in optimization, which gives it an advantage over heuristics based multi-stage procedures. The fast geometric convergence established in Theorem \ref{th:iter}, together with a strong signal strength, indicates that the zeros in $\bsbb^{(t)}$ represent irrelevant predictors with high probability (cf. Remark \ref{remark1} and Appendix \ref{subsec:recurcoorderror}). This allows us to  occasionally squeeze the design matrix  using ${\mathcal{J}(\bsbb^{(t+1)})}$     (e.g., when $q_{t+1}$ reaches $p/2^k$) to   reduce the problem size (Appendix \ref{appendix_impldetail}).
The apparent junk features are thus removed at an early stage, saving computational cost,   while the more difficult to identify irrelevant features are  addressed only when we are close to finding an optimal solution.
This trait makes slow kill particularly well-suited for big data learning. Slow kill offers similar advantages in group variable selection \citep{she2013group} and low-rank matrix estimation \citep{She2013Mat}.

In contrast,  forward pathwise and boosting algorithms \citep{buhlmann2003boosting,efron2004least,needell2009cosamp,ZhangIT11,zhang2013multi,wang2014optimal,zhao2018pathwise}  grow a model from the null in   a \emph{bottom-up} fashion. Such  algorithms   must consider almost all features at each iteration, making them computationally intensive, as they often require hundreds or thousands of boosting iterations.
Motivated by the $\ell_0$-optimization perspective, we can also investigate a class of ``steady grow" procedures in which $q_t$ increases from $0$ to $q$ in \eqref{eq:general_update}. Compared with boosting, the update and selection would incorporate the effect of the previous estimate in addition to the gradient. A retaining option can be introduced in steady grow that works in the opposite way to the squeezing operation in slow kill. The investigation of retaining and squeezing, as well as a combination of slow kill and steady grow, is left for future research.

Finally, how to obtain a sparse model with  a prescribed cardinality is the  problem of interest throughout the  paper. But   if one wants to determine the best  value for $q$, we suggest using a predictive information criterion \citep{SheCV} that can guarantee the optimal prediction error rate in a nonasymptotic sense (which  is presented  in Appendix \ref{appsub:tuning}).

%%%%%%%%%%%%%%%%%%%%%%%%%%%%%%%%%%%%%%%%%%%%%%%%%%%%%%%%%%%%%%%%%%%%%%%%%%%%%%%
%%%%%%%%%%%%%%%%%%%%%%%%%%%%%%%%%%%  Section 5
%%%%%%%%%%%%%%%%%%%%%%%%%%%%%%%%%%%%%%%%%%%%%%%%%%%%%%%%%%%%%%%%%%%%%%%%%%%%%%%

\section{Experiments} \label{experiments}
\subsection{Simulations}
\label{subsec:simu}
In this part, we conduct simulation studies to compare the performance of slow kill (abbreviated as SK in tables and figures below) with some popular sparse learning methods in terms of prediction accuracy, selection consistency, and computational efficiency.  Unless otherwise mentioned, the rows $\tilde\bsbx_i^T$ of the predictor matrix $\bsbX = [\tilde \bsbx_1, \dots, \tilde \bsbx_n]^T \in \mathbb R^{n\times p}$ are independently generated from a multivariate normal distribution with covariance matrix $\bsbSig$, where $\bsbSig$ either has a Toeplitz structure $[\tau^{|i-j|}]$ or has equal correlations $[\tau 1_{i\neq j}]$. High correlation strengths such as $\tau=0.9$ will be included in our experiments. We consider both regression and classification with a sparse $\bsbb^*$: $\beta^*_j = 1,$ if $j = 10k + 1, 0 \le k < s$  and so $s = \|\bsbb^*\|_0$.
In the regression experiments, $\bsby=\bsbX\bsbb^* +\bsbeps$ with $\epsilon_i\sim N(0,1)$, and for the classification experiments, $y_i = 1$ if $\tilde \bsbx^T_{i} \bsbb^* > 0$ and 0 otherwise.

In addition to slow kill, the following methods are included for comparison: lasso \citep{tibshirani1996}, elastic net (ENET)  \citep{zou2005regularization}, MCP  \citep{zhang2010}, SCAD  \citep{fan2001variable}, and IHT and NIHT     (\cite{blumensath2009iterative,blumensath2010normalized}, for regression only). (We also evaluated the performance of picasso \citep{zhao2018pathwise} in simulations as an improved version of \cite{wang2014optimal}. However, its pathwise computation  resulted in worse error rates and missing rates than  standard nonconvex optimization on the synthetic data. Therefore, we did not present the results. We will include the algorithm in our experiments with real data in later sections.)  The quadratic loss is used in regression and the logistic deviance is used in classification.
For slow kill,    we take a simple single starting point
$\bsbb^{(0)}  = \bsb0 $   and   $\eta_0=50$; an inverse cooling schedule
$q_{t+1} = \lfloor q+(T-t)/\{tT/(p-q)+ 2T/(p-2q)\} \rfloor  \ (0 \le t \le T)$ is used so that $q_T = q$ and  $q_1 = p/2$, and we set $T = 100$ in all experiments for convenience and efficiency.
We use the R package glmnet to implement lasso and elastic net, the package ncpen \citep{kim2018unified} for the aforementioned nonconvex penalties, and the package sparsify for IHT methods. (The core of glmnet is implemented using Fortran subroutines, while ncpen is mainly based on C++. Our implementation of slow kill could potentially be made more efficient and require less memory by using C or Fortran, but it already performs comparably or better than the other methods, as shown in later tables and figures.)
To ensure a fair comparison and eliminate the influence of different parameter tuning schemes, we select the estimate with 1.5s nonzeros for each method. To calibrate the bias, we refit each obtained model using only the selected variables. All  other algorithmic parameters are set to their default values.

Given each simulation setup, we repeat the experiment for 50 times and evaluate the performance of each algorithm according to the measures defined below: the missing rate $\times 100\%$  and the prediction error. Concretely, the missing rate is the fraction of undetected true variables, and in regression, the prediction error is calculated by $10$ times $ (\hat\bsbb - \bsbb^*)^T\bsbSig(\hat\bsbb - \bsbb^*)$ using the true signal, while in classification, it refers to the misclassification error rate $\times 100\%$ on a separate test set containing the same number of observations as the training dataset. The total computational time (in seconds) is also included to describe the computational cost. Since the implementation of a penalized method  often uses  warm starts,   we terminate  the algorithm once it reaches an estimate with the prescribed cardinality.

Table \ref{table-exp1} shows some experiment results in the regression setup.  Figure \ref{figure-exp1} plots more results of some representative methods when  varying  the sparsity level $s$ and the correlation strength $\tau$ (excluding elastic net and IHT, because  their performance is similar to that of lasso and poor, respectively). It can be seen that slow kill outperforms the other methods in terms of both statistical accuracy and computational time, particularly in more challenging situations with more relevant features and coherent designs.
%\newpage
\begin{table}[h]
\caption{\label{table-exp1} {\fontsize{8.5}{8.5}\selectfont  Regression: performance comparison in terms of prediction error, missing rate and computational time with different correlation structures. In more details, $p=5\mbox{,}000, n=150, s = 10$ and $\bsbSig =[\tau^{|i-j|}]$ or $[\tau 1_{i\neq j}]$ with $\tau = 0.9$ } }
\begin{center}
        {\fontsize{8.5}{10}\selectfont
                \setlength{\tabcolsep}{1.1mm}
                \begin{tabular*}{0.8\linewidth}{@{\extracolsep{\fill}} l rrr r rrr}

                        & \multicolumn{3}{c}{Toeplitz structure} && \multicolumn{3}{c}{Equal correlation}\\
                        \cmidrule(lr){2-4} \cmidrule(lr){6-8}
                        & \multicolumn{1}{c}{Error} & \multicolumn{1}{c}{Miss} & \multicolumn{1}{c}{Time} &&\multicolumn{1}{c}{Error} & \multicolumn{1}{c}{Miss} & \multicolumn{1}{c}{Time} \\
                        LASSO       & 16&32 & 5  && 15& 83& 13\\
                        ENET        & 16&31 & 13 && 14& 82& 34\\
                        IHT         & 85&68 & 55 && 16& 88& 57\\
                        NIHT        & 12&22 & 4  && 17& 80& 18\\
                        MCP         & 12&23 & 34 && 18& 78& 24\\
                        SCAD        & 12&23 & 13 && 16& 85& 6 \\
                        SK          & 2 &2  & 1  && 12& 50& 1 \\

                        \hline
                \end{tabular*}
        }
\end{center}
%{\tiny ENET, elastic net;  IHT, iterative hard thresholding; NIHT, Normalised Iterative Hard Thresholding; MCP, minimax concave penalty; SCAD, smoothly clipped absolute deviation penalty; SK, slow kill.}
\end{table}

\begin{figure*}[h]
%\vspace{-8mm}
\centering
\includegraphics[width=0.75\textwidth]{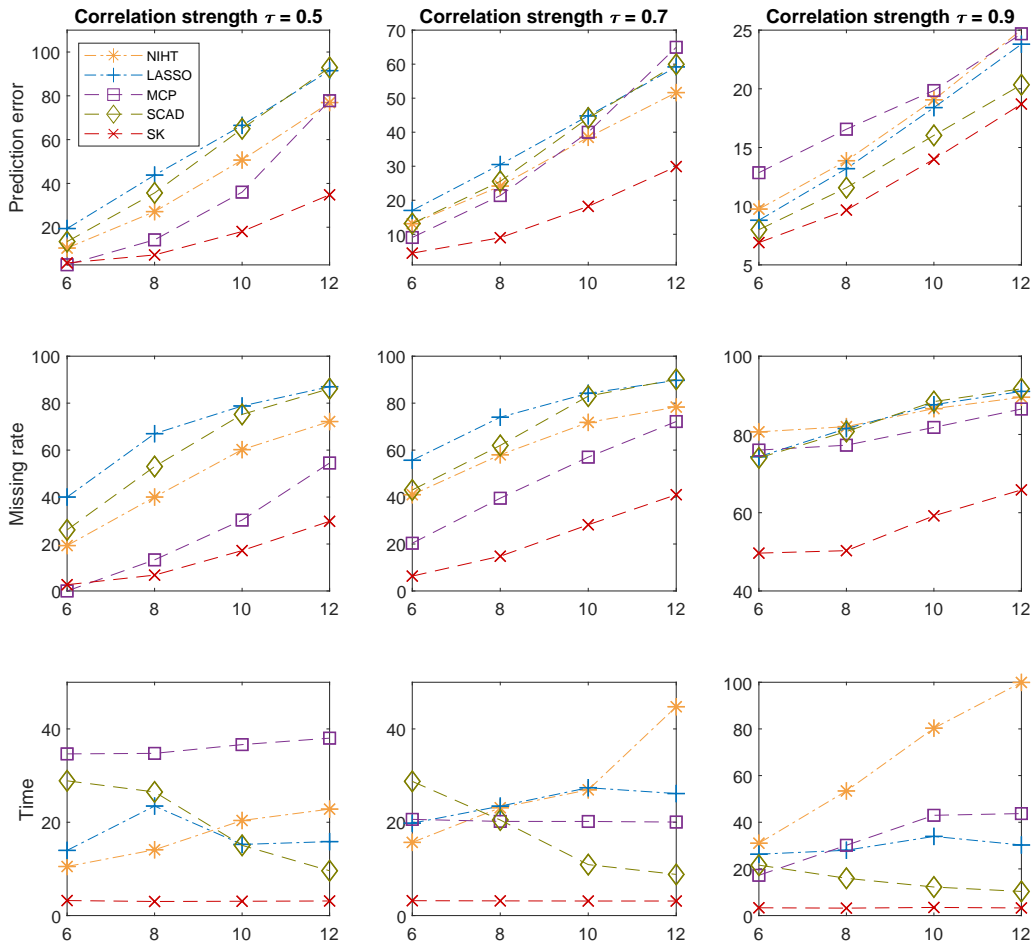}
\caption{{\fontsize{8.5}{10}\selectfont Regression: performance comparison in terms of prediction error, missing rate and computational time when varying the sparsity and the correlation strength of the model. In more details, $p = 10\mbox{,}000, n=150, s =6, 8, 10, 12$ and $\bsbSig = [\tau 1_{i\neq j}]$ with $\tau =0.5, 0.7, 0.9$.} }
\label{figure-exp1}
\end{figure*}

For classification, Table \ref{table-exp2-2} and Figure \ref{figure-exp2} make a comparison between different methods with various correlation structures and problem dimensions, and similar conclusions can be drawn. It is important to note that the excellent statistical accuracy of slow kill is \textit{not} accompanied by a sacrifice in computational time  compared to other methods.
In fact,   as seen in  Figure \ref{figure-exp2}, slow kill offers substantial time savings especially when $n$ is large, while being very successful at selection and prediction.

\begin{table}
\begin{center}
        \caption{\label{table-exp2-2} {\fontsize{8.5}{8.5}\selectfont Classification: performance comparison in terms of prediction error, missing rate and computational time with different correlation structures. In more details, $p= 2\mbox{,}000, n= 500, s = 10$ and $\bsbSig=[\tau^{|i-j|}]$ or $[\tau 1_{i\neq j}]$ with $\tau = 0.9$ } }
        {\fontsize{8.5}{10}\selectfont
                \setlength{\tabcolsep}{1.1mm}
                \begin{tabular*}{0.8\linewidth}{@{\extracolsep{\fill}} l rrr r rrr}

                        & \multicolumn{3}{c}{Toeplitz structure} && \multicolumn{3}{c}{Equal correlation}\\
                        \cmidrule(lr){2-4} \cmidrule(lr){6-8}
                        & \multicolumn{1}{c}{Error} & \multicolumn{1}{c}{Miss} & \multicolumn{1}{c}{Time} &&\multicolumn{1}{c}{Error} & \multicolumn{1}{c}{Miss} & \multicolumn{1}{c}{Time} \\
                        LASSO        & 8.0 & 24 & 10 && 5.1& 95 & 49\\
                        ENET         & 8.0 & 25 & 31 && 4.7& 95 &135\\
                        MCP          & 6.9 & 23 & 15 && 5.0& 93 & 20\\
                        SCAD         & 7.0 & 22 & 22 && 5.1& 94 & 16\\
                        SK           & 2.2 & 2  & 4  && 3.9& 78 & 4 \\
                        \hline
                \end{tabular*}
}\end{center}
%{\tiny ENET, elastic net;  MCP, minimax concave penalty; SCAD, smoothly clipped absolute deviation penalty; SK, slow kill.}
\end{table}

\begin{figure*}[h!]
%\vspace{-4mm}
\centering
\includegraphics[width=0.75\textwidth]{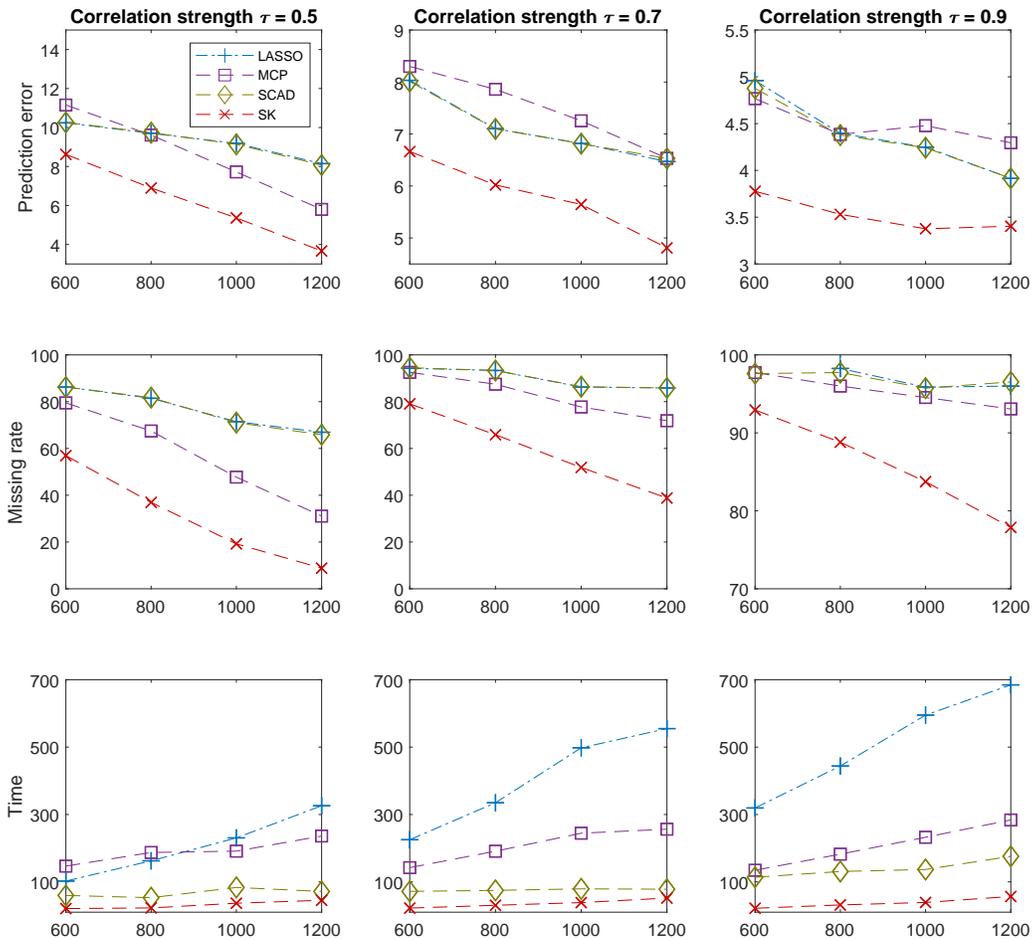}
\caption{{\fontsize{8.5}{8.5}\selectfont  Classification: performance comparison in terms of prediction error, missing rate and computational time with different correlation structures and sample sizes. In more details, $p=10\mbox{,}000, n=600, 800, 1000, 1200, s = 15$ and $[\tau 1_{i\neq j}]$ with $\tau = 0.5, 0.7, 0.9$. } }
\label{figure-exp2}
\end{figure*}

%For additional simulation results, interested readers may refer to Appendix \ref{subsec:moreexps}.

Next, we present some experiments in which the signal strength is varied. Recall that in the regression setup, we set $\beta_j^*=1$ for $j\in \mathcal J(\bsbb^*)$. For  $n=100, p=5000, \sigma=1$, the minimax  optimal rate is approximately    $\sigma \sqrt {(\log p) /  n } (\approx 0.292)$ (ignoring the constant factor for which a sharp value may be difficult to derive). We conducted additional  experiments by setting $\beta_j^* =  0.8, 0.6, 0.4, 0.2$. The comparison results for different methods are demonstrated in Figure \ref{figure-exp-betastrength}.
As the signal strength was low (e.g., $\beta_j^*=0.2, 0.4$), all methods performed poorly. For higher values,   slow kill outperformed the other methods by a large margin.

\begin{figure}[htp!]
        \centering
        \includegraphics[width=0.8\textwidth]{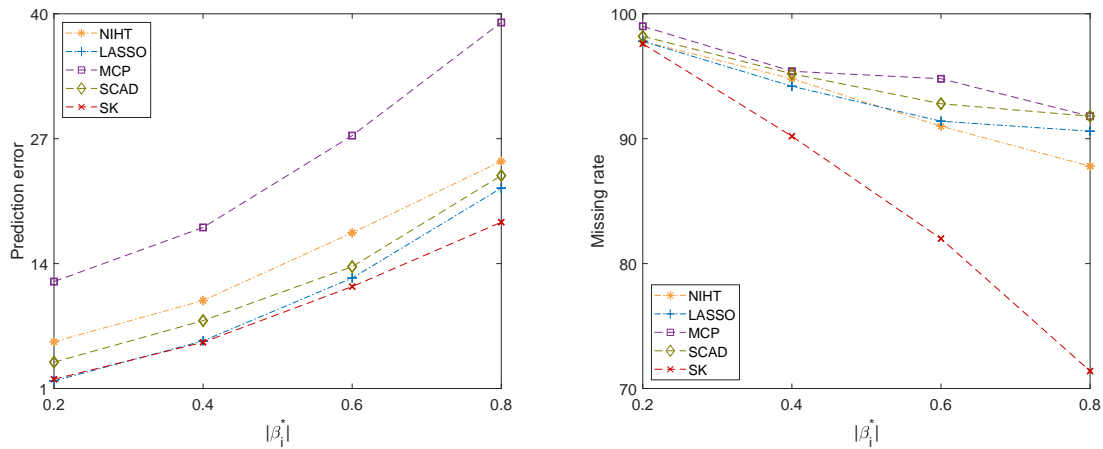}
        \caption{{\fontsize{8.5}{8.5}\selectfont  Comparison of prediction errors (left) and missing rates (right) of different methods under different signal strengths. The details of the regression setup are given in Section  \ref{subsec:simu},  and we set   $p=5\mbox{,}000, n=100, s = 10$, $ \tau = 0.8$, and   $\bsbb_j^* = 0.2, 0.4, 0.6, 0.8$ for $j\in\mathcal J(\bsbb^*)$.} }
        \label{figure-exp-betastrength}
\end{figure}

We conducted another  experiment to explore larger values of $\| \bsbb^*\|_2^2$. (As a reminder, in the previous setting where $s=10$ and $ \beta_j^*=1$, $\forall j\in \mathcal J(\bsbb^*)$, we had    $\| \bsbb^*\|_2^2=10$.) We tested $\| \bsbb^*\|_2^2= 50, 100, 150, 200$ by scaling up each $\beta_j^*$ by a corresponding factor.  The results of this experiment are shown in Figure \ref{figure-exp-betanorm}. As $\|\bsbb^*\|_2^2$ increases,  NIHT, MCP, and slow kill   exhibit clear advantages, with  the latter two showing similar prediction errors and missing rates.

\begin{figure}[h!]
%        \vspace{-4mm}
        \centering
        \includegraphics[width=0.8\textwidth]{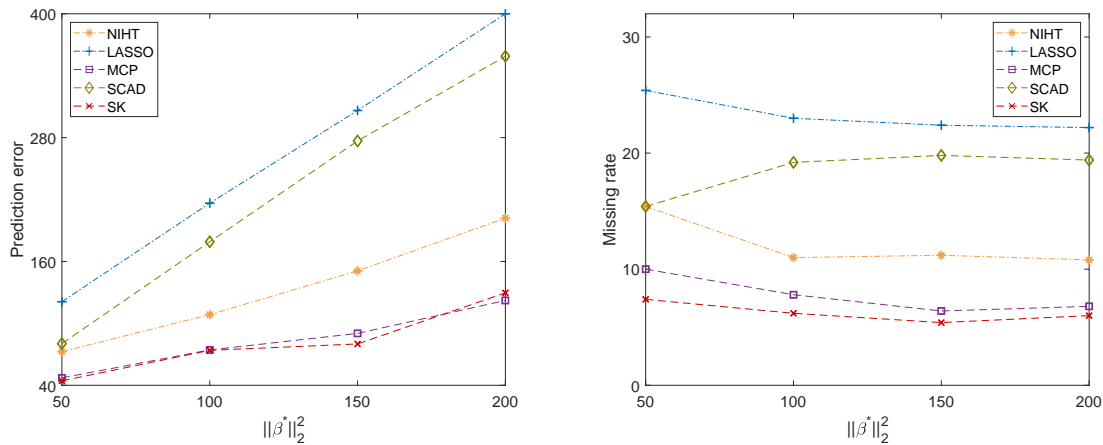}
        \caption{{\fontsize{8.5}{8.5}\selectfont  Comparison of prediction errors (left) and missing rates (right) of different methods for large signals. The details of the regression setup are given in Section  \ref{subsec:simu},  and we set   $p=5\mbox{,}000, n=100, s = 10$, $ \tau = 0.8$, and  $\|\bsbb^*\|_2^2 = 50, 100, 150, 200$ (by scaling up each $\beta_j^*$).} }
        \label{figure-exp-betanorm}
\end{figure}

\subsection{Handwritten digits classification}
The Gisette dataset \citep{Guyon2004result} was created to classify the highly confusing digits 4 and 9 for handwritten digit recognition. There are 5,000 predictors, including various pixel constructed features as well as some `probes' with little predictive power. Because the exact number of relevant features is unknown, we assess the performance of different methods given the same model cardinality to make a fair comparison. We randomly split the 7,000 samples into a training subset with 3,000 samples and a test subset with 4,000 samples for 20 times to report the average misclassification error rate and total computational time.

Due to  the relatively large size of the data, computational efficiency is a major concern. Many statistical packages were unable to deliver meaningful results in a reasonable amount of time.   Here, we compare the glmnet  \citep{Friedman2010},  logitboost \citep{friedman2000additive,scikit-learn}, picasso   with the MCP option \citep{ge2019picasso}, and slow kill with different numbers of selected features.

\begin{figure}[h!]
\centering
        \includegraphics[width=0.8\textwidth]{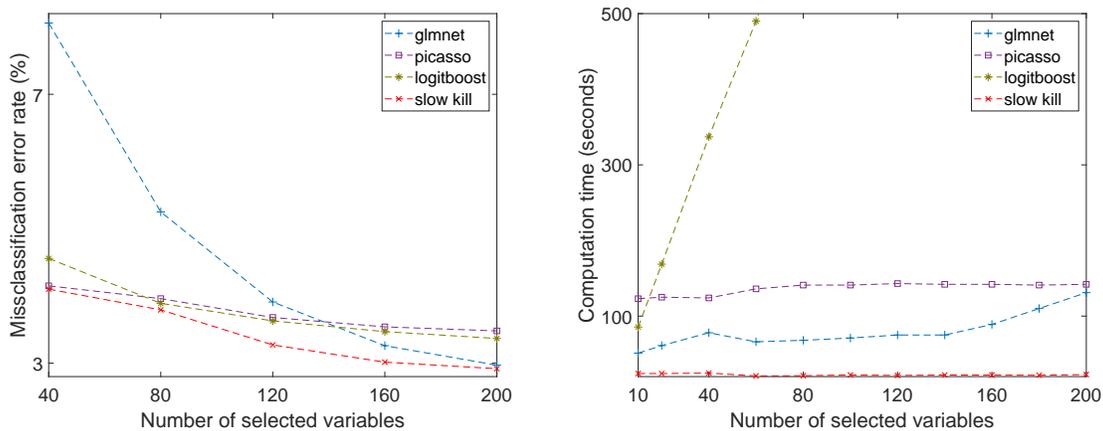}
\caption{{\fontsize{8.5}{8.5}\selectfont Gisette data. Left panel: mean misclassification error rate, right panel:  total computational time, with different numbers of selected features. Picasso is too costly compared with the other methods and only part of its cost curve is shown. }}
\label{figure:real-data}
\end{figure}

According to Figure \ref{figure:real-data}, logitboost and picasso   achieved better misclassification error rates on the dataset than glmnet, but slow kill consistently performed  the best.
In terms of computational cost, glmnet and slow kill were extremely scalable;  logitboost was quite expensive even for just $q = 40$, and picasso  suffered a similar issue when $q \ge 60$.

\subsection{Breast cancer microarray data}
The breast-cancer microarray dataset \citep{feltes2019cumida} from the Curated Microarray Database contains 35,981 gene expression levels of 143 tumor samples of patients with breast cancer and 146 paired adjacent normal breast tissue samples.
The goal is to identify some differentially expressed genes to help the classification of normal and tumor tissues.
We randomly split the dataset into a training subset (60\%) and a test subset (40\%) for 20 times and report the  misclassification error rates and total computational time of different methods in Table \ref{table-realdata2-1}. %We also conducted a computational complexity experiment in Table \ref{table-realdata2-2} to compare their cost.

\begin{table}
\begin{center}
        \caption{\label{table-realdata2-1}
                {\fontsize{8.5}{8.5}\selectfont  Breast cancer microarray data:  misclassification error rate ($\times 100\%$) and total computational time (in seconds) } }
        {\fontsize{6.5}{10}\selectfont
                \setlength{\tabcolsep}{1.1mm}

                \begin{tabular*}{0.9\linewidth}{@{\extracolsep{\fill}} l rrrrrrrrrr }
                        \hline
                        & \multicolumn{2}{c}{$q=60$} & \multicolumn{2}{c}{$q=80$} & \multicolumn{2}{c}{$q=100$} & \multicolumn{2}{c}{$q=120$} & \multicolumn{2}{c}{$q=140$} \\
                        \cmidrule(lr){2-3} \cmidrule(lr){4-5} \cmidrule(lr){6-7}  \cmidrule(lr){8-9} \cmidrule(lr){10-11}
                        & Error & Time & Error & Time  & Error & Time & Error & Time & Error & Time \\
                        \hline
                        GLMNET     &10.9&19 & 10.7&19 & 10.5&50 & 10.2&50  & 10.2&50\\
                        PICASSO    &11.4&43 & 11.3&43 & 11.1&48 & 11.3&48  & 11.2&42\\
                        LogitBoost &11.2&500& 11.2&680& 10.9&860& 10.6&1080& 10.8&1220\\
                        SK  &10.8&10 & 10.2&11 & 10.2&11 & 10.1&11  & 9.8 &11\\
                        \hline
                \end{tabular*}
        }
\end{center}
%{\tiny PICASSO, pathwise calibrated sparse shooting algorithm.}
\end{table}

According to Tables \ref{table-realdata2-1}, logitboost has the highest computational complexity,
and picasso  shows the worst overall classification performance on this dataset. In contrast, glmnet and slow kill can achieve lower misclassification error rates, and the latter  is much more cost-effective according to our experiments.

\subsection{Sub-Nyquist spectrum sensing and learning}
\label{subsec:gbsense}
Sub-Nyquist sampling-based wideband spectrum sensing
for millimeter wave is an important topic for   next-generation wireless communication systems.
With a multi-coset sampler \citep{ME09}, a  multiple-measurement-vector model in signal processing can be formulated as  $\bsbY = \bsbX \bsbB ^ * + \bsbE$, where the goal is to exploit the joint  (row-wise) weak sparsity of $\bsbB^*$ to reconstruct  the spectrum. Here, all the matrices are complex (e.g., $\bsbY \in \mathbb C^{n\times m}$, $\bsbX\in \mathbb C^{n\times p}$), and the size of the predictor matrix $\bsbX$ is determined by    the number  of cosets and the number of channels;
interested reader may refer to \cite{song2019real} for more detail. Nicely,  with the Hermitian inner product
$\langle \bsbA, \bsbB\rangle \triangleq \mbox{tr}\{ \bsbA^{\mbox{\tiny H}} \bsbB\}$
in place of the real inner product, and the generalized Bregman function redefined as $\breg_{l}(\bsbB_1, \bsbB_2) = l(\bsbB_1) - l(\bsbB_2) - \langle \nabla l(\bsbB_2), \bsbB_1 - \bsbB_2 \rangle/2 - \langle \bsbB_1 - \bsbB_2,\nabla l(\bsbB_2) \rangle/2 $,  all of our theorems and algorithms can be extended to the complex group sparsity pursuit.

We compared our method with two popular methods, SOMP   \citep{tropp2005simultaneous} and JB-HTP  \citep{qi2019low}, on a benchmark time-domain dataset in \cite{gao2021sub}.   Table \ref{table-gbsensedata} shows  the normalized mean square error $\|\hat \bsbB - \bsbB^*\|_F/\|\bsbB^*\|_F$ of each method as we vary $q$ (the   number of selected channels).
A demonstration of   spectral recovery is plotted in  Figure  \ref{figure:GBsensedata}, where    the predictive information criterion  in Appendix \ref{appsub:tuning}  was used for model selection  in slow kill. % Due to its impressive performance, slow kill ranked first in  the Sub-Nyquist Spectrum Sensing and Learning Challenge in 2021 \citep{Gbsense}.

\begin{table}[htp]
\begin{center}
        \caption{\label{table-gbsensedata}
                {\fontsize{8.5}{8.5}\selectfont Spectrum reconstruction error in terms of normalized mean square error   } }
        {\fontsize{8.5}{10}\selectfont
                \setlength{\tabcolsep}{1.1mm}
                \begin{tabular*}{0.8\linewidth}{@{\extracolsep{\fill}} l cccccc}
                        \hline
                        & $q=3$ & $q=4$ & $q=5$  & $q=6$ & $q=7$ & $q=8$ \\
                        \hline
                        SOMP                    &0.83&0.93&0.82&0.91&0.92&0.94\\
                        JB-HTP                  &0.94&1.00&0.99&0.95&1.07&0.96\\
                        $\mbox{SK}$  &0.74&0.65&0.53&0.38&0.42&0.50\\
                        \hline
                \end{tabular*}
        }
\end{center}
%{\tiny SOMP: simultaneous orthogonal matching pursuit; JB-HTP: joint-block-sparse hard Thresholding pursuit.}
\end{table}

\begin{figure}[h!]
\centering
\includegraphics[width=0.5\textwidth]{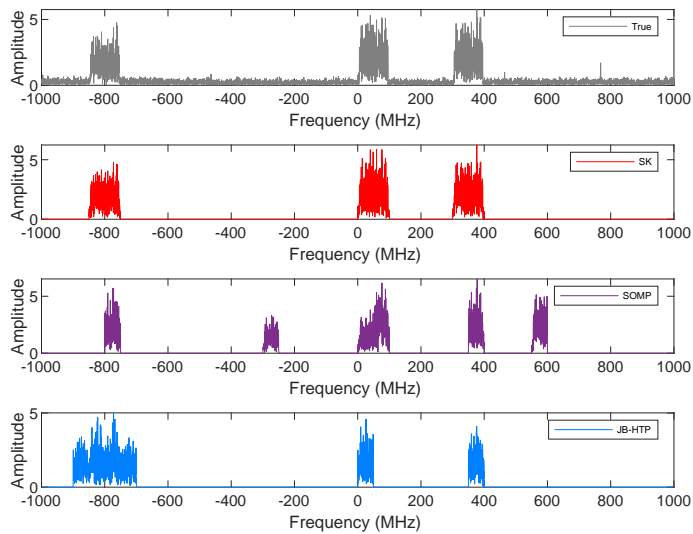}
\caption{{\fontsize{8.5}{8.5}\selectfont Spectrum sensing results by different methods. }}
\label{figure:GBsensedata}
\end{figure}

%%%%%%%%%%%%%%%%%%%%%%%%%%%%%%%%%%%%%%%%%%%%%%%%%%%%%%%%%%%%%%%%%%%%%%%%%%%%%%%
%%%%%%%%%%%%%%%%%%%%%%%%%%%%%%%%%%%  Summary
%%%%%%%%%%%%%%%%%%%%%%%%%%%%%%%%%%%%%%%%%%%%%%%%%%%%%%%%%%%%%%%%%%%%%%%%%%%%%%%
%\vspace{-.1in}
\section{Summary} \label{summary}
This paper proposed a new slow kill method for large-scale variable selection. It is  a scalable optimization-based algorithm that uses  three carefully designed and theoretically justified sequences of thresholds, shrinkage, and learning rates. %\ with iteration-varying thresholds and adaptive shrinkage.

Intuitively, slow kill uses a novel backward quantile control with adaptive $\ell_2$ shrinkage and increasing learning rates to relax regularity conditions and overcome obstacles in backward elimination. This method is significantly different from boosting and many forward stagewise procedures in the existing literature. Our theoretical studies led to insights on how to design a progressive hybrid regularization to achieve the optimal error rate and fast convergence. The technique is applicable to a general loss that is not necessarily a negative log-likelihood function, and its ability to reduce the problem size throughout the iteration makes it attractive for big data.

  % Slow kill can be extended to group variable selection and low-rank matrix estimation, which will be examined in future work.

%\section*{Acknowledgement}
%The authors would like to express their gratitude to the editor and anonymous referees for their valuable suggestions that greatly improved the paper. The authors would also like to thank Dr. Ciprian M. Crainiceanu for suggesting the term slow kill. %This work was supported in part by the National Science Foundation.
%%\section*{Supplementary material} \label{SM} Supplementary material  includes  all technical and implementation details  and computer code.

%%%%%%%%%%%%%%%%%%%%%%%%%%%%%%%%%%%%%%%%%%%%%%%%%%%%%%%%%%%%%%%%%%%%%%%%%%%%%%%
%%%%%%%%%%%%%%%%%%%%%%%%%%%%%%%%%%%  Proofs
%%%%%%%%%%%%%%%%%%%%%%%%%%%%%%%%%%%%%%%%%%%%%%%%%%%%%%%%%%%%%%%%%%%%%%%%%%%%%%%
%\newpage
%\setcounter{page}{1}
\numberwithin{equation}{section}
\numberwithin{theorem}{section}
\numberwithin{lemma}{section}
\numberwithin{definition}{section}
\numberwithin{table}{section}
\numberwithin{figure}{section}

%\appendices
\appendix
\section{Proofs} \label{sec:proofs}
The  definition of a sub-Gaussian random variable or a sub-Gaussian random vector is standard in the literature.
\begin{definition} \label{def:subgauss}
        We call $\xi$ a sub-Gaussian random variable if it has mean zero and the scale ($\psi_2$-norm) for $\xi$,   defined as $ \inf \{\sigma>0: \EE[\exp(\xi^2/\sigma^2)] \leq 2\}$, is finite. We call $\bsbxi\in \mathbb R^p$ a sub-Gaussian random vector with scale  bounded by $\sigma$ if all one-dimensional marginals $\langle \bsbxi, \bsba \rangle$ are sub-Gaussian satisfying $\|\langle\bsbxi, \bsba \rangle\|_{\psi_2}\leq\sigma \|\bsba \|_2$, $\mbox{for any} \ \bsba\in \mathbb R^{p}$.  Similarly, a random matrix $\bsbxi$ is called sub-Gaussian if $\vect(\bsbxi)$ is sub-Gaussian.
\end{definition}

\subsection{Theorem \ref{th:regconvergence} and  Theorem \ref{th:general convergence}}
\label{app:numericalproof}
First, for the algorithm \eqref{eq:update_rule} defined in the setup of Section \ref{iq}, we have the following numerical properties.
\begin{theorem} \label{th:regconvergence}
        Given any $\bsbX, \bsby$ and $\bsbb^{(0)}$, the sequence of iterates $\bsbb^{(t)}$ generated by \eqref{eq:update_rule} satisfies
        $f(\bsbb^{(t)}) - f(\bsbb^{(t+1)}) \ge   {\rho} \|\bsbb^{(t+1)}-\bsbb^{(t)}\|_2^2 /2-  \|\bsbX(\bsbb^{(t+1)} - \bsbb^{(t)})\|_2^2/2, \ \forall t \ge 0$
        and so when $\rho \ge \rho_+(2q)$,  $f(\bsbb^{(t)})$ converges, and $\bsbb^{(t)}$ satisfies
        $$ \min_{0 \le t \le T} \|\bsbb^{(t+1)} - \bsbb^{(t)} \|_2^2 \le  \frac{1}{T+1} \frac{2f(\bsbb^{(0)})}{\rho-\rho_+(2q)}.
        $$
        Moreover, as long as $\rho > \rho_+(2q) $  and $\eta_0> 0$, $\bsbb^{(t)}$ has a  {unique} limit point $\hat \bsbb $
        that satisfies the ``\emph{fixed-point}'' equation
        %\begin{align}
        $$\bsbb= \Theta^{\#}\{\bsbb - \bsbX^T (\bsbX\bsbb - \bsby)/{\rho}; q, \eta_0 / \rho\},$$ %\label{thetafixedpointeqreg}\end{align}
        and when $\|\hat\bsbb\|_0=q$, $\hat\bsbb$ is also a local minimizer of problem \eqref{eq:criterion}.
\end{theorem}

To prove the first conclusion in Theorem \ref{th:regconvergence}, notice that in the regression setting,
$$ g(\bsbb^{(t+1)}, \bsbb^{(t)}) - f(\bsbb^{(t+1)}) = \rho\|\bsbb^{(t+1)}-\bsbb^{(t)}\|_2^2/2 - \|\bsbX(\bsbb^{(t+1)} - \bsbb^{(t)})\|_2^2/2,$$
and thus
$$f(\bsbb^{(t)}) - f(\bsbb^{(t+1)}) \ge \frac{\rho}{2}\|\bsbb^{(t+1)}-\bsbb^{(t)}\|_2^2 -\frac{1}{2}\|\bsbX(\bsbb^{(t+1)} - \bsbb^{(t)})\|_2^2, \ \ \forall t \ge 0.$$
Taking the summation from $t=0$ to $t=T$ and using the fact that $\|\bsbX(\bsbb^{(t+1)} - \bsbb^{(t)})\|_2^2 \le  \rho_+(2q)\|\bsbb^{(t+1)} - \bsbb^{(t)}\|_2^2,$ we have
$$  \frac{(\rho-\rho_+(2q))}{2} \sum_{t=0}^T \|\bsbb^{(t+1)} - \bsbb^{(t)} \|_2^2 \le f(\bsbb^{(0)})-f(\bsbb^{(T+1)}),
$$
which leads to
$$ \min_{0 \le t \le T} \|\bsbb^{(t+1)} - \bsbb^{(t)} \|_2^2 \le \frac{2}{(T+1)(\rho-\rho_+(2q))} f(\bsbb^{(0)}).
$$

Next, we consider the general problem and prove Theorem \ref{th:general convergence}, which implies the second part of Theorem \ref{th:regconvergence}.
From  $\inf_{\bsbxi, \bsby} l_0(\bsbxi;\bsby) > - \infty$, we assume without loss of generality that $l_0(\bsbxi;\bsby) \geq 0$.
Recall $l_0(\bsbX\bsbb;\bsby)$ is abbreviated as $l(\bsbb)$ and thus $\nabla l(\bsbb) = \bsbX^T \nabla l_0(\bsbX \bsbb)$ by the chain rule.

From the construction $g(\bsbb, \bsbb^{(t)}) = f(\bsbb) + (\rho \Breg_2 - \breg_{l_0})(\bsbb, \bsbb^{(t)}) $, we get
\begin{equation*}
        (\rho \Breg_2 - \breg_{l_0})(\bsbb^{(t+1)}, \bsbb^{(t)})+f(\bsbb^{(t+1)}) \leq     g(\bsbb^{(t)}, \bsbb^{(t)})=f(\bsbb^{(t)}).
\end{equation*}
When  $\rho \geq \rho^l_+(q, q)$, $(\rho \Breg_2 - \breg_{l_0})(\bsbb^{(t+1)}, \bsbb^{(t)}) \geq 0$, from which it follows that the sequence of  $f(\bsbb^{(t)})$ is non-increasing and convergent.
In fact, one just needs
\begin{align}
f(\bsbb^{(t+1)}) \leq g(\bsbb^{(t+1)}, \bsbb^{(t)})
\end{align}
to enjoy the function value convergence, which can be used for line search.

In addition, we obtain
$$
(\rho-\rho_{+}^{ l }(q, q))\Breg_2(\bsbb^{(t+1)}, \bsbb^{(t)})\le f(\bsbb^{(t)}) - f(\bsbb^{(t+1)}).
$$

Finally, let us study the limit points of the sequence of iterates. We first notice that $\{ \bsbb^{(t)}\}^\infty_{t=0}$ is uniformly bounded under  $\eta_0>0$, since
$$ \eta_0\|\bsbb^{(t)}\|_2^2/2 \leq f(\bsbb^{(t)})  \leq f(\bsbb^{(0)}).$$
From  $\lim_{t \to \infty} \{f(\bsbb^{(t)}) - f(\bsbb^{(t+1)})\} = 0$, $\lim_{t \to \infty} (\rho \Breg_2 - \breg_{l_0})(\bsbb^{(t+1)}, \bsbb^{(t)})= 0$, and because $\rho > \rho^l_+(q, q)$,
$$\lim_{t \to \infty} (\bsbb^{(t+1)} - \bsbb^{(t)}) = 0.$$
Let $\hat\bsbb$ be any limit point of $\bsbb^{(t)}$ satisfying $\hat\bsbb = \lim_{k \to \infty} \bsbb^{(j_k)}$ for some sequence $j_k$. Then
\begin{align*}
        0 = \lim_{k \to \infty} (\bsbb^{(j_k +1)} - \bsbb^{(j_k)} )&= \lim_{k \to \infty} \Theta^{\#}\{ \bsbb^{(j_k)}  -\nabla l(\bsbb^{(j_k)} )/ \rho; q, \eta_0/\rho\} - \hat\bsbb\\
        &= \Theta^{\#}\{\hat\bsbb  - \nabla l(\hat\bsbb) / \rho; q, \eta_0/\rho\} - \hat\bsbb,
\end{align*}
where the second equality is due to the continuity of $\nabla l( \bsbb)$ and the $\Theta^{\#}$-uniqueness assumption.

Define ${\hat{\mathcal{J}}} = \{j: \hat\beta_j \neq 0\}$. Then we get
\begin{equation*}
        \hat\bsbb_{\hat{\mathcal{J}}} = \hat\bsbb_{\hat{\mathcal{J}}} / (1 + \eta_0/\rho) - \bsbX_{\hat{\mathcal{J}}}^T\nabla l_0(\bsbX_{\hat{\mathcal{J}}} \hat\bsbb_{\hat{\mathcal{J}}};\bsby)/(\rho  + \eta_0),
\end{equation*}
or equivalently,
\begin{equation*}
        \eta_0 \hat\bsbb_{\hat{\mathcal{J}}} + \bsbX_{\hat{\mathcal{J}}}^T\nabla l_0(\bsbX_{\hat{\mathcal{J}}}  \hat\bsbb_{\hat{\mathcal{J}}} ;\bsby) = 0.
\end{equation*}
Therefore, given $\hat{\mathcal J}$, $\hat\bsbb_{\hat{\mathcal{J}}}$ is a stationary point of
\begin{equation} \label{stationary_point}
        \min_{\bsbg}l_0( \bsbX_{\hat{\mathcal{J}}}\bsbg;\bsby) + \eta_0 \|\bsbg\|_2^2/2.
\end{equation}
When $l_0(\cdot; \bsby)$ is convex and $\eta_0 > 0$, \eqref{stationary_point} is strongly convex and thus $\hat\bsbb_{\hat{\mathcal{J}}}$ is the unique minimizer.

By Ostrowski's convergence theorem, the set of limit points of $\bsbb^{(t)}$ must be connected. On the other hand, the set of all restricted optimal solutions $\{\hat\bsbb_{\hat{\mathcal{J}}}\}$ is finite, and so
$$\lim_{t \to \infty} \bsbb^{(t)} = \hat\bsbb.$$
Under $\|\hat{\mathcal J}\|_0=q$, it is easy to see that the neighborhood  $\{\bsbb:  \| \bsbb-\hat\bsbb\|_{\infty}<\epsilon,\;J(\bsbb) \leq q\}$  with  $0<\epsilon  <\min_{j\in \hat{\mathcal J}}  \lvert\hat    \beta_j\lvert $ is just $\{\bsbb:    \mathcal J(\bsbb) =\hat{\mathcal J}, | \beta_j- \hat\beta_j | <\epsilon, \forall j \in \hat{\mathcal J}\}$. The local optimality of $\hat\bsbb$ and support stability of $\bsbb^{(t)}$ thus follow.

%%%%%%%%%%%%%%%%%%%%%%%%%%%%%%%%%%%%%%%%%%%%%%%%%%%%%%%%%%%%%%%%%%%%%%%%%%%%%%%
%%%%%%%%%%%%%%%%%%%%%%%%%%%%%%%%%%%  Proof of Theorem 2
%%%%%%%%%%%%%%%%%%%%%%%%%%%%%%%%%%%%%%%%%%%%%%%%%%%%%%%%%%%%%%%%%%%%%%%%%%%%%%%

\subsection{Proof of Theorem \ref{th:statistical_accuracy}}
We first introduce some lemmas that are helpful in proving the theorem.
The first is a generalization of Lemma 9 in \citep{She2017RRRR}.
\begin{lemma} \label{le:basic_l0}
        Let   $\mathcal J(\bsbB)$  denote the row support of matrix $\bsbB$ and define $J(\bsbB) = \| B\|_{2,0}= |\mathcal J(\bsbB)|$.  Consider the following problem with $0\le q \le p, \eta\ge 0$:
        \begin{equation*}
                \min_{\bsbB\in \mathbb R^{p\times m}}  \frac{1}{2} \|\bsbY - \bsbB\|_F^2 + \frac{\eta}{2}\|\bsbB\|_F^2 = l(\bsbB)  \quad \mbox{subject to} \ \|\bsbB\|_{2,0}\le q.
        \end{equation*}
        Then $\hat\bsbB = \vec{\Theta}^{\#}(\bsbY; q,\eta)$ (recall    $\vec\Theta^{\#}$ defined  in Section \ref{intro}) gives a globally optimal solution, and for any $\bsbB$ satisfying $J(\bsbB)\le s$, we have
        \begin{align}
                l(\bsbB) - l(\hat\bsbB) \ge (1-\mathcal L(\mathcal J, \hat{\mathcal  J}))(1+\eta) \frac{\| \hat\bsbB - \bsbB\|_F^2}{2}
        \end{align}
        where $\mathcal J = \mathcal J(\bsbB)$,   $\hat{\mathcal J} = \mathcal J(\hat \bsbB)$,  and  $ \mathcal L(\mathcal J, \hat{\mathcal  J}) = \sqrt{|\mathcal J \setminus \hat{\mathcal  J}| / | \hat{\mathcal  J} \setminus \mathcal J|}$. When $J(\hat\bsbB) = q   $ with $\vartheta(\equiv q/s)\ge 1$, $\mathcal L(\mathcal J, \hat{\mathcal  J})\le \sqrt{|\mathcal J| / |\hat{\mathcal J}|}\le 1 / \sqrt{\vartheta}$.
In the above statement,     $0/0  $ is understood as $1$. \end{lemma}

\begin{lemma} \label{le:stochastic_bound}
        There exist universal constants $A, C, c > 0$ such that for any $a > 0$, the following event
        \begin{equation} \label{stochastic_bound}
                \sup_{\bsbb_1, \bsbb_2} \langle \bsbeps, \bsbX(\bsbb_1 - \bsbb_2) \rangle - \frac{1}{2a}\|\bsbX (\bsbb_1 - \bsbb_2)\|_2^2- \frac{a}{2}A\sigma^2 \{J(\bsbb_1) \vee J(\bsbb_2) \} \log\Big\{ \frac{ep}{J(\bsbb_1) \vee J(\bsbb_2)} \Big\} \ge \frac{a}{2}\sigma^2t
        \end{equation}
        occurs with probability at most $C\exp{(-ct)}p^{-{c}A}$, where $t \geq 0$.
\end{lemma}

First, by definition, it is easy to show that  $\hat \bsbb$ satisfies
\begin{equation*}
        \hat\bsbb \in \argmin_{\bsbb} g(\bsbb, \hat\bsbb),
\end{equation*}
where $g(\bsbb, \bsbb^-) = \lVert\bsby-\bsbX\bsbb^{-}\rVert_{2}^{2}/2+\langle\bsbX^{T}(\bsbX\bsbb^{-}-\bsby),\bsbb-\bsbb^{-}\rangle+ {\rho} \lVert\bsbb-\bsbb^{-}\rVert_{2}^{2} /2 + \eta_0\|\bsbb\|_2^2/2$.
By $g(\hat\bsbb, \hat\bsbb) \le g(\bsbb^*, \hat\bsbb)$ and Lemma \ref{le:basic_l0},
\begin{equation*}
        \begin{split}
                \frac{1}{2}\|\bsbb^* - \hat \bsbb + \frac{1}{\rho}\bsbX^T(\bsbX\hat \bsbb - \bsby)\|_2^2 - \frac{1}{2}\|\frac{1}{\rho}\bsbX^T(\bsbX\hat \bsbb - \bsby)\|_2^2 + \frac{\eta_0}{2\rho}\|\bsbb^*\|_2^2 - \frac{\eta_0}{2\rho}\|\hat\bsbb\|_2^2 \\ \geq (1+\frac{\eta_0}{\rho})\frac{1 - \mathcal L(\mathcal J^*, \mathcal {\hat J})}{2} \|\hat \bsbb - \bsbb^*\|_2^2,
        \end{split}
\end{equation*}
where $\mathcal J^* = \mathcal J(\bsbb^*)$,  $\mathcal {\hat J} = \mathcal J(\hat \bsbb)$, and $\mathcal L(\mathcal J^*, \mathcal {\hat J}) \leq 1 / \sqrt{\vartheta}$.

It follows from the model $\bsby = \bsbX \bsbb^* + \bsbeps$ that
\begin{equation*}
        \|\bsbX \hat\bsbb - \bsbX \bsbb^*\|_2^2 + \frac{\eta_0}{2}\|\hat\bsbb\|_2^2 \leq \frac{\rho - (\sqrt{\vartheta}-1)\eta_0}{2\sqrt{\vartheta}} \|\hat\bsbb - \bsbb^*\|_2^2 +  \frac{\eta_0}{2}\|\bsbb^*\|_2^2 +  \langle\bsbX \hat\bsbb - \bsbX\bsbb^*, \bsbeps\rangle,
\end{equation*}
which gives
\begin{align}
        &\|\bsbX \hat\bsbb - \bsbX\bsbb^*\|_2^2 + \frac{\eta_0}{2}\|\hat\bsbb -\bsbb^*\|_2^2 \nonumber \\
        \le& \frac{\rho - (\sqrt{\vartheta}-1)\eta_0}{2\sqrt{\vartheta}}\|\hat\bsbb -  \bsbb^*\|_2^2 +  \eta_0\langle \hat\bsbb - \bsbb^*,  - \bsbb^*\rangle + \langle \bsbX\hat\bsbb - \bsbX \bsbb^*,  \bsbeps\rangle \nonumber \\ \label{eq:ineq_optimality}
        \le& \frac{\rho - (\sqrt{\vartheta}-1)\eta_0}{2\sqrt{\vartheta}}\|\hat\bsbb -  \bsbb^*\|_2^2 + \frac{b\eta_0}{2} \|\hat\bsbb-\bsbb^*\|_2^2 + \frac{\eta_0}{2b} \|\bsbb^*\|_2^2 + \langle \bsbX\hat\bsbb - \bsbX \bsbb^*,  \bsbeps\rangle
\end{align}
for any $b>0$.
Applying Lemma \ref{le:stochastic_bound} with $t = 0$, we can show that for any $a>0$, the following event
\begin{equation} \label{eq:stochastic_bound}
        \langle \bsbX  \hat\bsbb - \bsbX \bsbb^*, \bsbeps \rangle \leq \frac{1}{2a}\|\bsbX \hat\bsbb - \bsbX \bsbb^*\|_2^2 + \frac{a}{2}A\sigma^2\vartheta s\log{\frac{ep}{\vartheta s}}
\end{equation}
occurs with probability at least $1-Cp^{-c}$,  where $A, C, c > 0$ are some universal constants.

Combining \eqref{eq:ineq_optimality},  \eqref{eq:stochastic_bound} and  the regularity condition \eqref{eq:assumption_r0} yields
$$
\frac{\eta_0(\linrateparam - b)}{2} \|\hat\bsbb-\bsbb^*\|_2^2  + \Big(\frac{\delta}{2} - \frac{1}{2a}\Big)\|\bsbX \hat\bsbb - \bsbX\bsbb^*\|_2^2 \leq  \frac{\eta_0}{2b}\|\bsbb^*\|_2^2 + \frac{a}{2}A\sigma^2\vartheta s\log{\frac{ep}{\vartheta s}}$$
with probability at least $1- Cp^{-c}$.
By choosing $a = 2/\delta$ and $b = \linrateparam/2$, we have the bound for the prediction error as
\begin{align*}
        \|\bsbX \hat\bsbb - \bsbX \bsbb^*\|_2^2 + \frac{\eta_0\linrateparam}{\delta} \|\hat\bsbb-\bsbb^*\|_2^2
        &\le \frac{4\eta_0}{\delta\linrateparam}\|\bsbb^*\|_2^2 + \frac{4}{\delta^2}A\sigma^2\vartheta s\log{\frac{ep}{\vartheta s}} \\
        &\lesssim \frac{\eta_0}{\delta\linrateparam} \|\bsbb^*\|_2^2 + \frac{1}{\delta^2}\sigma^2\vartheta s\log{\frac{ep}{\vartheta s}},
\end{align*}
which holds with probability at least $1- Cp^{-c}$.\\

%%%%%%%%%%%%%%%%%%%%%%%%%%%%%%%%%%%%%%%%%%%%%%%%%%%%%%%%%%%%%%%%%%%%%%%%%%%%%%%
%%%%%%%%%%%%%%%%%%%%%%%%%%%%%%%%%%%  Proof of Lemma 1
%%%%%%%%%%%%%%%%%%%%%%%%%%%%%%%%%%%%%%%%%%%%%%%%%%%%%%%%%%%%%%%%%%%%%%%%%%%%%%%

%\paragraph*{Proof of Lemma \ref{le:basic_l0}}
\noindent\textbf{Proof of Lemma \ref{le:basic_l0}}
In this proof, given a matrix $\bsbB\in \mathbb R^{p\times m}$ and an index set $\mathcal I \subset [p]$, we use $\bsbB_{\mathcal I}$ to denote the submatrix of $\bsbB$ by extracting its rows indexed by $\mathcal I$.
Let $\mathcal J_1 = \mathcal J \cap \hat{\mathcal J}$, $\mathcal J_2 =  \hat{\mathcal J}\setminus \mathcal J$ and $\mathcal J_3 = \mathcal J \setminus \hat{\mathcal J}$. Then $\mathcal J= \mathcal J_1 \cup \mathcal J_3$ and $\hat {\mathcal J }=  \mathcal J_1 \cup \mathcal J_2$.

It can be easily shown that $\hat\bsbB_{\mathcal J_1}= \bsbY_{{\mathcal J}_1}/(1+\eta)$ and $\hat\bsbB_{\mathcal J_2}= \bsbY_{{\mathcal J}_2}/(1+\eta)$. By writing $\bsbB_{\mathcal J_1} = \bsbY_{{\mathcal J}_1}/(1+\eta)+\bsbDelta_{{\mathcal J}_1}$ and  $\bsbB_{\mathcal J_3} = \bsbY_{{\mathcal J}_3}/(1+\eta)+\bsbDelta_{{\mathcal J}_3}$, we have
\begin{align*}
 l (\bsbB ) - l (\hat\bsbB ) & = \frac{1+\eta}{2} \|\bsbDelta_{{\mathcal J}_1}\|_F^2 + \frac{1}{2(1+\eta)} \|\bsbY_{{\mathcal J}_2}\|_F^2   +\frac{1+\eta}{2} \|\bsbDelta_{{\mathcal J}_3}\|_F^2- \frac{1}{2(1+\eta)} \|\bsbY_{{\mathcal J}_3}\|_F^2,\\
\frac{1+\eta}{2}\| \hat\bsbB - \bsbB\|_F^2 &=\frac{1+\eta}{2} \|\bsbDelta_{{\mathcal J}_1}\|_F^2+\frac{1}{2(1+\eta)} \|\bsbY_{{\mathcal J}_2}\|_F^2 + \frac{1+\eta}{2} \|\frac{1}{1+\eta}\bsbY_{{\mathcal J}_3}+\bsbDelta_{{\mathcal J}_3}\|_F^2.
\end{align*}
Let $K\le 1$ satisfy
$$ l (\bsbB ) - l (\hat\bsbB ) \ge \frac{K}{2}(1+\eta)\| \hat\bsbB - \bsbB\|_F^2,
$$
which is implied by
\begin{equation} \label{eq:middle_ineq}
        \begin{split}
                &\frac{1}{2(1+\eta)} \|\bsbY_{{\mathcal J}_2}\|_F^2+\frac{1+\eta}{2} \|\bsbDelta_{{\mathcal J}_3}\|_F^2- \frac{1}{2(1+\eta)} \|\bsbY_{{\mathcal J}_3}\|_F^2 \\ \ge \  & \frac{K}{2(1+\eta)} \|\bsbY_{{\mathcal J}_2}\|_F^2+ \frac{K(1+\eta)}{2} \|\frac{1}{1+\eta}\bsbY_{{\mathcal J}_3}+\bsbDelta_{{\mathcal J}_3}\|_F^2.
        \end{split}
\end{equation}
\eqref{eq:middle_ineq} is equivalent to
\begin{align} \label{desiredineq}
\begin{split}&        (1-K)\|\bsbY_{{\mathcal J}_2}\|_F^2+ (1+\eta)^2\|\bsbDelta_{{\mathcal J}_3}\|_F^2   \ge (1+\eta)^2 K  \|\frac{1}{1+\eta}\bsbY_{{\mathcal J}_3}+\bsbDelta_{{\mathcal J}_3}\|_F^2 + \|\bsbY_{{\mathcal J}_3}\|_F^2.
\end{split}\end{align}
By construction, $\| \bsby_i \|_2 \ge \| \bsby_j \|_2$ for any $ i \in {\mathcal J}_2$ and $ j \in {\mathcal J}_3$. Thus $\|\bsbY_{{\mathcal J}_2}\|_F^2  \ge J_2 \|\bsbY_{{\mathcal J}_3}\|_F^2/J_3$, from which it follows that \eqref{desiredineq} is implied by
\begin{align*}
&         \{(1-K) ({J_2}/{J_3})-(1+K)\} \|\bsbY_{{\mathcal J}_3}\|_F^2+  (1-K) (1+\eta)^2 \|\bsbDelta_{{\mathcal J}_3}\|_F^2  \ge     2 K(1+\eta)\langle \bsbY_{{\mathcal J}_3}, \bsbDelta_{{\mathcal J}_3}\rangle.
\end{align*}
Therefore, restricting $K$ to $ (1+K)/(1-K)\le   {J_2}/{J_3}  $ or $K\le (J_2 - J_3)/(J_2 + J_3) \le 1$,  the largest possible $K$ should satisfy
$$ \{(1-K)({J_2}/{J_3})-(1+K)\}\cdot ( {1-K} )  = |K|^2
$$
or
$(1-K)^2 = J_3/  J_2$, or $K = 1- \sqrt{J_3/J_2} (\le (J_2 - J_3)/(J_2 + J_3))$. This gives
$$\mathcal L = 1-K = ({ {J_3}/{J_2}} )^{1/2}.
$$
 Note that when $\mathcal J_2=  \emptyset$,  $K$ can take $-\infty$ for   $\mathcal J_3 \neq   \emptyset$ and    $0$ for $  \mathcal J_3 =   \emptyset$ to ensure \eqref{desiredineq}.

Now assume   $J(\hat\bsbB) = q  $ with $\vartheta\ge 1$. If $\mathcal J_2 \ne   \emptyset$, $\mathcal L\le \sqrt{({J_3+J_1})/({J_2+J_1})}= \sqrt{{J}/{\hat J}} \le 1/\sqrt{\vartheta} $. Otherwise, we must have $\mathcal J_3 = \emptyset$, $\mathcal J = \hat{\mathcal J}$ and $\vartheta=1$.
The proof is complete.

The lemma can be used in the analysis of $\ell_0$-constrained (elementwise) sparsity pursuit, as well as   group variable selection (cf. Section \ref{subsec:gbsense}).
\\
%%%%%%%%%%%%%%%%%%%%%%%%%%%%%%%%%%%%%%%%%%%%%%%%%%%%%%%%%%%%%%%%%%%%%%%%%%%%%%%
%%%%%%%%%%%%%%%%%%%%%%%%%%%%%%%%%%%  Proof for stochatic lemma
%%%%%%%%%%%%%%%%%%%%%%%%%%%%%%%%%%%%%%%%%%%%%%%%%%%%%%%%%%%%%%%%%%%%%%%%%%%%%%%

%\paragraph*{Proof of Lemma \ref{le:stochastic_bound}}
\noindent\textbf{Proof of Lemma \ref{le:stochastic_bound}}
Given a matrix $\bsbA$, denote by $\mathcal P_{\bsbA}$ the orthogonal projection onto its range, and $\mathcal P_{\bsbA}^{\perp}$ its orthogonal complement.
In the proof, $\mathcal P_{\mathcal J}$ is used as a short notation for $\mathcal P_{\bsbX_{\mathcal J}}$ in the proof for any $J \subset [p]$.
Let $\mathcal{J}_1 = \mathcal J(\bsbb_1), \mathcal{J}_2 = \mathcal J(\bsbb_2), J_1 = |\mathcal{J}_1|, J_2 = |\mathcal{J}_2|$.

First, note that the term $\{J(\bsbb_1) \vee J(\bsbb_2) \} \log[ep/\{J(\bsbb_1) \vee J(\bsbb_2)\}]$ is used in \eqref{stochastic_bound}, instead of  $J(\bsbb_1 - \bsbb_2) \log\{ep/J(\bsbb_1\allowbreak - \bsbb_2)\}$, and although $J(\bsbb_1 - \bsbb_2) \le J(\bsbb_1) + J(\bsbb_2)$, $J(\bsbb_1) + J(\bsbb_2)$ can be larger than $p$. To tackle the issue, we employ a decomposition trick
\begin{align*}
        \bsbX\bsbb_1 - \bsbX\bsbb_2 &= \mathcal P_{\mathcal{J}_1} \bsbX(\bsbb_1 - \bsbb_2) +  \mathcal P_{\mathcal{J}_1}^{\perp} \bsbX(\bsbb_1 - \bsbb_2) \\
        &= \mathcal P_{\mathcal{J}_1} \bsbX(\bsbb_1 - \bsbb_2) + \mathcal P_{\mathcal{J}_1}^{\perp} \mathcal P_{\mathcal{J}_2} \bsbX(\bsbb_1 - \bsbb_2).
\end{align*}
Let $\bsbDelta = \bsbb_1 - \bsbb_2$. Then
\begin{equation} \label{stocha_terms}
        \langle \bsbeps, \bsbX\bsbDelta \rangle = \langle \bsbeps, P_{\mathcal{J}_1} \bsbX \bsbDelta \rangle
        + \langle \bsbeps, \mathcal{P}_{\mathcal{J}_1}^{\perp} \mathcal{P}_{\mathcal{J}_2} \bsbX \bsbDelta \rangle.
\end{equation}

Let us bound the first term on the right-hand side of \eqref{stocha_terms}. Define $P_o(J) = \sigma^2J \log(ep/J)$ for $0 \le J \le p$, which is an increasing function, and
$ \Gamma_{J} = \{\bsba \in \mathbb{R}^p:  \|\bsba\|_2 \le 1, \bsba \in \mathcal P_{\mathcal J} \ \mbox{for some} \ \mathcal J\subset [p],  |\mathcal J|\le  J\}$. Then for any $a, b > 0$
\begin{align*}
        &\langle \bsbeps, \mathcal P_{\mathcal J_1} \bsbX\bsbDelta\rangle  - \frac{1}{a} \| \mathcal P_{\mathcal J_1}\bsbX\bsbDelta\|_2^2 - bL P_o (J_1)  \\
        \le\,& \| \mathcal P_{\mathcal{J}_1} \bsbX\bsbDelta\|_2 \langle \bsbeps, \frac{ \mathcal P_{\mathcal{J}_1} \bsbX\bsbDelta}{\| \mathcal P_{\mathcal{J}_1} \bsbX\bsbDelta \|_2}\rangle - 2 \| \mathcal P_{\mathcal{J}_1} \bsbX\bsbDelta\|_2 \sqrt {\frac{b}{a}L P_o({ J_1} )}  \\
        \le\,&  \frac{1}{a} \| \mathcal P_{\mathcal J_1} \bsbX\bsbDelta \|^2_2 + \frac{a}{4}  \sup_{J_1\le p} \sup_{\bsbDelta\in\Gamma_{J_1}} \big\{\langle\bsbeps, \bsbDelta\rangle - 2\sqrt{ ({b}/{a})LP_{o} (J_1)} \big\}_+^2 \\
        \equiv\,&  \frac{1}{a} \| \mathcal P_{\mathcal J_1} \bsbX\bsbDelta \|^2_2 + \frac{a}{4}  \sup_{J_1\le p} R_{J_1}^2,
\end{align*}
where $R_{J_1} := \sup_{\bsbDelta\in\Gamma_{J_1}}\big\{\langle\bsbeps, \bsbDelta\rangle - 2\sqrt{ ({b}/{a})LP_{o} (J_1)} \big\}_+$ with $L$  a sufficiently large constant.
When $J_1=0$, $R_{J_1} = 0 $. When $J_1\ge 1$, for any $t\ge 0$,
if $4b/a$ is a constant greater than $1$, we have
\begin{equation} \label{Rbound}
        \begin{split}
                &\mathbb P(\sup_{1\le J_1 \le p}R_{J_1} \ge t\sigma)\\
                \le\,&  \sum_{J_1=1}^{p} \mathbb P\bigg(\sup_{\bsbDelta\in\Gamma_{J_1}}\langle\bsbeps, \bsbDelta\rangle - \sqrt{LP_o (J_1)} \ge t\sigma + 2\sqrt{\frac{b}{a}LP_o (J_1)} - \sqrt{LP_o(J_1)}\bigg)\\
                \le\,&   C\exp(-ct^2)\sum_{J_1=1}^{p} \exp[-c(2\sqrt{b/a}-1)^2 LP_o(J_1)/\sigma^2]\\
                \le\,&  C\exp(-ct^2)\exp(-c L\log p)\sum_{J_1=1}^p \exp(-cLJ_1)\\
                \le\, & C \exp(-ct^2 ) p^{-cL}.
        \end{split}
\end{equation}
The second inequality is due to Lemma 6 of \cite{She2016}, and we used  $ J\log(ep/J) \ge  J  +\log p $ for any $J \in [p]$ in the third inequality.
Therefore, for any $a, b >0,$ $4b>a$ and $t \ge 0$, we have
\begin{equation} \label{stocha_1}
        \EP \Big\{\langle \bsbeps, \mathcal P_{\mathcal J_1} \bsbX\bsbDelta\rangle - \frac{2}{a} \| \mathcal P_{\mathcal J_1}\bsbX\bsbDelta\|_2^2 - bL P_o (J_1)  \ge \frac{a}{4}t\sigma^2 \Big\}
        \le  C \exp(-ct) p^{-Lc}.
\end{equation}

Similarly, for the second term in \eqref{stocha_terms}, we can use Lemma 7 of \cite{SheetalPIQ} to prove that for any $t \ge 0$,
\begin{equation} \label{stocha_2}
        \EP \Big[\langle \bsbeps, \mathcal P_{\mathcal{J}_1}^{\perp} \mathcal P_{\mathcal{J}_2} \bsbX\bsbDelta\rangle - \frac{2}{a} \| \mathcal P_{\mathcal{J}_1}^{\perp} \mathcal P_{\mathcal{J}_2} \bsbX\bsbDelta\|_2^2 - bL \{P_o (J_1) + P_o (J_2)\}\ge \frac{a}{4}t\sigma^2 \Big]
        \le  C \exp(-ct ) p^{-Lc}.
\end{equation}

Combining \eqref{stocha_1}, \eqref{stocha_2} and using the fact that
$\|\mathcal{P}_{\mathcal J_1} \bsbX\bsbDelta\|_2^2 + \|\mathcal{P}_{\mathcal J_1}^{\perp} \mathcal \mathcal{P}_{\mathcal{J}_2} \bsbX\bsbDelta\|_2^2 = \|\bsbX \bsbDelta\|_2^2,$
we get for any $a, b >0,$ $4b>a$ and $t \ge 0$,
\begin{equation} \label{stocha_all}
        \EP \Big[\langle \bsbeps, \bsbX\bsbDelta\rangle - \frac{4}{a} \|\bsbX\bsbDelta\|_2^2 - 3bL \{P_o (J_1) \vee P_o (J_2)\}  \ge \frac{a}{2}t\sigma^2 \Big]
        \le  C \exp(-ct ) p^{-Lc}.
\end{equation}
Finally, using the increasing property of $P_o(J)$ for $ J \in [0, p]$, we have $P_o (J_1) \vee P_o (J_2) \le (J_1 \vee J_2) \log\{ep/(J_1 \vee J_2)\}$. A reparameterization of \eqref{stocha_all} gives the conclusion.

%%%%%%%%%%%%%%%%%%%%%%%%%%%%%%%%%%%%%%%%%%%%%%%%%%%%%%%%%%%%%%%%%%%%%%%%%%%%%%%
%%%%%%%%%%%%%%%%%%%%%%%%%%%%%%%%%%%  Proof of Theorem 3
%%%%%%%%%%%%%%%%%%%%%%%%%%%%%%%%%%%%%%%%%%%%%%%%%%%%%%%%%%%%%%%%%%%%%%%%%%%%%%%

\subsection{Proof of Theorem \ref{support_recovery}}
From the proof of Theorem \ref{th:statistical_accuracy}, we get with probability $1-Cp^{-c}$,
\begin{align*}
        \|\bsbX \hat\bsbb - \bsbX\bsbb^*\|_2^2 + \frac{\eta_0(1-b)}{2}\|\hat\bsbb-\bsbb^*\|_2^2
        \le \, &\frac{\rho - (\sqrt{\vartheta}-1)\eta_0}{2\sqrt{\vartheta}}\|\hat\bsbb -  \bsbb^*\|_2^2 + \frac{\eta_0}{2b} \|\bsbb^*\|_2^2 + \\
        &\frac{1}{2a}\|\bsbX \hat\bsbb - \bsbX \bsbb^*\|_2^2 + \frac{a}{2}A\sigma^2\vartheta s\log{\frac{ep}{\vartheta s}},
\end{align*}
which gives
\begin{align*}
        \|\bsbX \hat\bsbb - \bsbX\bsbb^*\|_2^2 - \frac{\eta_0b}{2}\|\hat\bsbb-\bsbb^*\|_2^2
        \le \, &\frac{\rho - (2\sqrt{\vartheta}-1)\eta_0}{2\sqrt{\vartheta}}\|\hat\bsbb -  \bsbb^*\|_2^2 + \frac{\eta_0}{2b} \|\bsbb^*\|_2^2 + \\
        &\frac{\rho_+((1+\vartheta)s)}{2a}\| \hat\bsbb - \bsbb^*\|_2^2 + \frac{a}{2}A\sigma^2 \vartheta s\log{\frac{ep}{\vartheta s}}.
\end{align*}
Under the regularity condition \eqref{eq:assumption_r0_new}, choosing $a = 2/\delta$ and $b = \delta\rho_+( (1+\vartheta)s )/(4\eta_0)$ give \eqref{estimation_error_bound}. (The result applies to $\eta_0 = 0$ as well.)

To show the second result, note that from Theorem 1, the fixed-point solution $\hat\bsbb$ must satisfy $\hat\bsbb = \Theta^{\#}\{\hat\bsbb - \bsbX^T \nabla l_0(\bsbX\hat\bsbb; \bsby) /\rho; q, \eta_0/\rho \}$, which means
\begin{align*}
        &\Big\|\hat\bsbb (1+\eta_0/\rho)-\hat\bsbb + \frac{1}{\rho} \bsbX^T \nabla l_0(\bsbX\hat\bsbb) \Big\|_{\infty}
        \le (1+\eta_0/\rho) \min_{j \in \hat{\mathcal{J}}} |\hat\beta_j| \\
        \implies
        &\Big\|\eta_0\hat\bsbb + \bsbX^T (\nabla l_0(\bsbX\hat\bsbb) - \nabla l_0(\bsbX\bsbb^*)) -\bsbX^T\bsbeps \Big\|_{\infty}
        \le (\rho+\eta_0) \min_{j \in \hat{\mathcal{J}}} |\hat\beta_j| \\
        \implies
        &\left\|  \bsbX^T (\nabla l_0(\bsbX\hat\bsbb) - \nabla l_0(\bsbX\bsbb^*)) + \eta_0(\hat\bsbb - \bsbb^*) \right\|_{\infty}
        \le \|\bsbX^T\bsbeps\|_{\infty} + \eta_0\|\bsbb^*\|_{\infty}  + (\rho+\eta_0)\min_{j \in \hat{\mathcal{J}}} |\hat\beta_j|.
\end{align*}
Next, we introduce a lemma.
\begin{lemma} \label{le:infnormconvt}
       Let $\tilde \bsbb, \bsbb\in \mathbb R^p$ satisfying $\| \tilde \bsbb\|_0 = q > s \ge \| \bsbb\|_0$, and for short, denote $ \mathcal J(\tilde \bsbb)$ and  $\mathcal J(\tilde \bsbb)$ by $\tilde {\mathcal J}$ and $   \mathcal J$, respectively. Then
\begin{align}
& \min_{j \in \tilde {\mathcal J} } | \tilde \bsbb_j|\le \min_{j \in \tilde {\mathcal J} \setminus \mathcal J  } | \tilde \bsbb_j| \le \frac{  \| (\tilde \bsbb  -\bsbb )_{\tilde {\mathcal J} \setminus \mathcal J  }\|_{2}}{\sqrt {|\tilde {\mathcal J} \setminus \mathcal J |}}\le \frac{  \| (\tilde \bsbb  -\bsbb )_{\tilde {\mathcal J} \setminus \mathcal J  }\|_{2}}{\sqrt {q-s}}  \le \frac{\|  \tilde \bsbb  -\bsbb  \|_{2}}{\sqrt {q-s}} \label{eq:infnormconvt1}\\
& \min_{j \in \tilde {\mathcal J} } | \tilde \bsbb_j|\le \max_{j \in \tilde {\mathcal J} \setminus \mathcal J  } | \tilde \bsbb_j| =  \| (\tilde \bsbb  -\bsbb )_{\tilde {\mathcal J} \setminus \mathcal J  }\|_{\infty} \le \|  \tilde \bsbb  -\bsbb  \|_{\infty.}  \label{eq:infnormconvt2}
\end{align}
\end{lemma}
The proof is simple and omitted.
Now, combining the regularity condition \eqref{eq:assumption_infty} and \eqref{eq:infnormconvt1} or \eqref{eq:infnormconvt2}  gives the desired result.

% $$ \min_{j\in  \hat{\mathcal{J}}} |\hat\beta_j| \le  \frac{\| (\hat\bsbb - \bsbb^*)_{\mathcal{J}_2}\|_2}{\sqrt{|\mathcal{J}_2|}}\le \frac{\| \hat\bsbb - \bsbb^*\|_2}{\sqrt{q-s}}.$$
%The conclusion follows from the regularity condition \eqref{eq:assumption_infty}.

%let $\mathcal{J}_1 = \mathcal{J}^* \cap \hat{\mathcal{J}}$ and $\mathcal{J}_2 = \hat{\mathcal{J}} \setminus \mathcal{J}^*$. Then $|\mathcal{J}_2| \ge q-s$ and $\mathcal{J}_2$ is non-empty.

%Next, from $\min_{j\in  \hat{\mathcal{J}}} |\hat\beta_j| \le \min_{ \mathcal{J}_2} |\hat\beta_j| = \min_{ \mathcal{J}_2} |\hat\beta_j - \beta_j^*|$, we have
%$$ \min_{j\in  \hat{\mathcal{J}}} |\hat\beta_j| \le  \frac{\| (\hat\bsbb- \bsbb^*)_{\mathcal{J}_2}\|_2}{\sqrt{|\mathcal{J}_2|}}\le \frac{\| \hat\bsbb - \bsbb^*\|_2}{\sqrt{q-s}}.$$
%The conclusion follows from the regularity condition \eqref{eq:assumption_infty}.

%%%%%%%%%%%%%%%%%%%%%%%%%%%%%%%%%%%%%%%%%%%%%%%%%%%%%%%%%%%%%%%%%%%%%%%%%%%%%%%
%%%%%%%%%%%%%%%%%%%%%%%%%%%%%%%%%%%  Proof with rmt new
%%%%%%%%%%%%%%%%%%%%%%%%%%%%%%%%%%%%%%%%%%%%%%%%%%%%%%%%%%%%%%%%%%%%%%%%%%%%%%%

\subsection{Proof of Theorem \ref{rmt}} \label{rmt-proof}
By definition, we have
$$\rho_+(2q) = \sup_{I \in [p]: |I| = 2q}\lambda_{\max}(\bsbX_{I}^T\bsbX_{I}),
$$
and under $q+s \le n,$
$$\rho_-(q+s) = \inf_{I \in [p]:|I| = q+s}\lambda_{\min}(\bsbX_{I}^T\bsbX_{I}).
$$

By Theorem of 6.1 of \cite{wainwright2019high}, we have
$$ \mathbb{P}\left\{\sqrt{\frac{\lambda_{\max}(\bsbX_{I}^T\bsbX_{I})}{n}} \geq (1+c_0)\sqrt{\lambda_{\max}(\bsbSig_{I})} + \sqrt\frac{\tr(\bsbSig_{I})}{n} \right\} \leq \exp(-nc_0^2/2), \ \ \forall I: |I| = 2q $$ and
$$ \mathbb{P}\left\{\sqrt{\frac{\lambda_{\min}(\bsbX_{I}^T\bsbX_{I})}{n}} \leq (1-c_0)\sqrt{\lambda_{\min}(\bsbSig_{I})} - \sqrt\frac{\tr(\bsbSig_{I})}{n} \right\} \leq \exp(-nc_0^2/2), \ \ \forall I: |I| = q+s$$
for all $c_0 > 0$.
Applying the union bound gives
\begin{align}
        &\mathbb{P}\left\{ \sqrt\frac{\rho_+(2q)}{n} \geq (1+c_0)\sqrt{\lambda^{(2q)}_{\max}} + \sqrt{\frac{2q}{n}} \right\} \le \binom{p}{2q} \exp(-nc_0^2/2).
\end{align}
Let $nc^2 = nc_0^2 - \log\binom{p}{2q}.$
Then using $\log\binom{p}{2q} \le 2q \log{(ep/q)} $, $c_0 \le c + \sqrt{2q\log(ep/q)/n}$. Therefore for any $c>0$,
\begin{equation} \label{rmt-ineq}
        \EP \left\{\sqrt{\frac{\rho_+(2q)}{n}} \ge (1+c)\sqrt{\lambda^{(2q)}_{\max}} + \sqrt{\frac{2q\log(ep/q)}{n}}\sqrt{\lambda^{(2q)}_{\max}} + \sqrt{\frac{2q}{n}} \right\} \le
        \exp(-nc^2/2).
\end{equation}
Similarly,
$$\EP \left\{\sqrt{\frac{\rho_-(q+s)}{n}} \le (1-c)\sqrt{\lambda^{(q+s)}_{\min}} - \sqrt{\frac{(q+s)\log(ep/q)}{n}}\sqrt{\lambda^{(q+s)}_{\min}} - \sqrt{\frac{q+s}{n}} \right\} \le
\exp(-nc^2/2).
$$
Let $c \in (0,1)$ and assume $n \ge \{2(q+s)/(1-c)^2\} \{ 1/\lambda^{(q+s)}_{\min} + \log(ep/q) \}.$ Then
$$
\frac{\rho_+(2q)}{\rho_-(q+s)} \le
\left\{  \frac{ (1+c)\sqrt{\lambda^{(2q)}_{\max}} + \sqrt{\{2\lambda^{(2q)}_{\max}q\log(ep/q)\}/n} + \sqrt{2q/n} } {  (1-c)\sqrt{\lambda^{(q+s)}_{\min}} - \sqrt{\{\lambda^{(q+s)}_{\min}(q+s)\log(ep/q)\}/n} - \sqrt{(q+s)/n} }   \right\}^2
$$
holds with probability at least $1 - 2\exp(-nc^2/2)$.

%%%%%%%%%%%%%%%%%%%%%%%%%%%%%%%%%%%%%%%%%%%%%%%%%%%%%%%%%%%%%%%%%%%%%%%%%%%%%%%
%%%%%%%%%%%%%%%%%%%%%%%%%%%%%%%%%%%  Proof in section 3
%%%%%%%%%%%%%%%%%%%%%%%%%%%%%%%%%%%%%%%%%%%%%%%%%%%%%%%%%%%%%%%%%%%%%%%%%%%%%%%

\subsection{Proof of Theorem \ref{th:new}}
Let $ E:= \sigma^2P_o(q) + \sigma^2$.
Similar to the proof of Theorem \ref{th:statistical_accuracy}, from the construction of $g$ and Lemma \ref{le:basic_l0}, we have
$$
\rho(1-1/\sqrt{\vartheta})(1+\eta_0/\rho)\Breg_2(\bsbb^*, \hat\bsbb)+g(\hat\bsbb, \hat\bsbb) \le g(\bsbb^*, \hat\bsbb),
$$
and thus
\begin{equation}
        2\bar\breg_{l_0}(\bsbX\hat{\bsbb}, \bsbX\bsbb^*) + \frac{\eta_0}{2}\|\hat{\bsbb}\|_2^2 \le   \frac{\rho - (\sqrt{\vartheta} - 1)\eta_0 }{\sqrt{\vartheta}}\Breg_2(\hat{\bsbb}, \bsbb^*) + \frac{\eta_0}{2}\|{\bsbb}^*\|_2^2 +  \langle \bsbeps,\bsbX\hat{\bsbb}-\bsbX\bsbb^*\rangle.
        \label{prepbound1} \end{equation}
Applying Lemma \ref{le:stochastic_bound} gives
\begin{equation} \label{stochastic1}
        \begin{split}
                \langle \bsbeps,\bsbX\hat{\bsbb}-\bsbX\bsbb^*\rangle
                \le
                \delta \Breg_2(\bsbX\hat{\bsbb}, \bsbX\bsbb^*) +  \frac{1}{\delta} A \sigma^2P_o(q) + R
        \end{split}
\end{equation}
for any $\delta>0$,
where $R :=  \sup_{\bsbb_1, \bsbb_2} \{\langle \bsbeps,\bsbX\bsbb_1-\bsbX\bsbb_2\rangle - \delta \Breg_2(\bsbX\bsbb_1, \bsbX\bsbb_2) -   A \sigma^2P_o(q)/\delta \}_+$ and
$$\EP (\delta R > \sigma^2 t) \le    C\exp(-ct)p^{-cA}, $$
where $A, C, c > 0$ are some constants.
Therefore,
\begin{equation} \label{stochastic2}
        \EE \langle \bsbeps,\bsbX\hat{\bsbb}-\bsbX\bsbb^*\rangle
        \le
        \EE \{ \delta \Breg_2(\bsbX\hat{\bsbb}, \bsbX\bsbb^*) \} +  \frac{C}{\delta} (\sigma^2P_o(q)+\sigma^2).
\end{equation}
Combining \eqref{prepbound1} and \eqref{stochastic2} gives
\begin{equation}
        \begin{split}
                &\EE \{(2 \bar{\breg}_{l_0}  - \delta \Breg_2) (\bsbX\hat\bsbb , \bsbX\bsbb^*  )
                + \eta_0 \Breg_2(\hat\bsbb, \bsbb^*) \} \\
                \le &
                \EE \Big\{   \frac{\rho - (\sqrt{\vartheta} - 1)\eta_0   }{\sqrt \vartheta}\  \Breg_2(   \hat  \bsbb,    \bsbb^*)
                + \eta_0\langle -\bsbb^*,    \hat\bsbb -  \bsbb^*   \rangle \Big\}
                + \frac{C}{\delta}E,
        \end{split}
\end{equation}
and so
\begin{equation}  \label{auxeq1}
        \EE \Big[(2 \bar{\breg}_{l_0}  - \delta \Breg_2) (\bsbX\hat\bsbb , \bsbX\bsbb^*  )
        -  \frac{\rho - \{(2-\linrateparam)\sqrt{\vartheta} - 1\}\eta_0   }{\sqrt \vartheta}\  \Breg_2(   \hat  \bsbb,    \bsbb^*) \Big] \\
        \le
        \frac{C}{\delta}E
        + \frac{\eta_0}{2\linrateparam}\|\bsbb^*\|_2^2
\end{equation}
for any $\linrateparam, \delta > 0$.

Next, from $l_0( \bsbX \hat {\bsbb} ) + \eta_0\|\hat{\bsbb}\|_2^2/2 \le l_0(  \bsbX \bsbb^{(0)}) + \eta_0\|{\bsbb}^{(0)}\|_2^2/2$,  we have
\begin{equation} \label{fixed-point_bound-new2}
        \begin{split}
                &\breg_{l_0} ( \bsbX \hat {\bsbb}, \bsbX \bsbb^*) + \eta_0\Breg_2(\hat {\bsbb},  \bsbb^*) \\
                \le & \,
                \breg_{l_0}(   \bsbX \bsbb^{(0)}, \bsbX \bsbb^*) + \eta_0\Breg_2(\bsbb^{(0)}, \bsbb^*)
                +  \eta_0\langle -\bsbb^*,    \hat\bsbb -  \bsbb^*   \rangle - \eta_0\langle -\bsbb^*,    \bsbb^{(0)} -  \bsbb^*   \rangle \\
                &+ \langle \bsbeps, \bsbX \hat {\bsbb} - \bsbX \bsbb^*   \rangle -  \langle \bsbeps, \bsbX   {\bsbb}^{(0)} - \bsbX \bsbb^*   \rangle.
        \end{split}
\end{equation}
Therefore, for any $\delta', \delta'', \linrateparam' >0 $
\begin{align*}
        &\EE \{  ( {\breg}_{l_0}  - \delta'  \Breg_2) (\bsbX \hat {\bsbb} , \bsbX \bsbb^*  ) +\eta_0\Breg_2(\hat\bsbb, \bsbb^*) \} \\
        \le&
        \EE \big\{ (\breg_{l_0} + \delta'' \Breg_2)(   \bsbX \bsbb^{(0)}, \bsbX \bsbb^*)
        + \eta_0\Breg_2(\bsbb^{(0)}, \bsbb^*)
        + \frac{\eta_0}{2\linrateparam}\|\bsbb^*\|_2^2 + \eta_0\linrateparam\Breg_2(\hat\bsbb, \bsbb^*) \\
        &+ \frac{\eta_0}{2\linrateparam'}\|\bsbb^*\|_2^2 + \eta_0\linrateparam'\Breg_2(\bsbb^{(0)}, \bsbb^*) \big\}
        + CE\big(\frac{1}{\delta'} + \frac{1}{\delta''}\big).
\end{align*}
By the assumption of the starting point $\EE \{ \Breg_2 (\bsbb^{(0)}, \bsbb^*) \} \le CME/n,$ we have
$$\EE \{ \Breg_2 (\bsbX\bsbb^{(0)}, \bsbX\bsbb^*) \} \le C\rho_+(q+s) ME/n, \EE \{ \breg_{l_0} (\bsbX\bsbb^{(0)}, \bsbX\bsbb^*) \} \le C\rho^l_+(q, s) ME/n.$$
Taking $1/\delta'' = \sqrt{\rho_+(q+s)M/n}$, we obtain
\begin{align*}
        &\EE \{  ( {\breg}_{l_0}  - \delta'  \Breg_2) (\bsbX \hat {\bsbb} , \bsbX \bsbb^*  ) +\eta_0(1-\linrateparam)\Breg_2(\hat\bsbb, \bsbb^*) \} \\
        \le& \, CE\big( \frac{1}{\delta'}
        + \sqrt{\frac{\rho_+(q+s)M}{n}}
        + \frac{\rho_+^l(q, s)}{n}M
        + \frac{\eta_0(1+\linrateparam')}{n}M \big)
        + \eta_0\big(\frac{1}{\linrateparam} + \frac{1}{\linrateparam'} \big) \frac{\|\bsbb^*\|_2^2}{2}.
\end{align*}
Let $Q_0 := \sqrt{\rho_+(q+s)M/n} + \rho_+^l(q, s)M/n + \eta_0(1+\linrateparam')M/n$.
Then
$$CE\big(\frac{1}{\delta'}+Q_0\big) \le \frac{C}{c_1 \wedge c_2 } E\big(\frac{c_1}{\delta'} + c_2Q_0\big)$$
for any $c_1, c_2 > 0$.
Taking $\delta': \delta^2 = \delta'^2 /(c_1 + c_2 Q_0 \delta')$ and $\linrateparam': 1/\linrateparam+1/\linrateparam' = (1/\delta' + Q_0)c_3\delta/\linrateparam $
for some large constant $c_3>0$, we get
\begin{equation}\label{auxeq2}
        \begin{split}
                \EE \{  (\frac{\delta}{\delta'} {\breg}_{l_0}  - \delta \Breg_2) (\bsbX \hat {\bsbb} , \bsbX \bsbb^*  ) + \frac{\delta}{\delta'}\eta_0(1-\linrateparam)\Breg_2(\hat\bsbb, \bsbb^*) \}
                \le
                \frac{CE}{c_1 \wedge c_2 } \frac{1}{\delta}
                + c_3\frac{\eta_0}{2\linrateparam}\|\bsbb^*\|_2^2.
        \end{split}
\end{equation}

Multiplying \eqref{auxeq1} by $(1 - 1/M)$ and   \eqref{auxeq2} by $1/M$ and adding the two inequalities yield
\begin{align}
        \begin{split}
                & \EE \Big[  (1-\frac{1}{M})\big\{2 \bar{\breg}_{l_0}    (\bsbX \hat {\bsbb} , \bsbX \bsbb^*  )  -    \frac{ \rho - \{(2-\linrateparam)\sqrt{\vartheta} - 1\}\eta_0  }{\sqrt \vartheta}\Breg_2(   \hat  \bsbb,    \bsbb^*)  \big\} \\
                & +( \frac{\delta}{M\delta'} {\breg}_{l_0}- \delta   \Breg_2 )(\bsbX \hat {\bsbb} , \bsbX \bsbb^*  )  + \frac{\delta}{M\delta'}\eta_0(1-\linrateparam)\Breg_2(\hat\bsbb, \bsbb^*) \Big]   \\
                \le & \,  C \big( \frac{E}{\delta} + \frac{\eta_0}{\linrateparam }\|\bsbb^*\|_2^2\big).  \label{ineq_initial}
        \end{split}
\end{align}

Simple calculation shows
$$ \frac{\delta'}{\delta}  =  \frac{ c_2 Q_0 \delta + \sqrt{c_2^2 Q_0^2 \delta^2 + 4c_1}}{2} \le  \frac{\sqrt{2} + 1}{2} \{c_2 Q_0 \delta \vee \sqrt{4c_1} \} \le C (Q_0\delta \vee 1).$$
It follows that
$$ \linrateparam' \le \frac{\linrateparam}{C(Q_0\delta \vee 1) +\delta Q_0 -1} \le C\frac{\linrateparam}{Q_0\delta \vee 1} \le C \linrateparam
$$
for some large constant $C$, and so $Q_0 \lesssim Q$. Under the condition that
\begin{equation} \label{regu-new}
        \begin{split}
                & K\sigma^2P_o(\vartheta s) + \Big\{ 2(1- \frac{1}{M})  \bar\breg_{l_0} + \frac{C}{M(Q\delta \vee 1) } \breg_{l_0} - 2\delta \Breg_2 \Big\} (\bsbX\hat\bsbb, \bsbX\bsbb^*) \\
                \ge & \frac{1 - 1/M }{\sqrt{\vartheta}}\big[\rho - \{(2-\linrateparam)\sqrt{\vartheta} - 1\}\eta_0 \big]\Breg_2(\hat\bsbb, \bsbb^*) -\frac{C}{M(Q\delta \vee 1) }\eta_0(1-\linrateparam)\Breg_2(\hat\bsbb, \bsbb^*),
        \end{split}
\end{equation}
\eqref{ineq_initial} yields
\begin{align}
        \begin{split}
                \EE [ \Breg_2  (\bsbX \hat {\bsbb} , \bsbX \bsbb^*  ) ]
                \le& \frac{K}{\delta}\sigma^2P_o(\vartheta s)  +  \frac{CE}{\delta^2} + C\frac{\eta_0}{\linrateparam }\|\bsbb^*\|_2^2\\
                \lesssim& \frac{K\delta \vee 1}{\delta^2} E + \frac{\eta_0}{\delta\linrateparam}\|{\bsbb}^{*}\|_2^2.
        \end{split}
\end{align}
With a reparameterization, the regularity condition \eqref{regu_feasible_point} implies \eqref{regu-new}.

%%%%%%%%%%%%%%%%%%%%%%%%%%%%%%%%%%%%%%%%%%%%%%%%%%%%%%%%%%%%%%%%%%%%%%%%%%%%%%%
%%%%%%%%%%%%%%%%%%%%%%%%%%%%%%%%%%%  Proof for a sequential result
%%%%%%%%%%%%%%%%%%%%%%%%%%%%%%%%%%%%%%%%%%%%%%%%%%%%%%%%%%%%%%%%%%%%%%%%%%%%%%%

\subsection{Proof of Theorem \ref{th:iter}}
\label{appsec:proofseq}
For convenience, denote $\Breg_2(\bsbX\bsbb, \bsbX\bsbb')$ by $\Breg_{2, \bsbX}(\bsbb, \bsbb')$.
From Lemma \ref{le:basic_l0}, we have
\begin{equation} \label{g-opt-general-proof}
        g(\bsbb^*, \bsbb^{(t)}) - g(\bsbb^{(t+1)}; \bsbb^{(t)}) \geq \rho_{t+1}(1-\mathcal L_{t+1})(1+\bar\eta_{t+1}) \Breg_2(\bsbb^{(t+1)}, \bsbb^*),
\end{equation}
where $\mathcal L_{t+1} = \mathcal L(\mathcal J(\bsbb^*), \mathcal J(\bsbb^{(t+1)})) \leq 1/\sqrt{\vartheta_{t+1}}$. (Recall $\vartheta_{t+1} = q_{t+1} / s > 1$, and $s \geq \|\bsbb^*\|_0.$)

Substituting $g(\bsbb, \bsbb^{(t)}) = l(\bsbb) + \eta_{t+1}\Breg_2(\bsbb, \bsb{0}) + (\rho_{t+1} \Breg_2 - \breg_{l})(\bsbb, \bsbb^{(t)})$ and $l(\bsbb^*) - l(\bsbb^{(t+1)}) = \langle \bsbeps, \bsbX \bsbb^{(t+1)}- \bsbX \bsbb^* \rangle-\back\breg_{l}(\bsbb^*, \bsbb^{(t+1)})$ into \eqref{g-opt-general-proof} gives
\begin{equation} \label{eq:optimality_bound_iter}
        \begin{aligned}
                &\{\rho_{t+1}(1-\mathcal L_{t+1})(1+\bar\eta_{t+1})\Breg_2 + \back\breg_{l}\}(\bsbb^*, \bsbb^{(t+1)})  + \eta_{t+1}\Breg_2(\bsbb^*, \bsbb^{(t+1)})  \\
                &+ (\rho_{t+1} \Breg_2 - \breg_{l})(\bsbb^{(t+1)}, \bsbb^{(t)}) \\
                \leq \, &(\rho_{t+1} \Breg_2 - \breg_{l})(\bsbb^*, \bsbb^{(t)})  + \langle \bsbeps, \bsbX \bsbb^{(t+1)} - \bsbX \bsbb^* \rangle + \eta_{t+1} \langle -\bsbb^*, \bsbb^{(t+1)} - \bsbb^* \rangle.
        \end{aligned}
\end{equation}

From Lemma \ref{le:stochastic_bound}, with probability at least $1 - Cp^{-cA}$
\begin{equation}
        \langle \bsbeps, \bsbX\bsbb^{(t+1)} - \bsbX\bsbb^* \rangle
        \le \delta_{t+1}\Breg_{2, \bsbX}(\bsbb^*, \bsbb^{(t+1)}) + \delta_{t+1}^{-1} A \sigma^2P_o(q_{t+1}), \ \mbox{for all} \ t \ge 0
\end{equation}
given any $\delta_{t+1} > 0$, where $A$ is a constant. Moreover, for any $\linrateparam_{t+1} > 0$,
\begin{equation}
        \langle -\bsbb^*, \bsbb^{(t+1)} - \bsbb^* \rangle \le \linrateparam_{t+1} \Breg_2(\bsbb^*, \bsbb^{(t+1)}) + \linrateparam_{t+1}^{-1}\Breg_2(\bsbb^*, \bsb{0}).
\end{equation}
Plugging these bounds into \eqref{eq:optimality_bound_iter} gives
\begin{align} \label{eq:optimality_bound_iter_all}
        &\big\{\rho_{t+1}(1-\mathcal L_{t+1})(1+\bar\eta_{t+1})\Breg_2 + \back\breg_{l} +  (1- \linrateparam_{t+1})\eta_{t+1}\Breg_2 - \delta_{t+1}\Breg_{2,\bsbX}\big\}(\bsbb^*, \bsbb^{(t+1)}) \nonumber \\
        & \, + (\rho_{t+1} \Breg_2 - \breg_{l})(\bsbb^{(t+1)}, \bsbb^{(t)}) \nonumber\\
        \leq \, &(\rho_{t+1} \Breg_2 - \breg_{l})(\bsbb^*, \bsbb^{(t)}) + \delta_{t+1}^{-1} A \sigma^2P_o(q_{t+1}) + \linrateparam_{t+1}^{-1}\eta_{t+1}\Breg_2(\bsbb^*, \bsb{0}).
\end{align}
By the definition of (generalized) isometry numbers and using $\mathcal L_{t+1} \le 1/\sqrt{\vartheta_{t+1}}$, we have
\begin{align}
        &\Big\{ \rho_{t+1}\big(1 - \frac{1}{\sqrt{\vartheta_{t+1}}}\big)(1+\bar\eta_{t+1}) + \rho_-^{l}(q_{t+1}, s)+ (1-\linrateparam_{t+1})\eta_{t+1} - \delta_{t+1}\rho_+(q_{t+1} + s) \Big\}\Breg_2(\bsbb^*, \bsbb^{(t+1)}) \nonumber\\
        & \, + (\rho_{t+1} \Breg_2 - \breg_{l})(\bsbb^{(t+1)}, \bsbb^{(t)}) \nonumber\\
        \le \, &\big\{ \rho_{t+1} - \rho^l_-(s, q_{t+1})  \big\} \Breg_2(\bsbb^*, \bsbb^{(t)}) + \delta_{t+1}^{-1} A \sigma^2P_o(q_{t+1}) + \linrateparam_{t+1}^{-1}\eta_{t+1}\Breg_2(\bsbb^*, \bsb{0}).
\end{align}
Let $\linrateparam_0$ be any number $\in (0, 1]$. Taking $\linrateparam_{t+1} = \linrateparam_0/2, \delta_{t+1} = (\linrateparam_0\rho_-^l(q_{t+1}, s) + \linrateparam_0\eta_{t+1}/2)/\rho_+(q_{t+1} + s)$, we have
\begin{align*}
        &(1 - 1/\sqrt{\vartheta_{t+1}})(1+\bar\eta_{t+1})\rho_{t+1} + \rho_-^{l}(q_{t+1}, s)+ (1-\linrateparam_{t+1})\eta_{t+1} - \delta_{t+1}\rho_+(q_{t+1} + s)\\
        = \, &(1 - 1/\sqrt{\vartheta_{t+1}})(1+\bar\eta_{t+1})\rho_{t+1} + (1-\linrateparam_0)\rho_-^{l}(q_{t+1}, s) +  (1-\linrateparam_0)\eta_{t+1}.
\end{align*}
Let
\begin{align*}
        E_{t+1} =& \, \frac{1}{\rho_{t+1} - \rho^l_-(s, q_{t+1})} \Big\{ \frac{A\sigma^2}{\linrateparam_0} \frac{\rho_+(q_{t+1} + s)}{\rho_-^l(q_{t+1}, s) + \eta_{t+1}/2} P_o(q_{t+1})  + \frac{\eta_{t+1}}{\linrateparam_{0}}\|\bsbb^*\|_2^2 \Big\} \\
        \le & \, \frac{A\sigma^2}{\linrateparam_0} \frac{\rho_+(q_{t+1} + s)}{(\rho_-^l(q_{t+1}, s)/\rho_{t+1} \vee \bar\eta_{t+1}) (1 - \rho^l_-(s, q_{t+1})/\rho_{t+1})\rho_{t+1}^2 } P_o(q_{t+1}) \\
        & \, + \frac{\bar\eta_{t+1}}{\linrateparam_{0}(1 - \rho^l_-(s, q_{t+1})/\rho_{t+1})}\|\bsbb^*\|_2^2
\end{align*}
for any $t \ge 0$.
By the definitions of $\kappa_{t}, h_t$, we can obtain
\begin{align}
        &\Breg_2(\bsbb^*, \bsbb^{(t+1)}) + h_{t+1} (\rho_{t+1} \Breg_2 - \breg_{l})(\bsbb^{(t+1)}, \bsbb^{(t)})  \le \kappa_{t+1} \Breg_2(\bsbb^*, \bsbb^{(t)}) + \kappa_{t+1}E_{t+1}.
\end{align}
Applying a recursive argument with $t = T, \dots, 0$ gives
\begin{align}
        &\Breg_2(\bsbb^*, \bsbb^{(T+1)}) + \sum_{t=0}^T \big(\Pi_{\tau = t}^{T} h_{\tau+1}\big) (\rho_{t+1} \Breg_2 - \breg_{l})(\bsbb^{(t+1)}, \bsbb^{(t)}) \nonumber\\
        \le & \,
        \bigg( \Pi_{t = 0}^{T} \kappa_{t+1}  \bigg)\Breg_2(\bsbb^*, \bsbb^{(0)})
        + \sum_{t=0}^T \Big(\Pi_{\tau = t}^{T} \kappa_{\tau+1}\Big) E_{t+1}, \nonumber
\end{align}
and thus the bound \eqref{eq:seq_result} follows.

To ensure
\begin{equation} \label{eta_choice}
        \frac{\rho_{t} - \rho^l_-(s, q_{t})}{(1 - 1/\sqrt{\vartheta_{t}})(1+\bar\eta_{t})\rho_{t} + (1-\linrateparam)(\rho_-^{l}(q_{t}, s) + \eta_{t})}
        \le \frac{1}{1+\alpha}
\end{equation}
for some $\alpha > 0$, we need
\begin{equation}
        \bar\eta_t \ge \frac{ (\alpha + 1/\sqrt{\vartheta_{t}}) - (2+\alpha-\linrateparam)\{\rho^l_-(s, q_{t}) \wedge \rho^l_-(q_{t}, s)\}/\rho_t }{2 - 1/\sqrt{\vartheta_{t}}  - \linrateparam}.
\end{equation}
The   result in the corollary follows by taking $\alpha = \linrateparam$ and noticing that $\rho_{t+1} \ge \rho^l_+(q_{t+1}, q_{t})$ implies
$(\rho_{t+1}\Breg_2 - \breg_l)(\bsbb^{t+1}, \bsbb^{t}) \ge 0$.

\subsection{A recursive coordinatewise error bound under  restricted isometry} \label{subsec:recurcoorderror}
Recall the general procedure defined in   \eqref{eq:general_update},
\begin{align}
\bsbb^{(t+1)} = \Theta^{\#}\Big\{\bsbb^{(t)} - \rho_{t+1}^{-1} \bsbX^T\nabla l_0(\bsbX\bsbb^{(t)}; \bsby);
q_{t+1}, \bar\eta_{t+1}\Big\},   \mbox{ with } \bar\eta_{t+1} = \eta_{t+1}/\rho_{t+1}. \label{genseq-app}
\end{align}
Following a similar approach  to Theorem \ref{support_recovery} for the set of fixed points, an error bound for $\bsbb^{(t+1)}$ in the $\infty$-norm can be established under appropriate regularity conditions.

To facilitate the proof, we first recall the definition of    $\rho_-^{l}(s_1, s_2)$ as given in \eqref{gen-rip-low}.  In particular, in the regression setup, $\rho_-(s_1, s_2)$ satisfies \begin{align*}
  & \|\bsbX (\bsbb_1 - \bsbb_2)\|_2^2 \ge \rho_-(s_1, s_2) \| \bsbb_1 - \bsbb_2\|_2^2, \forall \bsbb_i : \| \bsbb_i\|_0\le s_i \\ \Longleftrightarrow \quad & ( \bsbb_1 - \bsbb_2)^T ( \rho \bsbI -\bsbX^T   \bsbX)( \bsbb_1 - \bsbb_2)\le (\rho - \rho_-(s_1, s_2)) \| \bsbb_1 -\bsbb_2\|_2^2, \forall \bsbb_i : \| \bsbb_i\|_0\le s_i.
\end{align*}
The presence of positive restricted eigenvalues in the Gram matrix $\bsbX^T \bsbX$ implies the existence of proper upper bounds on the restricted eigenvalues of the matrix $\rho \bsbI - \bsbX^T \bsbX$. So when considering the $\infty$-norm error for   $\bsbb^{(t+1)}$, it appears more manageable to work with the matrix $\rho \bsbI - \bsbX^T \bsbX$ than with $\bsbX^T \bsbX$.

Motivated by this, given $l$, $\bsbX$, and $s_i$,  we introduce  a generalized restricted isometry number $\upsilon(s_1, s_2)$ that satisfies
\begin{align}
\|\rho ( \bsbb_1 - \bsbb_2) - \bsbX^T \{\nabla l_0(\bsbX \bsbb_1) - \nabla l_0(\bsbX\bsbb_2)\}  \|_{\infty} \le (\rho - \upsilon) \|\bsbb_1 - \bsbb_2\|_{\infty}, \mbox{ for all } \bsbb_i: \| \bsbb_i\|_0\le s_i, \rho \ge \upsilon.\label{geniso-inf}
\end{align}
In the case where  $l_0(\bsbX \bsbb) = \| \bsbX \bsbb - \bsby\|_2^2/2$, we have $\nabla l_0(\bsbX \bsbb_1) - \nabla l_0(\bsbX\bsbb_2)= \bsbX (\bsbb_1-\bsbb_2)$ and    $\rho ( \bsbb_1 - \bsbb_2) - \bsbX^T (\nabla l_0(\bsbX \bsbb_1) - \nabla l_0(\bsbX\bsbb_2)) = (\rho \bsbI - \bsbX^T \bsbX) (\bsbb_1 - \bsbb_2)$. Therefore,   \eqref{geniso-inf} can be understood as a variant of low coherence for the design matrix in the context of the $\infty$-norm.

\begin{theorem}
For the sequence of iterates generated by  procedure \eqref{genseq-app} and  $\upsilon_t$ denoting  $\upsilon(q_t, s)$   as defined by  \eqref{geniso-inf},
the following recursive coordinatewise error bound on $\bsbb^{(t+1)}$  holds for  any $t\ge 0$:

\begin{align*}
        \| \bsbb^{(t+1)} - \bsbb^* \|_{\infty} \le \  &  (1-\frac{\upsilon_t+\eta_{t+1}}{\rho_{t+1}+ \eta_{t+1}}) \| \bsbb^{(t)} - \bsbb^* \|_{\infty}
         + \frac{\|\bsbX^T\bsbeps\|_{\infty}}{  \rho_{t+1}+\eta_{t+1}} + \frac{\eta_{t+1} \|\bsbb^*\|_{\infty}}{\ \rho_{t+1}+\eta_{t+1}}+ \frac{1}{\sqrt{\vartheta_{t+1} - 1}} \frac{\| \bsbb^{(t+1)} - \bsbb^* \|_{2}}{\sqrt {s} }. \end{align*}
\end{theorem}

\begin{proof}
The proof follows similar lines of the proof of Theorem \ref{support_recovery}. First, by the definition of $\Theta^\#$,
 \begin{align*}
        &\Big\| (1+\bar\eta_{t+1})\bsbb^{(t+1)} - \bsbb^{(t)} + \frac{1}{\rho_{t+1}} \bsbX^T \nabla l_0(\bsbX \bsbb^{(t)}) \Big\|_{\infty}
        \le (1+\bar\eta_{t+1}) \min_{j \in   {\mathcal{J}(\beta_j^{(t+1)})}} | \beta_j^{(t+1)}|
\end{align*}
and so
\begin{align*}
        \|(\rho_{t+1}+ \eta_{t+1}) \bsbb^{(t+1)} -\rho_{t+1}\bsbb^{(t)} + \bsbX^T (\nabla l_0(\bsbX \bsbb^{(t)}) - \nabla l_0(\bsbX\bsbb^*)) -\bsbX^T\bsbeps \|_{\infty} \\
        \le ( \rho_{t+1}+\eta_{t+1}) \min_{j \in   {\mathcal{J}(\beta_j^{(t+1)})}} | \beta_j^{(t+1)}|.  \end{align*}
Writing
$$
(\rho_{t+1}+ \eta_{t+1}) \bsbb^{(t+1)} -\rho_{t+1}\bsbb^{(t)} =(\rho_{t+1}+ \eta_{t+1}) (\bsbb^{(t+1)} - \bsbb^*) -\rho_{t+1}(\bsbb^{(t)} - \bsbb^*) + \eta_{t+1} \bsbb^*
$$
and using the sub-additivity of the $\infty$-norm, we get
\begin{align*}
       (\rho_{t+1}+ \eta_{t+1}) \| \bsbb^{(t+1)} - \bsbb^* \|_{\infty} \le \  & \|\rho_{t+1}( \bsbb^{(t)} - \bsbb^*) - \bsbX^T (\nabla l_0(\bsbX \bsbb^{(t)}) - \nabla l_0(\bsbX\bsbb^*))  \|_{\infty}
        \\ & + \|\bsbX^T\bsbeps\|_{\infty} + \eta_{t+1} \|\bsbb^*\|_{\infty}+ ( \rho_{t+1}+\eta_{t+1}) \min_{j \in   {\mathcal{J}(\beta_j^{(t+1)})}} | \beta_j^{(t+1)}|.  \end{align*}
By \eqref{eq:infnormconvt1} of Lemma \ref{le:infnormconvt} and the definition of $\upsilon_t$, we get
\begin{align*}
       (\rho_{t+1}+ \eta_{t+1}) \| \bsbb^{(t+1)} - \bsbb^* \|_{\infty} \le \  &  (\rho_{t+1}-\upsilon_t) \| \bsbb^{(t)} - \bsbb^* \|_{\infty}
        \\ & + \|\bsbX^T\bsbeps\|_{\infty} + \eta_{t+1} \|\bsbb^*\|_{\infty}+ ( \rho_{t+1}+\eta_{t+1}) \frac{\| \bsbb^{(t+1)} - \bsbb^* \|_{2}}{\sqrt {q_{t+1} - s}}. \end{align*}
Additionally, we can  obtain $(\rho_{t+1}+ \eta_{t+1}) \| (\bsbb^{(t+1)} - \bsbb^*)_{\mathcal J^*} \|_{\infty} \le   (\rho_{t+1}-\upsilon_t) \| \bsbb^{(t)} - \bsbb^* \|_{\infty}           + \|\bsbX^T\bsbeps\|_{\infty} + \eta_{t+1} \|\bsbb^*\|_{\infty}$ or
\begin{align*}
         \| (\bsbb^{(t+1)} - \bsbb^*)_{\mathcal J^*} \|_{\infty} \le \  &  (1-\frac{\upsilon_t+\eta_{t+1}}{\rho_{t+1}+ \eta_{t+1}}) \| \bsbb^{(t)} - \bsbb^* \|_{\infty}
         + \frac{\|\bsbX^T\bsbeps\|_{\infty}}{  \rho_{t+1}+\eta_{t+1}} + \frac{\eta_{t+1} \|\bsbb^*\|_{\infty}}{\ \rho_{t+1}+\eta_{t+1}}, \end{align*}
by applying  \eqref{eq:infnormconvt2}. \end{proof}

\subsection{Model selection by predictive information criterion}
\label{appsub:tuning}

Although parameter $q$ as an upper bound of the true model support size can often be directly specified  based on domain knowledge, this section develops a new information criterion for the tuning of  $q$ to achieve the best prediction performance in finite samples. We assume \textit{multiple   responses} to cover the application in Section \ref{subsec:gbsense}. Let $\bsbY \in \mathbb R^{n\times m}$, $\bsbX\in \mathbb R^{n\times p}$ be  the response matrix and predictor matrix, respectively, and  $l_0(\bsbX \bsbB; \bsbY)$ be the given loss.  We use $\mathcal J(\bsbB)$ to denote the row support of $\bsbB$ and define $J(\bsbB) = |\mathcal J(\bsbB)|$.  Assume the true $\bsbB^*\in\mathbb R^{p\times m}$ is row-sparse and let $s^* =  J(\bsbB^*)$. The problem considered in the main sections corresponds to the special case $m=1$. To choose the best (row)  support size, we advocate the following complexity penalty to be added to the loss in the predictive information criterion:
\begin{align}
P(\bsbB) = J(\bsbB) m + J(\bsbB)\log \{ep / J(\bsbB)\}.
\end{align}
 Recall $\Breg_2(\bsbA_1, \bsbA_2) = \|\bsbA_1 - \bsbA_2 \|_F^2/2$ in the matrix context.
\begin{theorem} \label{thm:tuning}
Let  the   effective noise    $E=-\nabla l_0(\bsbX \bsbB^*)$    be sub-Gaussian with mean zero and   scale bounded by a constant and   $\bsbB^*\in \mathcal M$ and $\bsbB^*\ne \bsb0 $.  Assume that there exist constants $\delta>0$ and $A_0\ge 0$ such that $(\breg_{l_0}- \delta\Breg_2)(  \bsbX  \bsbB ,   \bsbX  \bsbB ') +A_{0} (P(  \bsbB ) +  P(  \bsbB' ))\ge0$, for all  $  \bsbB,  \bsbB'\in \mathcal M$.    Then for a sufficiently large constant $A$, any
      $\hat{\bsbB}$  that minimizes
%$%
\begin{equation} \label{PIC-general}
l_{0}( \bsbX  \bsbB ; \bsbY) + A P ( \bsbB)
%$ %
\end{equation}
subject to $  \bsbB\in \mathcal M$  must satisfy
\begin{equation} \label{tuning_general-1}
\EE \{  \| \bsbX \hat { \bsbB} -  \bsbX  \bsbB^* \|_F^2 \vee P(\hat \bsbB) \} \lesssim m s^*+  s^*\log(ep/s^*) .
\end{equation}
\end{theorem}

Theorem \ref{thm:tuning} does not involve any  regularization parameters (like $q ,  \lambda$), but it achieves the minimax optimal error rate  \eqref{tuning_general-1}. Moreover, the justification of \eqref{PIC-general}  does not require an infinite-sample-size,     design coherence  or  signal-to-noise ratio conditions.

   When  the noise distribution has a  dispersion parameter $\sigma^2$, Theorem \ref{thm:tuning} still applies, but the penalty in \eqref{PIC-general} becomes $A \sigma^2 P (\bsbB)$ with an unknown factor.  A preliminary  scale estimate can be possibly used. But an appealing result     for  regression is that  the estimation of $\sigma$ can be bypassed.
We give a scale-free form of predictive information criterion by
\begin{equation} \label{PIC-log}
mn\log \{\|\bsbY - \bsbX\bsbB \|_F^2\}+ A P (\bsbB),
\end{equation}
where   $A$ is an absolute constant.
     %motivated by Theorem \ref{thm:minimax}, the minimax lower bound suggests $P_o(s^*,o^*)$ be the ideal model-complexity penalty, which leads to a new information criterion for model comparison.
%Concretely, denoting $\mathcal J(\bm\beta) = \{j:\beta_j\ne 0\}$ and $J(\bm\beta) = |\mathcal J(\bm\beta)|$, we recommend the following \textit{predictive information criterion} (\textbf{PIC})
%\begin{equation} \label{PIC}
%\log l(\bm X\bm\beta + \bm\gamma; \bm y) + AP_o(\bm\beta, \bm\gamma)/n
%\end{equation}
%where $P_o(\bm\beta,\bm\gamma) = J(\bm\beta) + J(\bm\beta)\log(ep/J(\bm\beta)) + J(\bm\gamma) + J(\bm\gamma)\log(en/J(\bm\gamma))$.
%We use the low-dimensional regression problem as an example to illustrate the idea, since the new information criterion in this setup is similar to BIC family.

\begin{theorem} \label{thm:tuning_log_form}
Let $\bsbY = \bsbX\bsbB^* +\bsbE$, where  $E = [\epsilon_{i,k}]$ has independent centered sub-Gaussian$(\sigma^2)$ entries and $\EE \epsilon_{i,k}^2  \gtrsim  \sigma^2$ with  $\sigma^2$ unknown.  Define   $l_0(\bsbX\bsbB  ; \bsbY) = \|\bsbX\bsbB  - \bsbY\|_F^2$. Assume the true model is not over-complex in the sense that   $P(\bsbB^*)\le mn/A_0$ for some constant $A_0 > 0$. Let $\delta(\bsbB) = AP (\bsbB)/(mn)$, where $A$ is a positive constant satisfying $A<A_0$, and so $\delta(\bsbB^*)<1$. Then, for sufficiently large values of $A_0$ and $A$, any $\hat \bsbB$ that minimizes  $\log l_{0}(\bsbX\bsbB  ; \bsbY) + \delta(\bsbB)$
subject to $\delta(\bsbB)<1$ must satisfy
$ %\begin{equation} \label{robregscalefreerate}
\Breg_2 ( \bsbX \hat {\bsbB},  \bsbX  \bsbB^*  )   \lesssim \sigma^2 \{s^* m + s^*\log(ep/s^*)\}
$ %\end{equation}
with probability at least $1 -C p^{-c}\exp\{-c  m\}-C \exp ( -c    mn  )$   for some
        constants $C,  c >0$.
\end{theorem}
A more general form of $AP(\bsbB)$ can be expressed as     ``$\alpha_1 \times \text{degrees-of-freedom}    + \alpha_2 \times \text{inflation}$" with   $\alpha_1,\alpha_2$ as absolute constants.
The two theorems  can proved based on modifying the proofs of Theorems 2 and 3 in \cite{SheCV}. For completeness, we present some details below. Note that although the logarithmic form of the scale-free predictive information criterion is widely used, other non-asymptotic forms exist \citep{SheCV}. In fact, a key trick  in the proof is to convert these  forms into a fractional scale-free predictive information criterion, which is essential for establishing the desired properties.

\begin{proof}
We first prove Theorem \ref{thm:tuning} under the assumption that $\vect(\bsbE)$ is  subGaussian  with mean 0 and scale $\sigma$.
From the definition of $\hat \bsbB$, $\breg_{l_0}(\bsbX\hat \bsbB,\bsbX\bsbB^*)  + A\sigma^2P(\hat \bsbB)\le A\sigma^2P(\bsbB^*)   +   \langle \bsbE, \bsbX \hat \bsbB -       \bsbX \bsbB^* \rangle $. Similar to the proof of Lemma \ref{le:stochastic_bound}, we can show that  for any  $a, b,  a'>0$,  $4  b>a$, and   $t>0$,
\begin{align} \label{auxgroupspabound}
 \langle \bsbE, \bsbX   \bsbB -       \bsbX \bsbB^* \rangle \le(\frac{2}{a}+\frac{2}{a'})\Breg_{2}(\bsbX  \bsbB,\bsbX\bsbB^*)+ a' \sigma^2t + 4b L \sigma^2 \{P( \bsbB^* )+ P(  \bsbB )\}, \forall \bsbB\in \mathbb R^{p\times m}
\end{align}
occurs with probability at least $1 -C p^{-c} \exp(-cm) \exp(-ct)$, where $L, c, C$ are positive constants. (The probability bound can be derived by setting $L$ to a sufficiently large constant and observing that $Jm + J\log(ep/J) \ge m + \log(ep)$ holds for $J \ge 1$, and the union bound calculation, as in \eqref{Rbound}, does not need to cover the case $J=0$.)

Now, substituting $\hat \bsbB$ for $\bsbB$ in \eqref{auxgroupspabound} and  taking the expectation, we have   for any  $a, b,  a'>0$,  $4  b>a$,
\begin{align*}
&\EE \{\breg_{l_0}(\bsbX\hat \bsbB,\bsbX\bsbB^*) + A\sigma^2P(\hat \bsbB)\}\\ \le &  \EE\Big\{A\sigma^2P(\bsbB^*)  +(\frac{2}{a}+\frac{2}{a'})\Breg_{2}(\bsbX\hat \bsbB,\bsbX\bsbB^*) +c a' \sigma^2+ 4b L \sigma^2 [P(\bsbB^*)+ P( \hat \bsbB)]\Big\}.
\end{align*}
Combining it with the regularity condition gives
\begin{align*}
 \EE \big\{(\delta -\frac{2}{a}-\frac{2}{a'} )\Breg_{2}(\bsbX\hat \bsbB,\bsbX\bsbB^*) + (A - 4bL -C) P(   \hat  \bsbB ) \big\}  \le     (A+4bL+C) \sigma^2 P(\bsbB^*) + c a'  \sigma^{2}. \end{align*}
Since   $ P(\bsbB^*) \ge c > 0$,
choosing the constants satisfying $({1}/{a}+{1}/{a'})(1+{1}/{b'})<{\delta}/{2}$,  $4  b>a$, and $A > 4bL+C$  yields the   conclusion.

Next, we   prove  Theorem \ref{thm:tuning_log_form}.
We begin with a proof for $\hat B$ selected by  a fractional form of  scale-free form of predictive information criterion:  $ l_0(\bsbX   \bsbB ; \bsbY)/(1-\delta(\bsbB))$ subject to $\delta(\bsbB)\le 1$. Let     $h(\bsbB; A) = 1/\{  mn - A P(\bsbB)\}$.
From   the  optimality of   $\hat \bsbB$,  $l_0(\bsbX \hat \bsbB; \bsbY) h(\hat \bsbB; A)   \le l_0(\bsbX   \bsbB^*; \bsbY)   h(  \bsbB^*; A)$ or
\begin{align*}
 l_0(\bsbX \hat \bsbB; \bsbY) -l_0(\bsbX   \bsbB^*; \bsbY)    \le    l_0(\bsbX   \bsbB^*; \bsbY) \Big( \frac{ h(  \bsbB^*; A)}{ h(\hat \bsbB; A)}-1\Big),
\end{align*}
where we used $ h(\hat \bsbB; A)>0$.
Using the Bregman divergence for the quadratic function, we get
\begin{align}
   \Breg_2  (\bsbX \hat \bsbB, \bsbX \bsbB^* )     \le l_0(\bsbX   \bsbB^*; \bsbY) \Big( \frac{ h(  \bsbB^*; A)}{ h(\hat \bsbB; A)}-1\Big) + \langle \bsbE, \bsbX \hat \bsbB - \bsbX \bsbB^*\rangle. \label{sfpictempeq2}
\end{align}

From the definition of $h$ and the model parsimony assumption,   \eqref{sfpictempeq2} becomes
\begin{align}
&  \Breg_2  (\bsbX \hat \bsbB, \bsbX \bsbB^* )  \nonumber \\  \le \ &  l_0(\bsbX   \bsbB^*; \bsbY)\ \frac{ A P(\bsbB^*)- A P(\hat \bsbB)} {   mn  - A P(\bsbB^*)}  + \langle \bsbE, \bsbX \hat \bsbB - \bsbX \bsbB^*\rangle \nonumber\\
= \ &  \frac{1}{2}  \frac{A  \|\bsbE\|_F^2}{   mn\sigma^2  - A \sigma^2 P(\bsbB^*)}\sigma^2 P(\bsbB^*) -\frac{1}{2} \frac{A  \|\bsbE\|_F^2}{  mn - AP(\bsbB^*)}\sigma^2P(\hat\bsbB)  + \langle \bsbE, \bsbX \hat \bsbB - \bsbX \bsbB^*\rangle\nonumber\\
\le \ &  \frac{1}{2}  \frac{A  \|\bsbE\|_F^2}{  (1- A/A_{0}  )  mn\sigma^2 }\sigma^2 P(\bsbB^*) -\frac{1}{2 } \frac{A  \|\bsbE\|_F^2}{   mn \sigma^2}\sigma^2P(\hat\bsbB) + \langle \bsbE, \bsbX \hat \bsbB - \bsbX \bsbB^*\rangle. \label{eqsfpictemp1}
\end{align}

The stochastic term $ \langle \bsbE, \bsbX \hat \bsbB - \bsbX \bsbB^*\rangle$ can be bounded similarly by  \eqref{auxgroupspabound}:      for any  $a_{1}, b_{1}, a_{2}>0$ satisfying $4b_{1} >a_{1}$,
$$
\langle \bsbE, \bsbX  \hat \bsbB - \bsbX \bsbB^* \rangle
\leq2 ({1}/{a_{1}}+{1}/{a_{2}})\Breg_2  (\bsbX \hat \bsbB, \bsbX \bsbB^* )   + (b_{1}) L_{1} \sigma^2\{P(\hat\bsbB) + P(\bsbB^*)\},
$$
with probability at least $1-C p^{-c}\exp\{-c  m\}$ for some $c, C, L_{1}>0$.
 Plugging it into \eqref{eqsfpictemp1} gives
\begin{align*}
&\big(  1  -\frac{2}{a_{1}}-\frac{2}{a_2}\big)\Breg_2  (\bsbX \hat \bsbB, \bsbX \bsbB^* ) \\  \le   \; & \frac{1}{2}\Big\{ \frac{ A  \|\bsbE\|_F^2}{ (1 - A/A_{0}  ) mn\sigma^2    }+ 2 b_{1} L_1 \Big\}\sigma^2 P(\bsbB^*)   -\frac{1}{2} \Big\{\frac{A  \|\bsbE\|_F^2}{   mn \sigma^2} -  {2b_{1} L_1  }  \Big\}\sigma^2P(\hat\bsbB).
\end{align*}

Since $\epsilon_{i,k}$ are independent and  non-degenerate,  $c_{1} mn\sigma^2 \le \EE \| \bsbE\|_F^2 \le c_2mn \sigma^2$ for some  constants $c_{1}, c_2>0$. Let $\gamma$  be  some constant  satisfying $0< \gamma  <  1$.  On $\mathcal E =\{c_{1}(1-\gamma ) {mn\sigma^2}\leq \|\bsbE\|_F^2 \leq c_2(1+\gamma ) {mn\sigma^2} \}$,  we have
\begin{align*}
& \frac{ A  \|\bsbE\|_F^2}{(1- A/A_{0})mn\sigma^2 }  \le    \frac{c_2(1+\gamma )A_{0} A }{A_0 - A} \  \mbox{ and }  \ \frac{A  \|\bsbE\|_F^2}{mn \sigma^2} \ge  {c_{1}(1-\gamma) A }.
\end{align*}
Regarding the probability of the event, we write  $\| \bsbE\|_F^2 = \vect(\bsbE) \bsbA \vect(\bsbE)^T$ with $\bsbA = \bsbI\in \mathbb R^{nm\times nm}$ and bound it with the  Hanson-Wright inequality. In fact, from $\mbox{Tr}(\bsbA) =mn, \|\bsbA\|_2 =1,\| \bsbA\|_F=\sqrt{mn}$, the complement of $\mathcal E$ occurs with probability at most $C' \exp\{-c'm n\}$.

Now, with $A_0, A, a_{1}, a_{2}, b_{1}$   large enough such that       $({1}/{a_{1}}+{1}/{a_{2}})<{1}/{2}$,  $4  b_{1}>a_{1}$,   $A > 2b_{1} L_{1} /\{c_{1}(1-\gamma) \}\ $ and $A_0 >  A$,  we can obtain the desired prediction error rate   for the fractional form. Finally, based on the fact that   $1/(1-\delta) \ge \exp(\delta) \ge 1/(1-\delta/2)$ for any $0\le \delta<1$, the same error rate  holds for the logarithmic form  (see  \citep{SheCV} for more details).
\end{proof}
%%%%%%%%%%%%%%%%%%%%%%%%%%%%%%%%%%%%%%%%%%%%%%%%%%%%%%%%%%%%%%%%%%%%%%%%%%
%%%%%%%%%%%%%%%%%%%%%%%%%%%%%%%%%%%%%%%%%%%%%%%%%%%%%%%%%%%%%%%%%%%%%%%%%%
%\section{More  Implementation Details and Experiments} \label{appendix_detail}
\subsection{More implementation details}\label{appendix_impldetail}
Slow kill  is extremely simple to implement  and a summary is given below. For ease of presentation, we define an $\bar\eta$ function based on Theorem \ref{th:iter} and its discussions,
\begin{equation}
        \bar\eta(q_+, \rho_+) =
        \begin{cases}
                \frac{1}{2\sqrt{q_+     /\bar s} -1},    &\mbox{if} \ q_{+} > 2q \ \mbox{and} \  q \ge n/2 \\
                \frac{\eta_0}{\rho_+},  &\mbox{if} \ q_{+} \le 2q,\\
                \frac{\eta_0}{\rho_+    } \wedge \frac{1}{2\sqrt{q_+    /\bar s} -1}, &\mbox{otherwise},
        \end{cases}
\end{equation}
where $\bar s = q \wedge nL^2/\log(ep) \ge s$ with $L$ the Lipschitz parameter of $\nabla l_0$ and $\eta_0$ is a user defined parameter. (Like  $q$,  $\eta_0$ is a regularization parameter   customizable by the user.) We also define a $\bsbb$ function
\begin{equation} \label{beta_func}
        \bsbb(q_+, \rho_+, \bsbb^-) = \Theta^{\#}\Big\{\bsbb^{-} - \rho_+^{-1}\bsbX^T\nabla l_0(\bsbX\bsbb^{-}; \bsby);
        q_+, \bar\eta(q_+, \rho_+)\Big\},
\end{equation}
based on \eqref{eq:general_update}. (Often, an intercept should be included (say $\beta_1$) that is subject to no regularization.
We can add a column of ones in the design matrix and redefine the $\Theta^{\#}$ in \eqref{beta_func} to keep the first entry and perform quantile-thresholding on the remaining subvector.)

Recall the line search criterion for a trial $\rho$:
\begin{equation} \label{rho_criterion}
        \begin{split}
                (\rho \Breg_2 - \breg_{l})(\bsbb(q_{t+1}, \rho, \bsbb^{(t)}), \bsbb^{(t)}) \geq 0
        \end{split}
\end{equation}
or
\begin{align*}
        \frac{\rho}{2}\|\bsbb(q_{t+1}, \rho, \bsbb^{(t)}) - \bsbb^{(t)}\|_2^2 \ge \, & l_0(\bsbX\bsbb(q_{t+1}, \rho, \bsbb^{(t)})) - l_0(\bsbX\bsbb^{(t)})\\
        & - \langle \nabla l_0(\bsbX\bsbb^{(t)}), \bsbX\bsbb(q_{t+1}, \rho, \bsbb^{(t)}) - \bsbX\bsbb^{(t)} \rangle.
\end{align*}
Then the algorithm can be summarized as follows.\\
%\begin{algorithm}[h]
%       \caption{ Slow Kill \label{alg:SK}}
%       \begin{algorithmic}
%               \footnotesize{
%                       \STATEx{Input: $\bsbX, \bsby$, a quantile parameter sequence $q_t \to q \in [p]$, a target $\ell_2$-shrinkage $\eta_0 \ge 0$. }
%                       \STATEx{Initialization: $\bsbb^{(0)}, \rho_0$ (say $\bsb{0}$ and $L\|\bsbX\|_2^2$, respectively).}
%                       \STATEx{For each $q_{t+1}$ ($t \ge 0$), perform the following}
%                       \STATEx{\qquad\enspace  a) Find $\rho_{t+1}$ by line search with the criterion \eqref{rho_criterion}.}
%                       \STATEx{\qquad\enspace  b) Perform $\bsbb^{(t+1)} \gets \bsbb(q_{t+1}, \rho_{t+1}, \bsbb^{(t)})$ according to \eqref{beta_func}.}}
%       \end{algorithmic}
%\end{algorithm}
{\small

Input: $\bsbX, \bsby$, a quantile parameter sequence $q_t \to q \in [p]$, a target $\ell_2$-shrinkage $\eta_0 \ge 0$.

Initialization: $\bsbb^{(0)}, \rho_0$ (say $\bsb{0}$ and $L\|\bsbX\|_2^2$, respectively).

For each $q_{t+1}$ ($t \ge 0$), perform the following
\begin{compactitem}
\item[a)]  Find $\rho_{t+1}$ by line search with the criterion \eqref{rho_criterion}.

\item[b)]   Perform $\bsbb^{(t+1)} \gets \bsbb(q_{t+1}, \rho_{t+1}, \bsbb^{(t)})$ according to \eqref{beta_func}.\\
\end{compactitem}
}

%\begin{algorithm}[h]
%        \caption{{Slow Kill} \label{alg:SK}}
%        \vspace*{-8pt}\small{
%
%                Input: $\bsbX, \bsby$, a quantile parameter sequence $q_t \to q \in [p]$, a target $\ell_2$-shrinkage $\eta_0 \ge 0$.
%
%                Initialization: $\bsbb^{(0)}, \rho_0$ (say $\bsb{0}$ and $L\|\bsbX\|_2^2$, respectively).
%
%                For each $q_{t+1}$ ($t \ge 0$), perform the following
%
%                $\quad$ a) Find $\rho_{t+1}$ by line search with the criterion \eqref{rho_criterion}.
%
%
%                $\quad$  b) Perform $\bsbb^{(t+1)} \gets \bsbb(q_{t+1}, \rho_{t+1}, \bsbb^{(t)})$ according to \eqref{beta_func}.}
%\end{algorithm}

We can also add a squeezing operation as step c):
$  \bsbX \gets \bsbX_{\mathcal{J}(\bsbb^{(t+1)})}$
from time to time  (say when $q_{t+1}$ reaches $p/2^k$ for $k$ greater than some $k_0$).
In addition, after $q_{t+1}$ reaches $q$ and when the sparsity pattern of $\bsbb^{(t+1)}$ stabilizes, one can use a classical optimization method to solve a smooth problem to get the nonzero entries of the final estimate. As for step a), many standard line search methods can be used, e.g., backtracking  \citep{Boyd2004}. We use an adaptive search with warm starts. Concretely, given $\alpha \in (0,1)$, we begin with $\rho \gets \rho_{t}$, and set $\rho \gets \alpha \rho$ if \eqref{rho_criterion} is satisfied for $\bsbb(q_{t+1}, \rho, \bsbb^{(t)})$ and $\rho \gets \rho/\alpha$ otherwise, until a small enough $\rho_{t+1}$ makes \eqref{rho_criterion} hold while $\alpha \rho_{t+1}$ does not. In practice, it is wise to limit the number ($M$) of searches. We use  $\alpha = 0.5, M = 5$ for implementation.

{%\small
\bibliographystyle{IEEEtranN}
\bibliography{SKrefs}
}

\end{document}